%% file: main.tex
\documentclass[runningheads,a4paper]{llncs}

\usepackage{geometry}
\usepackage{pgfplots}
\usepackage[all]{nowidow}
\usepackage[utf8]{inputenc}
\definecolor{blue}{HTML}{1F77B4}
\definecolor{orange}{HTML}{FF7F0E}
\definecolor{green}{HTML}{2CA02C}

\pgfplotsset{compat=1.14}

\usepackage{microtype}
\usepackage{graphicx}

\usepackage{hyperref}

\usepackage{yhmath}
\usepackage{cancel}
\usepackage{color}
\usepackage{siunitx}
\usepackage{array}
\usepackage{multirow}
\usepackage{amssymb}
\usepackage{tabularx}
\usepackage{booktabs}

\usepackage{algorithm}
\usepackage{algorithmic}
\usepackage{setspace}

\newtheorem{remarks}[theorem]{Remark}
\newtheorem{assumption}{Assumption}
\newenvironment{sproof}{\paragraph{Sketch of the proof. }}{}

\usepackage{amsmath}
\usepackage{mathtools}
\usepackage{thmtools, thm-restate}

\usepackage[capitalize,noabbrev]{cleveref}

\newcommand{\qedsymbol}{$\blacksquare$}

\ifdefined\proof
    \renewenvironment{proof}{\par\noindent{\bf Proof\ }}{\hfill\qedsymbol\\[2mm]}
\else
    \newenvironment{proof}{\par\noindent{\bf Proof\ }}{\hfill\qedsymbol\\[2mm]}
\fi

\usepackage[textsize=tiny]{todonotes}
\usepackage{dsfont}
\usepackage[T1]{fontenc}

\usepackage{listings}

\definecolor{codegreen}{rgb}{0,0.6,0}
\definecolor{codegray}{rgb}{0.5,0.5,0.5}
\definecolor{codeorange}{rgb}{230, 96, 0}
\definecolor{backcolour}{rgb}{0.95,0.95,0.92}

\lstdefinestyle{mystyle}{
    backgroundcolor=\color{backcolour},
    commentstyle=\color{codegreen},
    keywordstyle=\color{blue},
    numberstyle=\tiny\color{codegray},
    basicstyle=\ttfamily\footnotesize,
    breakatwhitespace=false,  
    breaklines=true,                 
    captionpos=b,                    
    keepspaces=true,                 
    numbers=left,                    
    numbersep=5pt,                  
    showspaces=false,                
    showstringspaces=false,
    showtabs=false,                  
    tabsize=2,
    emph = {forward, loss, train, predict, __init__},
    emphstyle=\color{purple}\ttfamily
}
\newcommand{\R}{\mathbb{R}}
\newcommand{\N}{\mathbb{N}}

\newcommand{\E}{\mathbb{E}}

\newcommand{\bmH}{\mathcal{H}}

\newcommand{\bmX}{\mathcal{X}}

\newcommand{\bmO}{\mathcal{O}}
\newcommand{\bmY}{\mathcal{Y}}
\newcommand{\bmZ}{\mathcal{Z}}

\newcommand{\deff}{d_{\text{eff}}}
\newcommand{\newP}{\mathsf{P}}

\DeclareMathOperator{\Tr}{Tr}

\DeclareMathOperator{\rank}{rank}
\DeclareMathOperator{\hs}{HS}

\DeclareMathOperator*{\argmin}{arg\,min\,}

\usepackage[round]{natbib}

\begin{document}

\title{Vector-Valued Least-Squares Regression under Output Regularity Assumptions}
%
%
\author{Luc Brogat-Motte\inst{1} \and
Alessandro Rudi\inst{2} \and
Céline Brouard\inst{3} \and
Juho Rousu\inst{4} \and
Florence d'Alché-Buc\inst{1}
}

%
%
\institute{LTCI, Télécom Paris, Institut Polytechnique de Paris, France
\email{luc.motte@telecom-paris.fr}\\ \and
INRIA, Paris, France, École Normale Supérieure, Paris, France, PSL Research, France\\ \and
Université de Toulouse, INRAE, UR MIAT, France \and
Department of Computer Science, Aalto University, Finland
}
\maketitle              
\thispagestyle{plain}\pagestyle{plain}
\begin{abstract}
We propose and analyse a reduced-rank method for solving least-squares regression problems with infinite dimensional output. We derive learning bounds for our method, and study under which setting statistical performance is improved in comparison to full-rank method. Our analysis extends the interest of reduced-rank regression beyond the standard low-rank setting to more general output regularity assumptions. We illustrate our theoretical insights on synthetic least-squares problems. Then, we propose a surrogate structured prediction method derived from this reduced-rank method.  We assess its benefits on three different problems: image reconstruction, multi-label classification, and metabolite identification.
\end{abstract}

\section{Introduction}\label{sec:intro}
\input{sections/introduction}

\section{Problem Setting and Proposed Estimator}\label{sec:pbm_setting}
\input{sections/setting}

Notations are gathered in Table \ref{tab:notations}.
\section{Theoretical analysis}\label{sec:theory}
\input{sections/theory}

\section{Application to Structured Prediction}\label{sec:struc_pred}
\input{sections/struct_pred}

\section{Numerical Experiments}\label{sec:experiments}
\input{sections/applications}
\input{sections/new-experiments}

\section{Conclusion}

In this paper, we proposed a novel reduced-rank regression estimator in the case of regularized least squares regression with infinite dimensional outputs and gave excess-risk bounds under general output regularity assumptions. In particular, we characterized a family of situations where reduced-rank regression is statistically and computationally beneficial. We used the proposed reduced-rank regression for structured prediction, and derived theoretical guarantees on the resulting estimator. Experiments on structured prediction problems confirm the advantages in practice of the approach. 

\section*{Acknowledgements}
The first and last authors are funded by the French National Research Agency (ANR) through ANR-18-CE23-0014 APi (Apprivoiser la Pré-image) and the Télécom Paris Chair DSAIDIS. This work was supported by the French government under the management of the Agence Nationale de la Recherche as part of the “Investissements d’avenir” program, reference ANR-19-P3IA-0001 (PRAIRIE 3IA Institute). We also acknowledge support from the European Research Council (grants REAL 947908). This research was partially funded by Academy of Finland (grants 334790, 339421, and 345802).

\newpage
\appendix
\section{Proofs of the Learning Bounds}\label{app:proof}
\input{sections/proofs}

\section{Additional Experimental Details}\label{app:exp_details}
\input{sections/exp_details}

\renewcommand\bibname{References}
\setlength\bibsep{10pt}
\bibliography{references}
\bibliographystyle{plainnat}



\end{document}

%% file: sections/introduction.tex
Learning vector-valued functions plays a key role in a large variety of fields such as economics \citep{lutkepohl2013vector}, physics, computational biology, where multiple variables have to be predicted simultaneously. As opposed to solving multiple single regression problems, the interest of vector-valued regression lies on the ability to take into account the dependence structure among the output variables by appropriate regularization \citep[see for instance][]{micchelli2005learning,baldassarre2012multi,alvarez2011kernels,lim2015operator} or by imposing a low-rank assumption \citep{anderson1951estimating,izenman1975reduced,velu2013multivariate}.  Regarding the infinite dimensional output case, besides functional output regression \citep{kadri2016}, the motivation for vector-valued regression mainly comes from the application of surrogate approaches in Structured Output Prediction \citep[][]{weston2003kernel,geurts2006, kadri2013,brouard2016input,ciliberto2020general}.  In order to learn a model to predict an output with some discrete structure, surrogate approaches embed the structured output variable into a Hilbert space and thus boil down to vector-valued regression with a potentially infinite dimensional output space. At prediction time, decoding allows to return a prediction in the original structured output space. Image completion \citep{weston2003kernel}, label ranking \citep{korba-icml} and graph prediction \citep{brouard2016fast} are all examples of structured prediction tasks that can be handled by surrogate approaches.

One way to implement infinite dimensional output regression consists in learning in vector-valued Reproducing Kernel Hilbert Spaces (vv-RKHS) \citep{micchelli2005learning}. In particular, regularized least-squares estimators in vv-RKHS enjoy strong theoretical guarantees \citep[see][]{caponnetto2007optimal}. However complex tasks such as structure prediction very often involve a limited amount of training data compared to the complexity of the input and output data. To overcome this issue, the structure of the target output can be leveraged. This is typically the goal of reduced-rank approaches \citep[][]{mukherjee2011reduced, luise2019leveraging}. 
 
 In this paper, our aim is to improve upon the regularized least-squares estimators by imposing a rank constraint on the least-squares estimator. Our contributions are three-fold.
 
As a first contribution, we introduce a novel reduced-rank estimator for vector-valued least-squares regression in the general case of infinite dimensional outputs. Denoting $\bmY$ a Hilbert space and $\bmX$ a Polish space, we consider the following relationship between the input variable and the output variable:
\begin{equation}\label{eq:regression}
y = h^*(x) + \epsilon, 
\end{equation}
where the pair of random vectors $(x,y)$ takes its values in $\bmX \times \bmY$, $\epsilon \in \bmY$ is a random noise independent of $x$ with expectation $\E[\epsilon]=0$ and $h^* : \bmX \rightarrow \bmY$ is a measurable function. \\
Assuming we have already an estimator $\hat{h} : \bmX \rightarrow \bmY$ of $h^*$ built from a training i.i.d. sample $(x_i,y_i)_{i=1}^n$, we propose to learn a linear operator $\hat{P}$ of rank $p$, for $p \in \N^*$ allowing to project $\hat{h}(x)$ onto $Z \subset \bmY$ with $\dim(Z) \leq p$  giving rise to the following new estimator:
\begin{equation*}
x \mapsto \hat{P}\hat{h}(x).
\end{equation*}
This novel estimator generalizes the reduced-rank kernel ridge regression estimator proposed by \citet{mukherjee2011reduced} to the infinite dimensional case.

The second contribution of this paper is to study the proposed least-squares estimator under output regularity assumptions and provide excess-risk bounds. We assume that $h^*$ belongs to a vector-valued reproducing kernel Hilbert Space, namely $h^* = H\phi(.)$ with $H \in \bmY \otimes \bmH_x, \|H\|_{\hs}<+\infty$, and $\phi: \bmX\rightarrow \bmH_x$ is a canonical map associated to a scalar-valued kernel $k:\bmX \times \bmX \rightarrow \R$. The difficulty of the learning problem in Eq. \eqref{eq:regression} can be characterized by standard complexity measures. For instance, the capacity condition measures the regularity of the features in terms of eigenvalue decay rate of the covariance operator $C=\E[\phi(x) \otimes \phi(x)]$, and the source condition measures the regularity of $H$ in terms of alignment of $H^*H$ with C \citep{caponnetto2007optimal, ciliberto2020general, varre2021last}. The more regular the problem is, the better are the statistical guarantees. In this work, we consider regularity assumptions on the outputs of the learning problem. We measure the eigenvalue decay rates of the covariance operator $\E[h^*(x) \otimes h^*(x)]$, and $\E[\epsilon \otimes \epsilon]$, and also the alignment of $HH^*$ with $HCH^*$.

The third contribution of this paper is a novel structured prediction method, which leverages our reduced-rank estimator in the surrogate regression problem. The proposed approach makes use of both an input and an output kernel. In this case, the resulting surrogate regression problem's output space is thus a reproducing kernel Hilbert space. The least-squares analysis allows to prove the  the statistical and computational interest of the structured prediction method. In particular, consistency and learning rates for our structured prediction method are given. Moreover, we show by an extensive empirical study  on different real world structured prediction tasks that the proposed approach improves upon full rank and state-of-the art structured prediction approaches.

\vspace{5mm}

\paragraph{Outline. }

The paper is organized as follows. In Section \ref{sec:pbm_setting}, we provide a novel reduced-rank method for solving vector-valued least-squares problems. In Section \ref{sec:theory}, we give learning bounds for the proposed least-squares estimator. Then, we study under which setting this method improves the statistical and computational performance. In particular, our analysis includes and extends the interest of reduced-rank regression beyond the standard setting of reduced-rank regression where the optimum is assumed to be low-rank, and the noise homogeneous in $\bmY$. In Section \ref{sec:struc_pred}, we show how the proposed estimator can be advantageously used in structured prediction with surrogate methods. We give an excess-risk bound for the resulting structured predictor, inherited from our least-squares theoretical analysis.
In Section \ref{sec:experiments}, we illustrate our theoretical analysis on synthetic least-squares problems. We empirically show the benefit of the method in structured prediction on three different problems: image reconstruction, multi-label classification, and metabolite identification.

%% file: sections/setting.tex
In this section, we introduce the learning setting of vector-valued least-squares regression. Then, we give background on kernel ridge regression. Finally, we present the reduced-rank least-squares estimator proposed in this work.

\paragraph{Vector-valued least-squares regression. } We consider the problem of estimating a function $h: \bmX \rightarrow \bmY$ with values in a separable Hilbert space $\bmY$ with norm $\|.\|_{\bmY}$, given a finite set $\{(x_i, y_i)_{i=1}^n\}$ independently drawn from an unknown distribution $\rho$ on $\bmX \times \bmY$, minimizing the expected risk
\begin{align}\label{eq:ls}
    R(h) = \E_{\rho}[\|h(x) - y\|^2_{\bmY}].
\end{align}
The solution is given by $h^*(x) := \E_{\rho(y|x)}[y]$. We define the noise $\epsilon$ as the random variable defined by the following equation
\begin{align}\label{eq:noise}
    y = h^*(x) + \epsilon.
\end{align}
In practice, solving $\eqref{eq:ls}$ requires the choice of an hypothesis space $\bmH$. In this work, we consider reproducing kernel Hilbert space (RKHS).

\paragraph{Reproducing kernel Hilbert spaces. } Given a positive definite kernel $k:\bmX \times \bmX \rightarrow \R$, one can build a Hilbert space $\bmH_x$ of scalar-valued functions $\bmH_x$, called the associated RKHS of $k$, defined by the completion $\bmH_x = \overline{\text{span}\{k(x, .) \,|\, x \in \bmX\}}$ according to the norm induced by the scalar product $\langle k(x, .), k(x', .)\rangle_{\bmH_x} := k(x,x')$. There is a one-to-one relation between a kernel $k$ and its associated RKHS \citep{aronszajn1950theory}. A crucial tool is the representer theorem which allows to solve in practice regularized empirical risk minimization problems over RKHS \citep{wahba90,scholkopf01}.

\paragraph{Vector-valued reproducing kernel Hilbert spaces.} The theory of vector-valued RKHSs (vv-RKHSs) extends the theory of real-valued RKHS by enabling to build Hilbert spaces of vector-valued functions \citep{senkene1973hilbert, micchelli2005learning, carmeli2010vector}. We note $A^*$ the adjoint of any operator $A$. An operator-valued kernel is an application $K: \bmX \times \bmX \rightarrow \mathcal{L}(\bmY)$ with values in the set of bounded linear operator on $\bmY$, satisfying the two following properties: $K(x,x') = K(x',x)^*$ and $\sum_{i,j=1}^n \langle K(x_i,x_j') y_i,\, y_j\rangle_{\bmY} \geq 0$ for any $n\in\mathbb{N}^*$, $(x_1, y_1), \dots, (x_n, y_n) \in \bmX \times \bmY$. Then, akin to scalar-valued kernel, one can build a Hilbert space $\bmH$ of vector-valued function from $\bmX$ to $\bmY$, called the associated RKHS of $K$, defined by the completion $\bmH = \overline{\text{span}\{K(x,.)y \,|\, (x,y) \in \bmX \times \bmY\}}$ according to the norm induced by the scalar product $\langle K(x, .)y, K(x', .)y'\rangle_{\bmH} := \langle K(x,x')y, y'\rangle_{\bmY}$. There is a one-to-one relation between a kernel $K$ and its associated vv-RKHS. Learning with operator-valued kernels is also possible thanks to representer theorems \citep{micchelli2005learning}. 

\paragraph{Kernel ridge regression. } The kernel ridge regression method (KRR) considers the estimator minimizing the following empirical objective
\begin{align}\label{eq:kkr}
    \min_{h \in \bmH} \frac{1}{n}\sum\limits_{i=1}^n \|h(x_i) - y_i\|^2_{\bmY} + \lambda \|h\|_{\bmH}^2
\end{align}
where $\bmH$ is the RKHS associated to an operator-valued  kernel $K$. In this work, we consider kernel of the form $K(x,x') = k(x,x') I_{\bmY}$, where $k: \bmX \times \bmX \rightarrow \R$ is a positive definite scalar-valued kernel on $\bmX$. In this case, the solution of the problem above can be computed in closed-form as follows:
\begin{align}\label{eq:sol_ridge_gen}
    \hat h(x) = \sum\limits_{i=1}^n \alpha_i(x) y_i, \quad \text{ with } \alpha(x) = (K + n\lambda)^{-1}k_x
\end{align}
where $K = (k(x_i, x_j))_{i,j=1}^n \in \R^{n \times n} $, and $k_x = (k(x, x_i))_{i=1}^n \in \R^n$.
\paragraph{Related works in reduced-rank regression. } Reduced-rank (or low-rank) estimators are estimators whose predictions $\hat y \in \bmY$ lie in a linear subspace $Z \subset \bmY$, estimated from the data.  Reduced-rank regression methods have been proposed for both linear models \citep{izenman1975reduced} and non parametric models \citep{mukherjee2011reduced, foygel2013nonparametric, rabusseau2016low,luise2019leveraging}. Two ways of building reduced-rank estimators have been proposed so far. A first way consists in imposing small rank constraints on the estimated linear operator  \citep{izenman1975reduced, mukherjee2011reduced,rabusseau2016low}: on other words,  the obtained estimators can be written as full-rank estimators that has been projected with estimated projection operators for a chosen rank $p$. Among those works devoted to finite dimensional vector-valued regression, the contribution of \citet{rabusseau2016low} differs in many ways. They consider a tensor output (the constraint is thus a multilinear rank constraint) and also provide learning bounds. Another way to address reduced-rank regression is to use nuclear norm (or trace norm) penalization as a convex relaxation to rank penalization as developed in \citep{romera2013multilinear, foygel2013nonparametric, luise2019leveraging}. It is worth mentioning that only \citet{luise2019leveraging} tackle an infinite dimensional vector valued-regression problem and provide a statistical study. More precisely, in terms of statistical guarantees, \citet{rabusseau2016low} and \citet{luise2019leveraging} show improved constants in learning bounds when using reduced-rank regression, in comparison with full-rank, in their respective settings.

\paragraph{Proposed least-squares estimator. } We introduce a non-parametric estimator belonging to the family of reduced-rank estimators. Let $\lambda_1, \lambda_2 >0$ and $p \in \N^*$. Let $\mathcal{P}_p$ be the set of the orthogonal projections from $\bmY$ to $\bmY$ of rank $p$. We note $\hat h_{\lambda}$ a KRR estimator defined using with the training sample $(x_i, y_i)_{i=1}^n$ and a regularization parameter $\lambda>0$. \\
Ideally, we would propose the reduced-rank estimator $x \mapsto P\hat h_{\lambda_2}(x)$ where $P$ is the operator defined as follows:
\begin{align}
    P := \argmin\limits_{\newP \in \mathcal{P}_p} \E[\|\newP h^*(x) - h^*(x)\|^2_{\bmY}]\label{eq:idealP}.
\end{align}
Nevertheless, $P$ is unknown, so we replace it by the following empirical estimator
\begin{align}
    \hat P_{\lambda_1} := \argmin\limits_{\newP \in \mathcal{P}_p} \frac{1}{n}\sum\limits_{i=1}^n \|\newP \hat h_{\lambda_1}(x_i) - \hat h_{\lambda_1}(x_i)\|^2_{\bmY}\label{eq:obj}, 
\end{align}
based on a KKR estimator $\hat h_{\lambda_1}$ of $h^*$, with possibly $\lambda_1 \neq \lambda_2$.
%
Eventually, this approximation gives rise to the following proposition for our reduced-rank estimator with hyperparameters $(p, \lambda_1, \lambda_2)$:
\begin{align}
    x \mapsto \hat P_{\lambda_1} \hat h_{\lambda_2}(x). \label{eq:es}
\end{align}


\begin{remarks}
Note that $P$ is the projection onto the span of the $p$ eigenvectors of the covariance operator $\E[h^*(x) \otimes h^*(x)]$ corresponding to the $p$ greatest eigenvalues. Similarly, $\hat P_{\lambda_1}$ is the projection onto the span of the $p$ eigenvectors of the empirical covariance operator $\frac{1}{n} \sum_{i=1}^n \hat h_{\lambda_1}(x_i) \otimes \hat h_{\lambda_1}(x_i)$ corresponding to the $p$ greatest eigenvalues.
\end{remarks}

The proposed estimator allows to cope with any separable Hilbert output space $\bmY$ (potentially infinite dimensional), which is of practical interest (See Section \ref{sec:struc_pred}). Furthermore, efficient and theoretically grounded approximation methods for KRR and kernel principal component analysis \citep{rudi2015less, rudi2016generalization, sterge2020gain} can be straightforwardly leveraged to alleviate the computation of this estimator.
For sake of simplicity, in the remainder of the paper, except when it is necessary, we omit the dependency in $\lambda_1$ and $\lambda_2$ and use notations $\hat h$ and $\hat P$.

\input{sections/notations}

\begin{remarks} The proposed estimator can be seen as a generalization of the reduced-rank estimator defined in \citep{mukherjee2011reduced} for finite dimensional vector-valued to the infinite dimensional output case and when $\lambda_1$ and $\lambda_2$ are not necessarily equal. In this work, we additionally provide learning bounds by leveraging the linear structure of the noise $\epsilon$ and those of the outputs $h^*(x)$.
\end{remarks}

%% file: sections/notations.tex
\vspace{0.7cm}
\begin{table}
\centering
{
\setlength\arrayrulewidth{1pt}
\begin{tabular}{|c|l|}
\hline
$\bmX$ & input space\\
$\bmY$ & regression output space\\
$\bmZ$ & structured output space\\
$\rho$ & probability distribution on $\bmX \times \bmY$\\
$\|.\|_{\bmY}$ & norm of the Hilbert space $\bmY$\\
$n$/$n_{te}$ & number of training data/test data\\
$h^*$ & least-squares optimum $x \rightarrow \E_{\rho(y|x)}[y]$\\
$\Delta$ & structured loss $\Delta: \bmZ \times \bmZ \rightarrow \R^+$\\
$f^*$ & structured prediction optimum $x \rightarrow \argmin_{\hat z \in \bmZ} \E_{\rho(z|x)}[\Delta(z, \hat z)]$\\
$k$ & positive definite kernel on $\bmX$\\
$\bmH_x$ & RKHS associated to $k$\\
$\bmH$ & vv-RKHS associated to $K(x,x') = k(x,x')I_{\bmY}$\\
$\mathcal{P}_p$ & space of orthogonal projections from $\bmY$ to $\bmY$ with rank $p$\\
$P$ & $\argmin_{\newP \in \mathcal{P}_p} \E[\|\newP h^*(x) - h^*(x)\|_{\bmY}^2]$\\
$A^*$ & adjoint of A\\
$A \preceq B$ &  $\forall u, \langle u, Au \rangle \leq \langle u, Bu \rangle$ \\
$\mu_p(A)$ & $p$-th eigenvalue of $A$ sorted in decreasing order\\
$\|.\|_{\hs}$ & Hilbert-Schmidt norm\\
$\|.\|_{\infty}$ & operator norm\\
$a \otimes b$ & defined such as $\forall x, a \otimes b x = \langle b,\, x \rangle a$\\
$S_p(A)$ & $\sum_{k=1}^p \mu_k(A)$\\
\hline
\end{tabular}
}
\vspace{0.4cm}
\caption{Notations}
\label{tab:notations}
\end{table}

%% file: sections/theory.tex
In this section, we present a statistical analysis of the proposed estimator. We start, in Section \ref{subsec:assumptions}, by giving the assumptions on the learning problem that we considered. Then, in Section \ref{subsec:main_result}, we provide learning bounds. Finally, in Section \ref{subsec:th_poly}, we study under which setting reduced-rank regression is statistically and computationally beneficial.


\subsection{Assumptions}\label{subsec:assumptions}

Here, we introduce and discuss the main assumptions that we need in order to prove our results.
\begin{assumption}[attainable case]\label{as:1} We assume that the solution $h^*$ belongs to the RKHS associated to the kernel $K(x,x') = k(x,x') I_{\bmY}$, i.e. there exists a linear operator $H$ from $\bmH_x$ to $\bmY$ with $\|H\|_{\hs} < +\infty $ such that:
\begin{align}
    h^*(x) = H\phi(x).
\end{align}
\end{assumption}
This assumption states that the solution $h^*$ indeed belongs to the chosen hypothesis space $\bmH$. It is a standard assumption in the learning theory \citep{ciliberto2020general}.
\begin{assumption}[regularity of target's outputs]\label{as:2}
The operator $M=\E[h^*(x) \otimes h^*(x)]$  satisfies the following property. There exists $\alpha \in [0,1]$ such that:
\begin{align}
    c_1 := \Tr(M^{\alpha}) < +\infty.
\end{align}
\end{assumption}
Assumption \ref{as:2} is always verified for $\alpha=1$ (as $\Tr(M) \leq \|H\|_{\hs}^2 \kappa^2$), and the smaller the $\alpha$ the faster is the eigenvalue decay of $M$. It quantifies the regularity of the target's outputs $h^*(x) \in \bmY$. As a limiting case, when $M$ is finite rank $\alpha=0$. The capacity condition is a standard assumption for least-squares problems, which can be written $\Tr(C^{r}) < +\infty$ with $r \in [0,1]$, and that characterises instead the regularity of the features $\phi(x) \in \bmH_x$. Remark that it implies the Assumption \ref{as:2} to hold with at least $\alpha \leq r$, but $\alpha \ll r$ is possible.

\begin{assumption}[output source condition]\label{as:3} The operators $H$ and $C=\E[\phi(x) \otimes \phi(x)]$ satisfy the following property. There exists $\beta \in [0,1]$, $c_2>0$ such that:
\begin{align}
    HH^* \preceq c_2 M^{1-\beta}.\label{eq:as3}
\end{align}
\end{assumption}
Assumption \ref{as:3} is always verified for $\beta=1$ (as $\|H\|_{\infty} < +\infty$), and the smaller the $\beta$ the stricter the assumption is. It quantifies the alignment of the left-singular vectors of $H$ with the main components of $M$. The source condition is a standard assumption for least-squares problems, which can be written $H^*H \preceq a C^{1-r}$ with $r \in [0,1], a>0$, and that quantifies instead the alignment of the right-singular vectors of $H$ with the main components of $C$ \citep[See, e.g.][]{ciliberto2020general, caponnetto2007optimal}. The Assumption \ref{as:3} allows to show a fast convergence rate of $\hat P$. In general, Assumption \ref{as:3} can be maximum ($\beta=0$) while the source condition is arbitrarily weak ($r=1$).
\begin{assumption}[diffuse noise and concentrated signal]\label{as:4} The operators $M$ and $E=\E[\epsilon \otimes \epsilon]$ satisfy the following property. There exists $\gamma \in [0,1]$, $c_3>0$ such that
\begin{align}
    c_3 M^{1-\gamma} \preceq E. \label{eq:as4}
\end{align}
\end{assumption}
Assumption \ref{as:4} quantifies the alignment of the main components of $E$ and $M$, and the greater the $\gamma$ the more the noise is diffuse in comparison to the signal. As a limiting case, when $\gamma \rightarrow 1$, then $\sigma^2I_{\bmY} \preceq E$ with a certain $\sigma^2>0$, which is only possible in finite dimension (e.g. $E=\sigma^2I_{\bmY}$, homogeneous noise commonly assumed in low-rank regression).
\newpage

\begin{figure}[ht!]
\centering
\begin{minipage}{.48\textwidth}
    \centering
    \includegraphics[width=1.\linewidth]{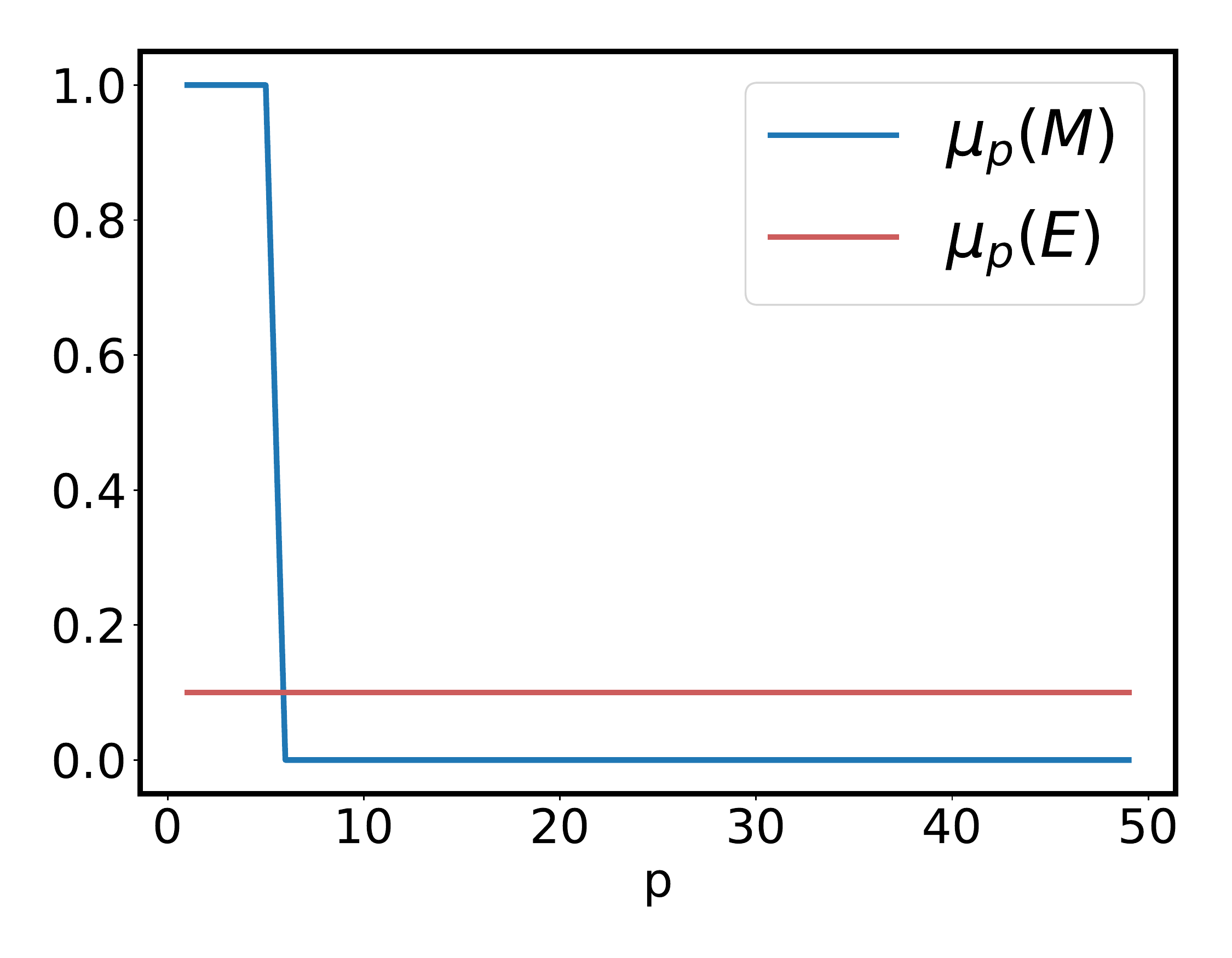}
\end{minipage}
\begin{minipage}{.48\textwidth}
    \centering
    \includegraphics[width=1.\linewidth]{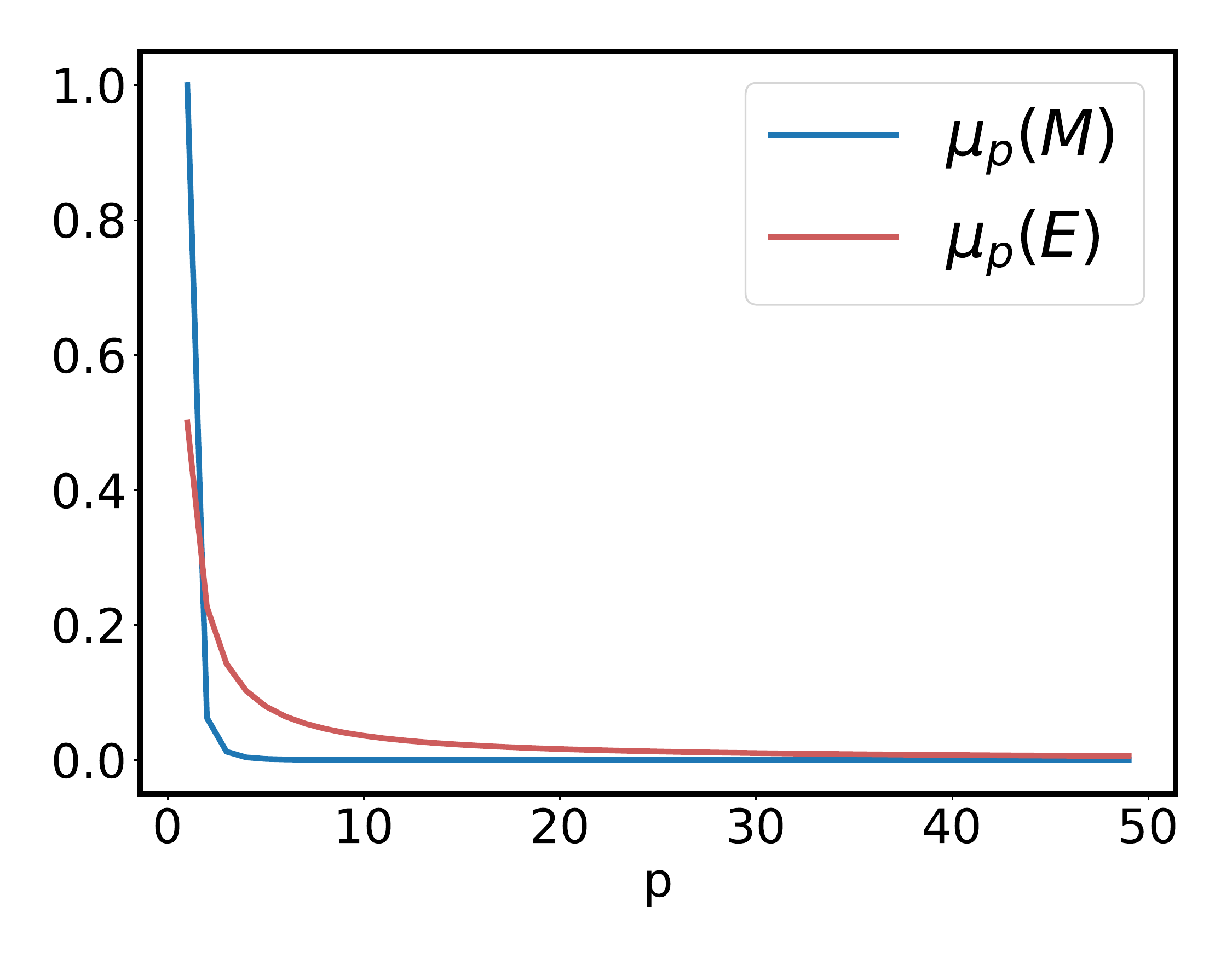}
\end{minipage}
\caption{Illustration of finite-rank setting with $r=5$, $\sigma_{c}^2 = 1, \sigma_{\epsilon}^2 = 0.1$ (Left) and polynomial setting with $r_c=3/2, r_h=5/4, r_e=8/7$ (Right). We plot $p \rightarrow \mu_p(M) = \langle v_p, M v_p\rangle_{\bmY}$ and $p \rightarrow \mu_p(E)=\langle v_p, E v_p\rangle_{\bmY}$.}
\label{fig:assumptions_examples}
\end{figure}
\vspace{0.2cm}

\begin{example}[finite-rank example] The standard low-rank regression setting (See Figure \ref{fig:assumptions_examples} left) corresponds to $\bmY=\R^{d}$, $C=\sigma_c^2I_{\bmH_x}$ with $\sigma_c^2>0$, $H = \sum_{i=1}^{r} v_i \otimes u_i$ with $r \in \mathbb{N}^*$, $E = \sigma_{\epsilon}^2 I_{\bmY}$ with $\sigma_{\epsilon}^2>0$, $(u_i)_i$, $(v_i)_i$ being orthonormal bases (ONB) of respectively $\bmH_x$ and $\bmY$. In this case, the assumptions are verified with $\alpha=0, \beta=0, \gamma=1$.
\end{example}

\begin{example}[polynomial example] In this paper, we study reduced-rank regression beyond low-rank setting. For instance, we can consider polynomial forms (See Figure \ref{fig:assumptions_examples} right) for $C = \sum_{i = 1}^{+\infty} i^{-r_c} u_i \otimes u_i$, $H = \sum_{i =1}^{+\infty} i^{-r_h} v_i \otimes u_i$, $E= 0.5 \times \sum_{i = 1}^{+\infty} i^{-r_e} v_i \otimes v_i$, with $(u_i)_i$ and $(v_i)_i$ being (ONB) of $\bmH_x$ and $\bmY$, respectively. In this case, the assumptions are verified with $\alpha = \frac{2}{2r_h + r_c}$, $c_1=\Tr(M^{\alpha})<2$, $\beta = \frac{r_c}{2r_h + r_c}$, $\gamma= 1-\frac{r_e}{2r_h + r_c}$.
\end{example}

\subsection{Main Result}\label{subsec:main_result}

Now, we present the main result of this work which is Theorem \ref{th}. Under Assumptions \ref{as:1}, \ref{as:2}, \ref{as:3}, \ref{as:4}, it provides a bound on the proposed estimator's excess-risk for a chosen $p=\rank(\hat P)$. 
\begin{restatable}[Learning bounds]{theorem}{th}\label{th} Let $\hat P\hat h$ be the proposed estimator in Eq. \eqref{eq:es} with $\rank(\hat P)= p$, built from $n$ independent couples $(x_i, y_i)_{i=1}^n$ drawn from $\rho$. Let $\delta \in [0,1]$. Under the Assumptions \ref{as:1}, \ref{as:2}, \ref{as:3}, \ref{as:4}, there exists constants $c_4, c_5, c_8>0$, $n_0 \in \mathbb{N}^*$ defined in the proof, and independent of $p, n, \delta$, such that, if $\mu_{p+1}(M) \geq  c_8 \log^8(\frac{8}{\delta}) n^{-\frac{1}{\beta +1 }}$ and $n \geq n_0$, then with probability at least $1-3\delta$,
\begin{align}\label{eq:bound}
    \E_x[\|\hat P\hat h(x) - h^*(x)\|^2_{\bmY}]^{1/2} &\leq \Big(c_4 \sqrt{p}n^{-1/4} + c_5 S_p(E)^{1/4}\Big)n^{-1/4}\log(n/\delta) + \sqrt{3c_1}\mu_{p+1}(M)^{1/2(1-\alpha)}
\end{align}
with $S_p(E) = \sum_{i=1}^p \mu_i(E)$.
\end{restatable}
The bound is the sum of two terms: the first one increases with $p$, the second one decreases with $p$. When $p = o(\sqrt{n})$, the first term is dominated by a term proportional to $S_{p}(E)^{1/4} \log(n/\delta)n^{-1/4}$, which should be compared to the dominating term of the kernel ridge estimator's bound $\Tr(E)^{1/4}n^{-1/4}$ (cf. Lemma \ref{lem:krr}): instead of the total amount of noise $\Tr(E)$, the reduced-rank estimator only incurs the quantity within the $p$ main components of $E$, plus a logarithmic term in $n$. The second term of the sum decays w.r.t $p$ at the speed of the eigenvalue decay rates of $\E_x[h^*(x) \otimes h^*(x)]$, modulo an exponent $1-\alpha$. Finally, the condition $\mu_{p+1}(M) \geq c_8 n^{-\frac{1}{\beta+1}}$ stems from the estimation error of $P$, and can translate into the existence of a plateau threshold $p^*$ from which the second term cannot decrease anymore (See \citet{Rudi2013OnTS}). Hence, the stronger is Assumption \ref{as:3}, the faster is the estimation of $\hat P$ and the divergence rate of the plateau threshold. We give here a sketch of the proof for the Theorem \ref{th}. The complete proof is detailed in Appendix \ref{app:proof}. 
\begin{sproof} The proof consists in decomposing the excess-risk of the estimator $\hat P \hat h$ as follows.
\begin{equation}
        \E_x[\|\hat P \hat h(x) - h^*(x)\|_{\bmY}^2]^{1/2} \leq \underbrace{\E_x[\|\hat P \hat h(x) - \hat Ph^*(x)\|_{\bmY}^2]^{1/2}}_{\text{regression error on a subspace}}
        + \underbrace{\E_x[\|\hat P h^*(x) - h^*(x)\|_{\bmY}^2]^{1/2}}_{\text{reconstruction error}}.
\label{eq:decomp}
\end{equation}
Then each right-hand term is bounded using a dedicated lemma given in the Appendix \ref{app:proof}. Lemma \ref{lem:krr_sub} bounds the regression error on the subspace defined by $\hat P$ (akin to a variance). Lemma \ref{lem:set} bounds the reconstruction error (akin to a bias). We exploit techniques and schemes similar to those used in \citep{Rudi2013OnTS, rudi2016generalization, Ciliberto2016, ciliberto2020general, luise2019leveraging} in order to prove these lemmas. Namely, L$^2$-norms of functions in $\bmH$ are expressed as Hilbert-Schmidt norms of Hilbert-Schmidt operators in $\bmY \otimes \bmH_x$. Relevant norms decompositions lead to study the deviation of the sample operators from the true operators $\E[y \otimes \phi(x)]$ and $\E[\phi(x) \otimes \phi(x)]$. For this purpose, Bernstein’s inequalities for the operator norm, or the Hilbert-Schmidt norm, of random operators between separable Hilbert spaces are applied \citep{tropp2012user}. The previously introduced assumptions of Section \ref{subsec:assumptions} play an important role in the proof of Lemma \ref{lem:set}, allowing to obtain faster learning rate for $\hat P$.
\end{sproof}

\begin{remarks}[Independence assumption on $\phi(x)$ and $\epsilon$] In this work, we assume that $\phi(x)$ is independent of $\epsilon$. This allows to keep a clear exposition of the proofs, by performing lighter mathematical derivations. Nevertheless, such assumptions is not exploited by the proposed method, and similar results hold without this assumption as we discuss in Appendix \ref{subsec:app_rk_indep}.
\end{remarks}

\subsection{Polynomial Eigenvalue Decay Rates}\label{subsec:th_poly}

In this subsection, we discuss under which setting reduced-rank ridge regression can be statistically and computationally advantageous in comparison to standard full-rank ridge regression. For this purpose, we apply Theorem \ref{th} considering polynomial eigenvalue decay rates for $M$ and $E$.

\begin{assumption}[polynomial eigenvalue decay rates]\label{as:5} $M$ and $E$ have polynomial eigenvalue decay rates with parameter $s>1$ and $e>1$, if there exist constants $a, A, b, B>0$ such that:
\begin{align}
    a p^{-s} \leq \mu_p(M) \leq Ap^{-s},\\
    b p^{-e} \leq \mu_p(E) \leq Bp^{-e}.
\end{align}
\end{assumption}

Parameters $s$ and $e$ characterize the shapes of the signal's and noise's  distributions in $\bmY$, and provide information complementary to the total amounts of variance $\Tr(M)$ and $\Tr(E)$. Moreover, notice that Assumption \ref{as:5} does not require an exact polynomial decay of the eigenvalues $\mu_k \propto k^{-r}$. In particular, one can define a measure of distortion of $\mu_k(M)$ and $\mu_k(E)$ from exact polynomial decays as the values $\frac{A}{a}$ and $\frac{B}{b}$, respectively. The greater are these ratios the greater are the distortions.

\begin{remarks}[Assumptions relationship] Assumption \ref{as:5} implies that Assumption \ref{as:2} holds with $c_1=\Tr(M^{\frac{2}{s}})$, and Assumption \ref{as:4} holds with $\gamma = 1 - \frac{e}{s}$ and $c_3 = A^{e/s}b^{-1}$.
\end{remarks}

Under the Assumptions \ref{as:1}, \ref{as:3}, and \ref{as:5} we derive the following corollary from Theorem \ref{th} in the special case of polynomial eigenvalue decay rates.

\begin{restatable}[Learning bounds (polynomial decay rates)]{corollary}{cor}\label{cor} Let $\delta \in~]0,1]$, $n \geq n_0$. Under Assumptions \ref{as:1}, \ref{as:3}, and \ref{as:5}, assuming $\frac{B}{b} \leq \theta$ with $\theta \geq 1$, then by taking only
\begin{align}
    p = c_9 (\log^8(\frac{8}{\delta}))^{-\frac{1}{s}} n^{\frac{1}{(\beta + 1)s}},
\end{align}
we have with probability at least $1-3\delta$:
\begin{align}\label{eq:pol_bound}
    \E_x[\|\hat P\hat h(x) - h^*(x)\|^2_{\bmY}]^{1/2} &\,\leq \,\, c_{10}(s, e) \, \log^{5/4}(\frac{n}{\delta}) \, n^{-1/4} \,+\, c_{11}(e) \, n^{-\frac{1}{2}\frac{1-2/s}{1+\beta}} \,  \log^8(\frac{8}{\delta}),
\end{align}
where $c_{10}(s,e) =  \tilde c_{10} \left(\frac{e(e-1)}{s}\right)^{1/4}\left(1 + \log\left(\frac{e}{e-1}\right)\right) $, $c_{11}(e) = \tilde c_{11} \left(1 + \log\left(\frac{e}{e-1}\right)\right)$. $\tilde c_{10}$, $\tilde c_{11}$, $n_0$, are constants independent of $n, \delta, s, e$, and $c_9$ is a constant independent of $n,\delta$, defined in the proofs.
\end{restatable}
As a first remark, note that the chosen components number $p$ of order $\mathcal{O}(n^{\frac{1}{(\beta+1)s}})$ is significantly smaller than $n$ when $s$ is big (concentrated signal). For instance, $s=2$ yields at most to $p = O(\sqrt{n})$. Then, notice that the bound is the sum of two terms. The first term is decaying in $O(n^{-1/4})$ modulo a logarithm term in $n$, and its multiplicative constant can be arbitrarily small when $e$ is small (spread noise), as $c_{10}(s,e) \xrightarrow[e \to 1^+]{} 0$. The decreasing rate of the second term varies within the open interval $]0, 1/2[$. The greater is $s$ and the smaller is $\beta$, the better is the rate.

\paragraph{Comparison with full-rank estimator's bound. }The bound provided in Eq. \eqref{eq:pol_bound} sheds light on the role of $M$ and $E$'s shapes, flat ($s,e \rightarrow 1^+$) or concentrated ($s,e \rightarrow +\infty$), in the performance of the reduced-rank estimator. At the opposite, remark that the full-rank ridge estimator's bound is dominated by a term of the form $c (\kappa + \|H\|_{\hs})\Tr(E)n^{-1/4} \log(\frac{4}{\delta})$ with $c>0$ a constant independent of $n,\delta,s,e$ (See Lemma \ref{lem:krr}). So, the ridge estimator is not impacted by the shapes of $M$ and $E$, but is only affected by the total amounts of signal $\|H\|_{\hs}$, and noise $\Tr(E)$. 

\paragraph{Favorable settings for reduced-rank. }Which situations are favorable to the proposed reduced-rank method? 
To simplify the discussion, let us not consider the terms $(1+\log(e/(e-1)))$ appearing in $c_{10}, c_{11}$. If $s$ is big enough and $\beta$ small enough then the right term of \eqref{eq:pol_bound} is $o(n^{-1/4})$ (e.g. $s=6$, $\beta=0$ gives $\mathcal{O}(n^{-1/3})$). So, for $n$ big enough, it remains to compare the left term of the bound with the dominating term of the ridge bound. When $e$ becomes close to $1^+$ the left term can be arbitrarily smaller than the ridge bound, because $c_{10}(s,e) \rightarrow 0$, while $c\Tr(E)$ is unchanged. Let be $q \in \mathbb{N}^*$. For the following family of settings:
\begin{align}
    & \beta < 1 -\frac{4}{s}, & e \in ]1, e^*(n, q)]
\end{align}
with $e^*(n, q) = \sup\{e \,\slash\, c_{10}(s,e) < \frac{c\Tr(E)^{1/4}}{q \log^{5/4}(n)}\}$, the reduced-rank bound is $q$ times smaller than the full-rank one, when $n$ is big enough. \\This gain is obtained because the projection yields to an important noise reduction and a small increase in bias. This can be think as a direct generalization of the low-rank regression setting.

\medskip
 
In the following corollary, we duly show that, despite the $(1+\log(e/(e-1)))$ terms, one can find settings $(n, s, e) \in \mathbb{N}^* \times \R^+ \times \R^+$ such that the learning bound \eqref{eq:pol_bound} is arbitrarily smaller than the kernel ridge estimator's one under the same assumptions on the learning problem.

\begin{corollary}[Statistical gain of reduced-rank regression]\label{lem:stat_gain} Let $\delta \in ]0,1]$ and $\epsilon>0$. If $\beta < 1$, then there exists a setting $s,e >1$, $n\in\mathbb{N}^*$, such that, under the assumptions of Corollary \ref{cor}, with probability at least $1-3\delta$,
\begin{align}
    \E_x[\|\hat P \hat h(x) - h^*(x)\|_{\bmY}^2]^{\frac{1}{2}} \leq \epsilon \times \Tr(E)^{1/4} \times n^{-1/4}.
\end{align}
\end{corollary}
\begin{proof} We exhibit such a setting $(n,s,e)$. We choose $(s, \beta)$ such that $\beta < 1 - \frac{4}{s}$. One can check that in this case $c_{11} n^{-\frac{1}{2}\frac{1-2/s}{1+\beta}}\log(n/\delta) = o(n^{-1/4})$, and also $c_{10} \left(\frac{e \theta}{\zeta(e) s}\times \log^5(\frac{n}{\delta})\right)^{1/4}n^{-1/4} = o(n^{-1/4})$ (when $e \rightarrow 1^+, n\rightarrow +\infty$, with $e \geq 1 + \frac{1}{n^a}$ for any $a > 0$). So, taking $n$ big enough we obtain the desired inequality. 
\end{proof}

Corollary \ref{lem:stat_gain} shows that a significant statistical gain is possible using reduced-rank regression, even if the support of $h^*(x)$ covers the entire output space $\bmY$, i.e. beyond the standard low-rank setting. Besides the statistical gain, reducing the rank of the predictions' space is of interest for reducing the computational complexity at prediction time.

As it will be presented in the application to structured prediction (See Section \ref{sec:struc_pred}), decoding predictions in surrogate approaches or simply computing mean squared errors require to calculate inner products between the predictions provided by the regression estimator and elements of the output space. 
In the following lemma, we analyze the complexity in time of such computations. Note that the same complexity holds for computing distances between predictions and elements of the output space. We consider the setting where the dimension of $\bmY$ is bigger than $n$ (e.g. infinite).

\begin{corollary}[Computational gain of reduced-rank regression]\label{lem:comp_gain} Let $\hat h: \bmX \rightarrow \bmY$ be a kernel ridge estimator trained on $n$ points. Let $\hat P : \bmY \rightarrow \bmY$ be a projection operator of rank $p$. Given $N$ output points $(y_i)_{i=1}^N$, computing the inner products $\left(\langle \hat P \hat h(x) ,\, y_i\rangle_{\bmY}\right)_{i=1}^N$ has a time and space complexity of order
    $\mathcal{O}(p(N+n))$
while computing the inner products $\left(\langle \hat h(x) ,\, y_i\rangle_{\bmY}\right)_{i=1}^N$ has a time complexity
    $\mathcal{O}(nN)$.
\end{corollary}
\begin{proof} In order to compute $\left(\langle \hat P \hat h(x) ,\, y_i\rangle_{\bmY}\right)_{i=1}^N$ one needs to compute 
\begin{align}
    \underbrace{\alpha(x)^T}_{(1, n)} \underbrace{(UY_{tr})^{T}}_{(n,p)} \underbrace{UY}_{(p,N)}
\end{align}
with $\alpha(x) = (K+n\lambda I)^{-1}k_x$, $k_x = (k(x, x_1), \dots, k(x, x_n))$,  $U=\sum_{i=1}^p e_i \otimes u_i$, where $(u_i)_{i=1}^p$ is an orthogonal basis of the range of $\hat P$, $(e_i)_{i=1}^p$ an orthogonal basis of $\R^p$, and $Y_{tr}$ is the operator with the $n$ training output points as columns, $Y$ the operator with the $N$ output points as columns. This costs $p(N + n)$ in time and space complexity.
In order to compute the $\left(\langle \hat h(x) ,\, y_i\rangle_{\bmY}\right)_{i=1}^N$ one needs to compute
\begin{align}
    \underbrace{\alpha(x)^T}_{(1, n)} \underbrace{K^y}_{(n,N)}
\end{align}
with $K^y$ the gram matrix between the $n$ training points and $N$ output points for the kernel $k_y(y,y') = \langle y,\, y'\rangle_{\bmY}$. This costs $nN$ in time and space complexity.
\end{proof}

Corollary \ref{lem:comp_gain} shows that a significant computational gain is possible when $N \gg p$ and $n \gg p$, as in this case $p(N +n) \ll nN$. Combining this result with Corollary \ref{lem:stat_gain} we conclude that, under the output regularity assumptions made, the proposed method offers both statistical and computational gains by projecting the ridge estimator onto an estimated linear subspace.

\begin{remarks}[Consequences for finite dimensional $\bmY$. ] The obtained results are not limited to the infinite dimensional setting and are still valuable when $\bmY=\R^d$. One can notice that in the finite dimensional case Assumptions \ref{as:2}, \ref{as:3}, and \ref{as:4} are always verified choosing the best exponents $\alpha = \beta =0, \gamma  =1$ (if $M, E \succ 0$), but it is at the price of very large constants $c_1, c_2$ and very small $c_3$, which make the bounds very large. In fact, it amounts to using the rough inequalities $\Tr(A) \leq d \times \|A\|_{\infty}$ and $A \preceq \frac{\mu_{1}(B)}{\mu_{d}(B)} B $ for any bounded operators $A, B$, thereby loosing information on the shape of $M$ and $E$. At the opposite, choosing $\alpha, \beta, \gamma$ such that the constants $c_1, c_2, c_3$ remain close to $1$ allows to obtain finer bounds, taking into account the signal/noise configuration, closed to the observed behaviors.
\end{remarks}

\paragraph{Take-home message. } The proposed reduced-rank regression estimator enjoys a statistical gain under more general assumptions than standard low-rank assumptions. As parameter $\lambda$, the rank $p$ acts as a regularization parameter whose impact should disappear when the size of the training sample increases, i.e. $p \xrightarrow[n\rightarrow +\infty]{} +\infty$. The settings where the proposed method performs better than the kernel ridge estimator require faster eigenvalue decay rates for $\E[h^*(x) \otimes h^*(x)]$ than for $\E[\epsilon \otimes \epsilon]$ (concentrated signal/diffuse noise). But this is not sufficient: Assumption \ref{as:3} with a sufficiently small $\beta$ ($\beta < 1 - \frac{4}{s}$) is also necessary to ensure a fast enough estimation of $P$. Last but not least, reducing the predicted outputs' dimension can also yield to substantial computational gains.

%% file: sections/struct_pred.tex
In this section, we develop an application of the reduced-rank estimator to structured prediction. The novel method fits into the generic framework of surrogate approaches for structured prediction and exploits an infinite dimensional embedding by the mean of a kernel.
We describe the algorithm and give learning bounds for the proposed structured prediction estimator.
\subsection{Surrogate Reduced-Rank Estimator for Structured Prediction} \label{subsec:setting_struct_pred}
 Structured prediction consists in solving a supervised learning task where the output variable is a structured object. Denoting $\bmZ$ the structured output space, a structured loss $\Delta: \bmZ \times \bmZ \rightarrow \mathbb{R}$ measures the discrepancy between a true output and a predicted output.  The goal of structured prediction is to minimize the following expected risk:
\begin{align}\label{eq:sp}
    R(f) = \E_{\rho}[\Delta(f(x), z)],
\end{align}
over a class of functions $f: \bmX \to \bmZ$, using a finite set $(x_i, z_i)_{i=1}^n$ independently drawn from an unknown distribution $\rho$ on $\bmX \times \bmZ$. In other words, if we note $f^*: \bmX \to \bmZ$ the minimizer of $R(f)$, the aim of learning is therefore to get an estimator $\hat{f}$ of $f^*$ based on the finite sample $(x_i, z_i)_{i=1}^n$ with the best possible statistical properties.

\paragraph{A surrogate approach: Output Kernel Regression}

We consider here the case when $\Delta$ is defined as a metric induced by a positive definite kernel $k_z$ acting over the structured output space $\bmZ$: 
\begin{align}\label{eq:loss}
    \Delta(z,z') = \|\psi(z) - \psi(z')\|^2_{\bmH_z}.
\end{align}
This boils down to embedding objects of $\bmZ$ into the Reproducing Kernel Hilbert Space associated to $k_z$ using the canonical feature map $\psi:\bmZ\rightarrow \bmH_z$ associated to $k_z$, and then consider the square loss over $\bmH_z$. Relying on the abundant literature about kernels on structured objects \citep{gartner2003survey}, this class of losses covers a wide variety of structured prediction problems.

However, learning directly $f$ through $\psi$ still raises an issue and a simple way to overcome it consists in seeking instead a {\bf surrogate} model $h:\bmX \to \bmH_z$ able to predict the embedded objects in the infinite dimensional space $\bmH_z$ and leverage the kernel trick in the output space. This approach is referred as Output Kernel Regression (OKR) \citep{weston2003kernel,geurts2006,brouard2016input}. The original structured prediction problem is then replaced by the following surrogate vector-valued regression problem stated in terms of the surrogate true risk:
\begin{align}\label{eq:sur}
    \min_{h: \bmX \to \bmH_z} \E_{\rho}[\|h(x) - \psi(z)\|^2_{\bmH_z}].
\end{align}
Assume $h^*$ is the function $x \to \E_y [\psi(z) | x]$ (solution  of Eq. \eqref{eq:sur}). Then at prediction time, one can retrieve a prediction in the original space $\bmZ$ through an appropriate decoding function $d: \bmH_z \to \bmZ$: 
\begin{align}\label{eq:struc_pred}
    z^* =f^{**}(x) := d \circ h^*(x) := \argmin\limits_{z \in \bmZ} \| h^*(x) - \psi(z)\|^2_{\bmH_z}.
\end{align}
 \begin{figure}[t]
    \centering
    \includegraphics[width=45mm]{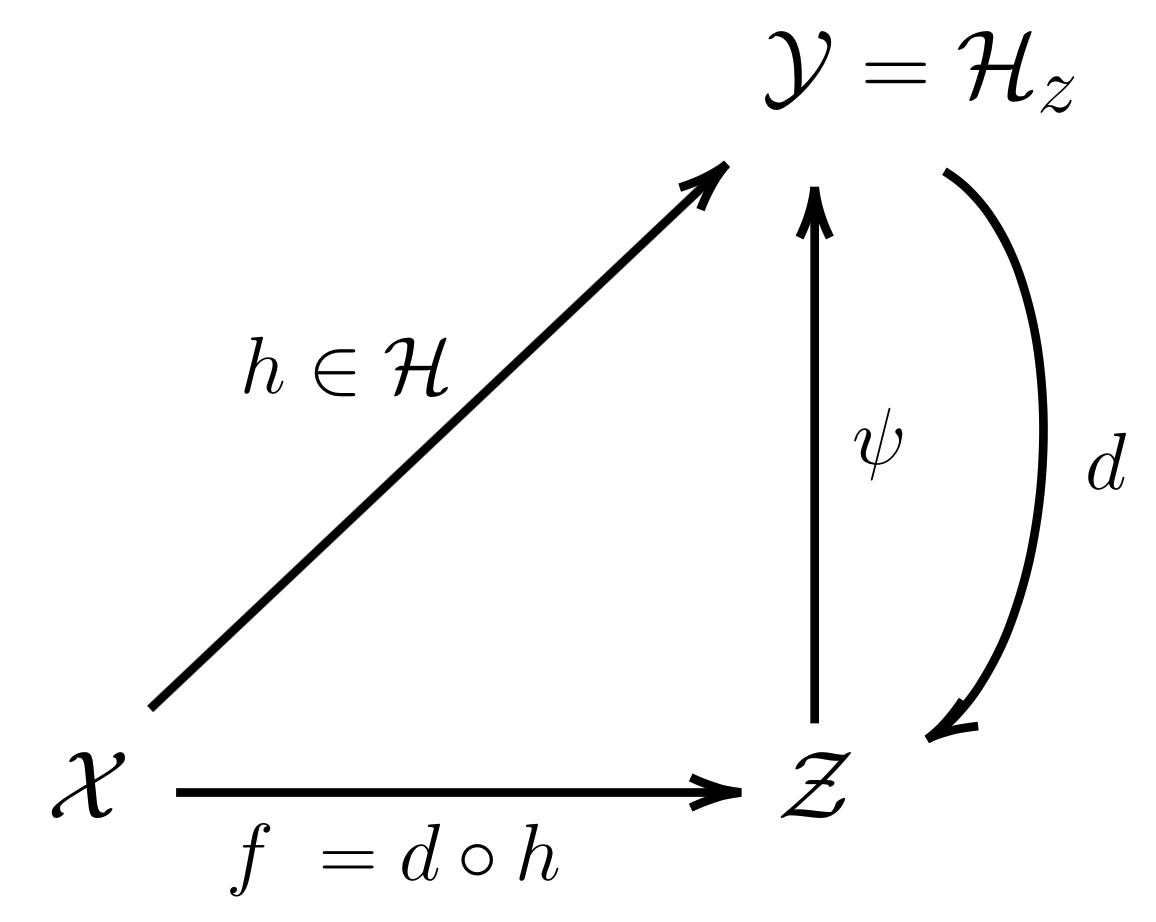}
    \caption{Schematic illustration of OKR.}\label{fig:okr}
    \vspace{0.4cm}
\end{figure}
The overall approach is illustrated on Fig. \ref{fig:okr}.

\citet{Ciliberto2016} have proved that $f^{**}$ solves exactly the original structured prediction problem, i.e. $f^{**}=f^{*}$. Fo that purpose, they have shown that $\Delta(z,z') = \|\psi(z) - \psi(z')\|^2_{\bmH_z}$ belongs to the wide family of Structure Encoding Loss Functions (SELF), as it can be written $z, z' \rightarrow \langle \gamma(z), \theta(z') \rangle_{\nu}$ with $\nu = \bmH_z \oplus \R \oplus \R$, $\gamma(z) = (\sqrt{2}\psi(z), \|\psi(z)\|^2_{\bmH_z}, 1)$, and $\gamma(z) = (-\sqrt{2}\psi(z'), 1, \|\psi(z')\|^2_{\bmH_z})$. 

Moreover, when providing an estimator $\hat{h}$ of $h^*$ using the training sample $(x_i, z_i)_{i=1}^n$, we benefit from the so called comparison inequality from \citet{Ciliberto2016}
\begin{align}\label{eq:comp}
    R(\hat{f}) - R(f^*) \leq c \times \E_x[\|\hat h(x) - h^*(x)\|_{\bmH_z}^2]^{1/2},
\end{align}
where $\hat{f} = d \circ \hat{h}$ and the constants $c$ and $Q$ are defined as: $c = 2\sqrt{2Q^2+ Q^4 + 1}$, and $Q=\sup_z\|\psi(z)\|_{\bmH_z}$.

\paragraph{Reduced-rank regression in structured prediction. } 

The OKR problem depicted in Eq. \eqref{eq:sur} can be solved in various hypothesis spaces and trees-based approaches \citep{geurts2006} as well as kernel methods \citep{weston2003kernel,geurts2006,brouard2011,kadri2013,laforgue20a} have been developed so far to tackle it. We focus here on Input Output Kernel Regression (IOKR), a method that exploits operator-valued kernels \citep{brouard2016input} and assumes that $h$ belongs to a vv-RKHS. In particular, IOKR-ridge solves the kernel ridge regression problem in Eq. \eqref{eq:kkr} with the following choice 
s: the output space is $\bmY:= \bmH_z$, the chosen operator-valued kernel writes as $K(x,x')= k(x,x')I_{H_z}$,and  the hypothesis space$\bmH$ is the vv-RKHS associated to $K$. Instantiating Eq. \ref{eq:sol_ridge_gen}, the solution to IOKR-ridge writes as:
\begin{align}\label{eq:iokr-ridge}
    \hat h(x) = \sum\limits_{i=1}^n \alpha_i(x) \psi(z_i),
\end{align}
where $\alpha_i$'s are defined according Eq. \ref{eq:sol_ridge_gen}.

In this section, we propose to solve the surrogate problem in Eq. \eqref{eq:sur} using our reduced-rank estimator based on the IOKR-ridge estimator. This gives rise to the definition of a novel structured output prediction $\hat f$:
\begin{align}\label{eq:struct-iokr-ridge}
    \hat f(x) := \argmin\limits_{z \in \bmZ} \| \hat P \hat h(x) - \psi(z)\|^2_{\bmH_z}.
\end{align}
Because of the comparison inequality Eq. \eqref{eq:comp}, the resulting structured predictor directly benefits from the learning bound on the least-squares problem.

\begin{restatable}[Excess-risk bound for the structured predictor]{theorem}{th2}\label{th:struct_pred} Let $\delta \in ]0,1]$, $n \geq n_0$. Under Assumptions \ref{as:1}, \ref{as:3}, and \ref{as:5}, assuming $\frac{B}{b} \leq \theta$ with $\theta \geq 1$, then by taking only
\begin{align}
    p = c_9 (\log^8(\frac{8}{\delta}))^{-\frac{1}{s}} n^{\frac{1}{(\beta + 1)s}}
\end{align}
then with probability at least $1-3\delta$
\begin{align}\label{eq:struct_bound}
    R(\hat f) - R(f^*) &\,\leq \,\, c \times \left(c_{10}(s, e) \, \log^{5/4}(\frac{n}{\delta}) \, n^{-1/4} \,+\, c_{11}(e) \, n^{-\frac{1}{2}\frac{1-2/s}{1+\beta}} \,  \log^8(\frac{8}{\delta})\right)
\end{align}
where $c_{10}(s,e) =  \tilde c_{10} \left(\frac{e(e-1)}{s}\right)^{1/4}\left(1 + \log\left(\frac{e}{e-1}\right)\right) $, $c_{11}(e) = \tilde c_{11} \left(1 + \log\left(\frac{e}{e-1}\right)\right)$. $\tilde c_{10}$, $\tilde c_{11}$, $n_0$, are constants independent of $n, \delta, s, e$ and $c_9$ is a constant independent of $n,\delta$, defined in the proofs.
\end{restatable}

The bound provided in Theorem \ref{th:struct_pred} is similar to the one of Corollary \ref{cor} modulo the multiplicative constant $c$, and thus the interpretation is the same. In particular, when $s$ is sufficiently big and $e, \beta$ sufficiently small,  we can obtain a significant statistical gain in comparison to the not projected estimator, as shown in Corollary \ref{lem:stat_gain}.

\subsection{Algorithms and Complexity Analysis}

To define the final reduced-rank IOKR-ridge estimator $\hat f$, one has to apply Algorithm \ref{algo:train} to compute all the parameters of $\hat{P}\hat {h}$ necessary to the decoding phase described in Algorithm \ref{algo:test}.

\paragraph{Complexity in time} At decoding/prediction time, one needs to compute $n_{te}$ times the prediction $\hat f (x_i)$, for the testing data points $(x_i)_{i=1}^{n_{te}}$. Each prediction requires to calculate the distances in Eq. \eqref{eq:struc_pred}. This is made possible by using the kernel trick, avoiding to compute the infinite dimensional vectors $\hat h(x)$ and $\psi(z)$. These computations cost $\mathcal{O}(n_{te}n |\bmZ|)$ in time, where $n$ and $|\bmZ| \in \mathbb{N}^*$ are the size of the training data set and the number of output candidates, respectively. Note that $|\bmZ|$ is typically very big in structured prediction. For instance, in multilabel classification with $d$ labels $|\bmZ| = \{0,1\}^d = 2^d$. In practice, one often chooses a subset of $\bmZ$ as a candidate set. Hence, the decoding phase badly scales with $n$, and in general is computationally expensive. Because of the projection onto a finite dimensional space, the proposed method can significantly alleviate these computations. When using $\hat P \hat h$ with $\hat P$ of rank $p$, the decoding time complexity reduces to $\mathcal{O}(n_{te} p |\bmZ|)$ as shown in Corollary \ref{lem:comp_gain}. Furthermore, the training phase consists in a matrix inversion for computing $\hat h$ plus a singular value decomposition for computing $\hat P$. Hence, the time complexity of the training algorithm without approximation is $\mathcal{O}(2n^3)$. It can still be reduced using efficient and theoretically grounded approximation methods for KRR and kernel principal component analysis developed in \citep{rudi2015less, rudi2016generalization, sterge2020gain}.

\newpage
\begin{algorithm}[t]\label{algo:train}
\setstretch{1.5}
\begin{algorithmic}
   \STATE {\textbf{Input:}} $K_x, K_z \in \R^{n \times n}$, $ \lambda \geq 0 $, $p \in \N^*$
   \STATE {\textbf{KRR estimation:}} $W = (K_x + n\lambda I)^{-1} \in \R^{n \times n}$
   \STATE {\textbf{Subspace estimation:}} 
   \STATE $K_h = WK_xK_zK_xW \in \R^{n \times n}$
   \STATE $\beta = 
            \begin{bmatrix}
            \vert & & \vert \\
            \frac{u_1}{\sqrt{\mu_1}}  & \dots  & \frac{u_p}{\sqrt{\mu_p}}   \\
            \vert & & \vert
            \end{bmatrix} \in \R^{n \times p} \gets SVD(K_h) = \sum_{l=1}^{n} \mu_l u_lu_l^T$
    \STATE {\textbf{Training outputs projection:}} 
    \STATE $K_{zh} = K_zWK_x \in \R^{n \times n}$
    \STATE $UY = K_{zh}\beta \in \R^{n \times p}$
   \STATE {\textbf{Return:}} $W$ (KRR coefficients), $\beta$ (projection coefficients), $UY$ (projected training outputs)
\end{algorithmic}
\vspace{0.4cm}
\caption{Reduced-rank IOKR-ridge - Training phase}
\vspace{0.4cm}
\end{algorithm}

\begin{algorithm}[t]\label{algo:test}
\vspace{1cm}
\setstretch{1.5}
   \caption{Reduced-rank IOKR-ridge - Decoding phase}
\begin{algorithmic}
   \STATE {\textbf{Input:}} $k_x^{te} \in \R^{n}$, $Z_{candidates} \in \R^{n_{c} \times d}$, $UY \in \R^{n \times p}$, $W \in \R^{n \times n}$
    \STATE {\textbf{Output candidates projection:}} 
    \STATE $K_{zh} = WK_xK_z^{tr/c} \in \R^{n \times n_c}$
    \STATE $UY_c = K_{zh}\beta \in \R^{n_c \times p}$
   \STATE {\textbf{Distances computation:}} 
   \STATE $\alpha = Wk_x^{te} \in \R^{n}$
   \STATE $Uh_{te} = UY^T\alpha \in \R^{p}$
   \STATE $S := "\langle \hat P \hat h(x_{te}),\, \psi(Z_{candidates}) \rangle_{\bmH_z}" = (Uh_{te})^T UY_c \in \R^{n_{c}}$
   \STATE $N := "\|\psi(Z_{candidates})\|^2_{\bmH_z}" = \left(K_z(z, z)\right)_{z \in Z_{candidates}} \in \R^{n_{c}}$
   \STATE $D = N - 2 S$
   \STATE {\textbf{1-NN prediction :}}
   \STATE $\hat i = \argmin_{i \in [1, n_c]} D_i$
   \STATE $\hat z = Z_{candidates}[\hat i] \in \bmY$
   \STATE {\textbf{Return:}} $\hat z$ (prediction)
\end{algorithmic}
\vspace{0.4cm}

\end{algorithm}

\begin{table}[!ht]
\vspace{1cm}
  \centering
  \begin{tabular}{lll}
    \toprule
    Algorithm & IOKR & Reduced-rank IOKR \\
    \midrule
     Training & $\bmO(n^3)$ & $\bmO(2n^3)$ \\
     Decoding & $\bmO(n_{te} n|\bmZ|)$ & $\bmO(n_{te} p|\bmZ|)$\\
    \bottomrule
  \end{tabular}
  \vspace{0.4cm}
\caption{Time complexity of IOKR versus reduced-rank IOKR.}
\label{time_complexity}
\end{table}

%% file: sections/applications.tex
We now carry out experiments with the methods proposed in this work. In Section \ref{subsec:exp_synth}, we illustrate our theoretical insights on synthetic least-squares problems. In Section \ref{subsec:exp_struct}, we test the proposed structured prediction method on three different problems:  image reconstruction, multi-label classification, and metabolite identification.

\subsection{Reduced-rank regression: statistical gain and importance of Assumption \ref{as:3}}\label{subsec:exp_synth}
We illustrate, on synthetic least-squares problems, the theoretical insights, given in Subsection \ref{subsec:th_poly}. For $d=300$, $\bmX = \bmH_x = \bmY = \R^d$, we choose $\mu_p(C) = \frac{1}{\sqrt{p}}$, $\mu_p(E) = \frac{0.2}{p^{1/10}}$. We draw randomly the eigenvector associated to each eigenvalue. We draw $H_0 \in \R^{d \times d}$ with independently drawn coefficients from the standard normal distribution. We consider two different optimums $H=H_0$ ($\beta=1$) and $H=(H_0CH_0)H_0$ ($\beta=1/3$). Then, we generate $n \in [10^2, \dots, 5 \times 10^3], n_{val}=1000, n_{test}=1000$ couples $(x,y)$ such that $x \sim \mathcal{N}(0, C)$, $\epsilon \sim \mathcal{N}(0, E)$, and $y = Hx + \epsilon$. We select the hyper-parameters of the three estimators $\hat h$, $P\hat h$, and $\hat P \hat h $ in logarithmic grids, with the best validation MSE. On the Figure \ref{fig:illustration_2} we plot the test MSE obtain by the three estimators for various $p$ and $n$, and for the two different optimums $H=H_0$ (left) and $H=(H_0CH_0)H_0$ (right). There exists for both $H$ (left/right) a minimum MSE w.r.t $p$ for $P \hat h$ below the MSE of $\hat h$ when $n$ is big enough: $P$ offers a valuable regularization of $\hat h$. Moreover, we observe that the selected $p$ increases when $n$ increases with a decreasing gain, following the provided bounds' behavior. Furthermore, we observe that because of the estimation error of $\hat P$, there is no gain for $\hat P \hat h$ when $H=H_0$, while when $H=(H_0CH_0)H_0$ there is a gain for $n$ big enough. This illustrates the faster convergence rate of $\hat P$ when $\beta$ is small.
\begin{figure}[t]
\centering
\begin{minipage}{.48\textwidth}
    \centering
    \includegraphics[width=1.\linewidth]{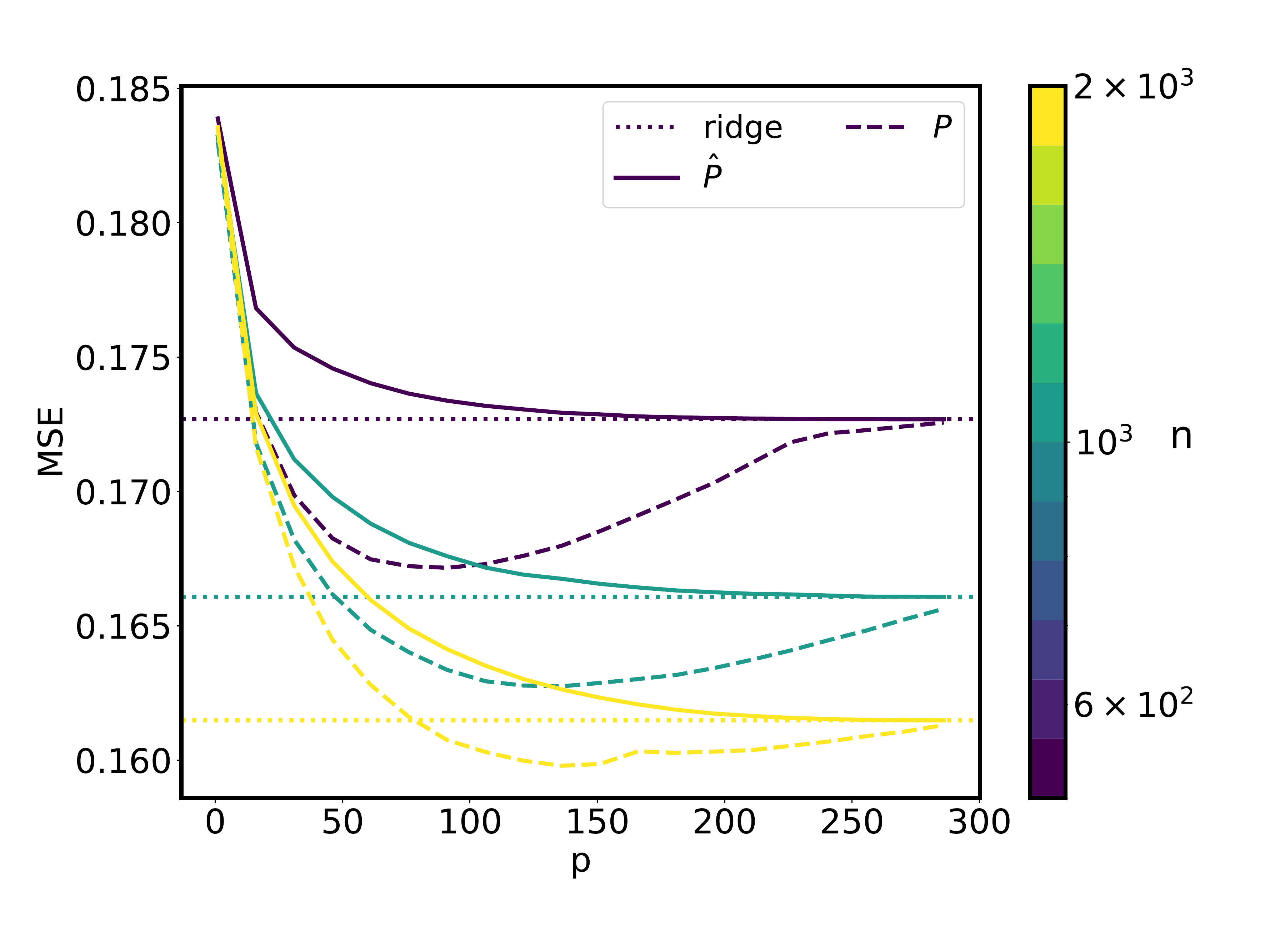}
\end{minipage}
\begin{minipage}{.48\textwidth}
    \centering
    \includegraphics[width=1.\linewidth]{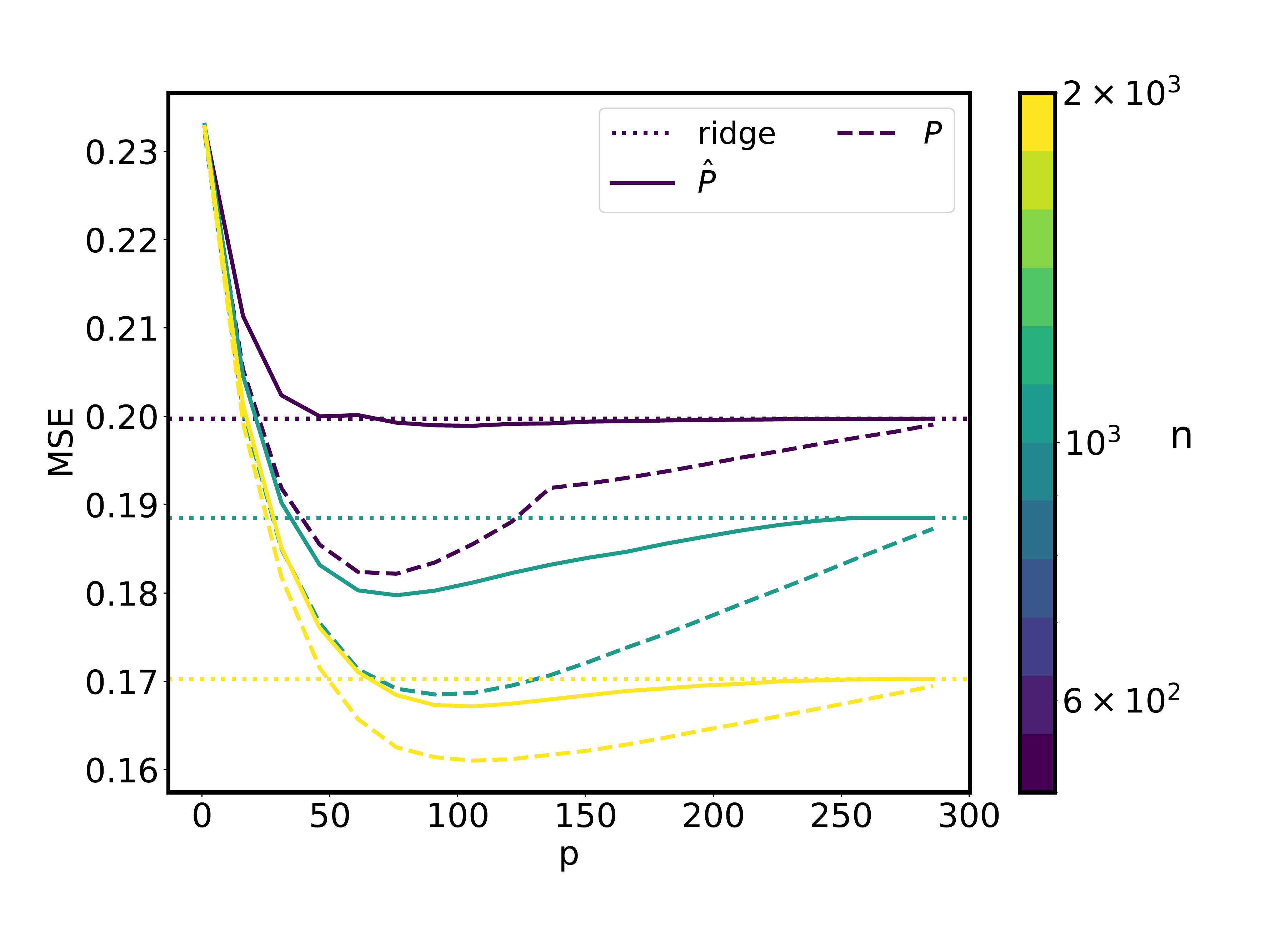}
\end{minipage}
\caption{Test MSE w.r.t $p$ ($x$ axis) and the quantity of training data $n$ (color bar), obtained with the optimal projection $P$ and its estimation $\hat P$, for various output source condition. (Left) Output source condition $\beta=1$, $H=H_0$. (Right) Output source condition $\beta=1/3$, $H=(H_0CH_0)H_0$.}
\label{fig:illustration_2}
\vspace{0.4cm}
\end{figure}

%% file: sections/new-experiments.tex
\subsection{Experiments on Structured Prediction }\label{subsec:exp_struct}

In this section, we assess the performance of the reduced-rank IOKR estimator calculated using Algorithms \ref{algo:train} and \ref{algo:test} proposed in Section \ref{sec:struc_pred} on three real-world structured prediction tasks: image reconstruction, multi-label classification, and metabolite identification. Our experiments show how reduced-rank regression can be advantageously used for surrogate methods in structured prediction in order to improve both statistical and computational aspects. In these experiments, we choose $\lambda_1=\lambda_2$ in order to reduce the quantity of hyperparameters. 

\paragraph{State of the art approaches}
For each task, we compared our reduced-rank method to relevant existing SOTA approaches. SPEN \citep{belanger2016}, a neural network learned by minimizing the structured hinge loss, is an Energy-Based Model (EBM), considered as a strong benchmark in the literature. Contrary to surrogate approaches, EBM involves the computation of the decoding phase during the training phase. Kernel Dependency Estimation (KDE) \citep{weston2003kernel} shares with IOKR the use of kernels in the input and output space with the following differences: in KDE, Kernal PCA is used to decompose the output feature vectors into $p$ orthogonal directions. Kernel ridge regression is then used for learning independently the mapping between the input feature vectors and each direction.
By applying KPCA on the outputs KDE aims at estimating the linear subspace of the output embedding $\psi(y)$ while the proposed reduced-rank estimator aims at estimating the linear subspace of the $h^*(x)$. Additionally, for the multi-label classification problem, we choose the exact setting of previous benchmark experiments (See for instance, \citep{gygli2017,lin2014}) and thus benefited from the collected results and comparison with other methods.

\subsubsection{Image Reconstruction}

\paragraph{Problem and data set. } The goal of the image reconstruction problem provided by \citet{weston2003kernel} is to predict the bottom half of a USPS handwritten postal digit (16 x 16 pixels), given its top half. The data set contains 7291 training labeled images and 2007 test images.

\paragraph{Experimental setting. } As in \citet{weston2003kernel} we used as target loss an RBF loss $\|\psi(y) - \psi(y')\|^2_{\bmH_y}$ induced by a Gaussian kernel $k$ and visually chose the kernel’s width $\sigma^2_{output} = 10$, looking at reconstructed images of the method using the ridge estimator (i.e. without reduced-rank estimation). We used a Gaussian input kernel of width $\sigma^2_{input}$. For the pre-image step, we used the same candidate set for all methods constituted with all the 7291 training bottom half digits. We considered $\lambda:= \lambda_1=\lambda_2$ for the proposed method. The hyper-parameters for all tested methods (including $\sigma^2_{input}, \lambda, p$, and SPEN layers' sizes) have been selected using logarithmic grids via 5 repeated random sub-sampling validation (80\%/20\%). 

\paragraph{Reduced-rank estimator for surrogate problem. } We start by evaluating the performance of the reduced-rank estimator in solving the Hilbert space valued least-squares problem described in Eq. \eqref{eq:sur}. We plot on Figure \ref{usps_mse} the test mean squared error of our estimator, and of the ridge estimator, w.r.t the quantity of training data $n$ from $n=500$ to $n=7000$. We observe that the reduced-rank estimator ($p<+\infty$) always performs better than the kernel ridge estimator ($p=+\infty$). Nevertheless, we see that this gain is smaller for small $n$ or big $n$. This is a typical behavior observed in our experiments, which can be interpreted as a difficulty in estimating $\hat P$ when $n$ is small, and the diminishing usefulness of regularization when $n$ increase. Indeed $p$ can be thought of as a regularization parameter exploiting a different regularity assumption than $\lambda$, but whose action, similarly to $\lambda$, should decrease when $n$ increases, such that $p \rightarrow +\infty$ when $n \rightarrow + \infty$.
\begin{figure}[t]
    \centering
    \includegraphics[width=75mm]{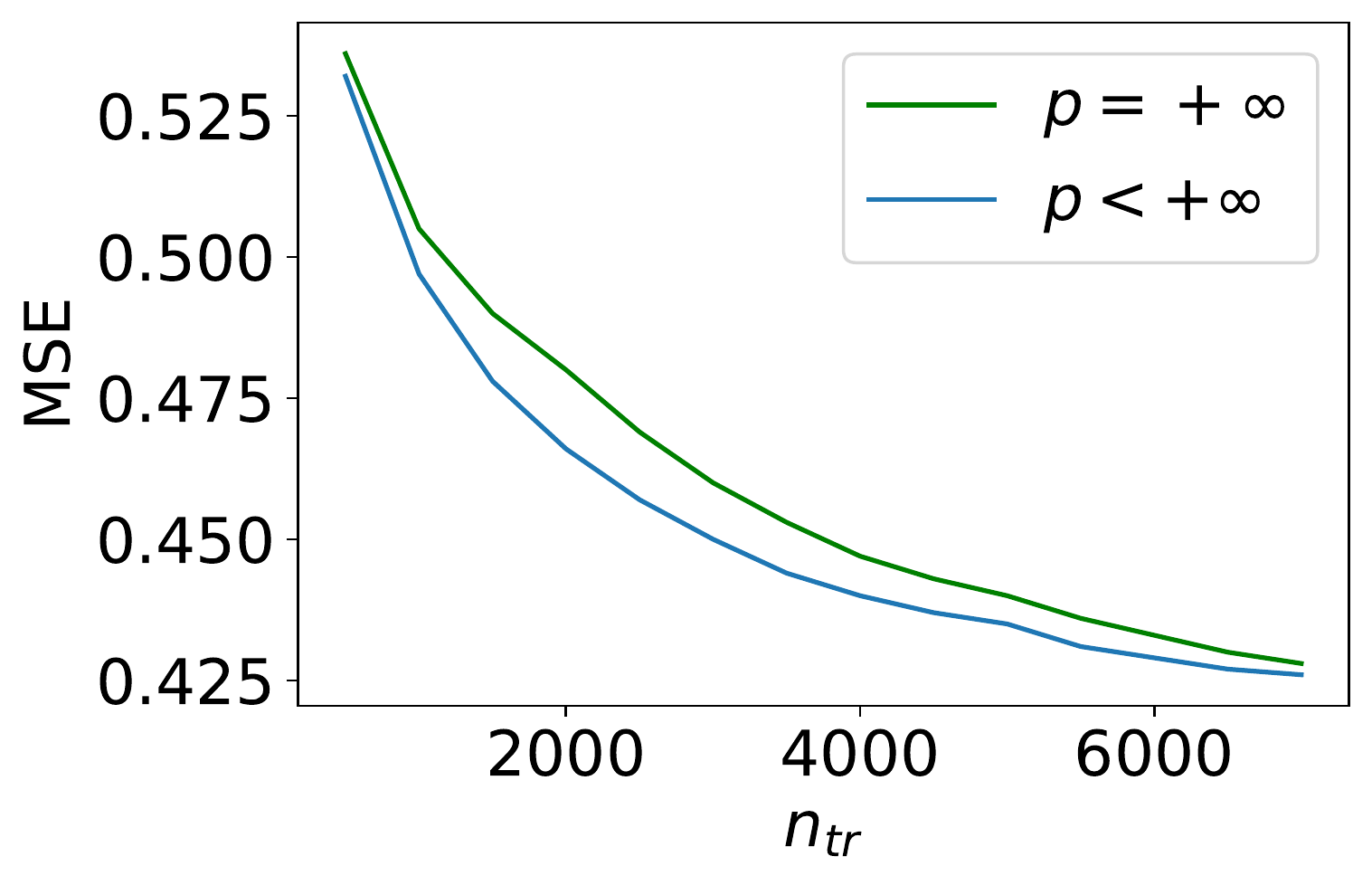}
    \caption{Test MSE of the proposed reduced-rank estimator ($p<+\infty$), and of the ridge estimator ($p=+\infty$) w.r.t $ n$ on the USPS problem.}\label{usps_mse}
    \vspace{0.4cm}
\end{figure}

\paragraph{Comparison with SOTA methods. } Then, in a second experiment, we compare the structured predictor (see Eq. \eqref{eq:struct-iokr-ridge}) using reduced-rank estimation, to state-of-the-art methods: SPEN \citep{belanger2016}, IOKR \citep{brouard2016input}, and Kernel Dependency Estimation (KDE) \citep{weston2003kernel}. We fix $n=1000$ where the reduced-rank estimation seems helpful, according to Figure \ref{usps_mse}. For SPEN we employed the standard architecture and training method described in the corresponding article (cf. supplements for more details). We evaluated the results in term of RBF loss (e.g. Gaussian kernel loss), as in \citet{weston2003kernel}. The obtained results are given in Table \ref{usps_table1}. Firstly, we see that SPEN obtains worse results than KDE, IOKR, and reduced-rank IOKR. Furthermore, note that the number of hyperparameters for SPEN (architecture and optimization) is usually larger than reduced-rank IOKR. Finally, notice that IOKR correspond to the proposed method with $p=+\infty$. Hence, this shows the benefit of exploiting output regularity thanks to reduced-rank estimation in structured prediction.

\begin{table}[ht!]
\vspace{0.4cm}
    \centering
    \begin{tabular}{lll}
    \toprule
    Method     &  RBF loss & p \\
    \midrule
    SPEN  & 0.801 $\pm$ 0.011 & 128\\
    KDE  & 0.764 $\pm$ 0.011 & 64\\
    IOKR & 0.751 $\pm$ 0.011 & $\infty$\\
    Reduced-rank IOKR & {\bf0.734 $\pm$} 0.011& 64 \\
    \bottomrule
    \end{tabular}
    \vspace{0.2cm}
    \caption{Test mean losses and standard errors for the proposed method, IOKR, KDE, and SPEN on the USPS digits reconstruction problem where $n= 1000$, and $n_{test}=2007$.}\label{usps_table1}
\end{table}

\subsubsection{Multi-label Classification}

\paragraph{Problem and data set. } Bibtex and Bookmarks \citep{katakis2008} are tag recommendation problems, in which the objective is to propose a relevant set of tags (e.g. url, description, journal volume) to users when they add a new Bookmark (webpage) or Bibtex entry to the social bookmarking system Bibsonomy. Corel5k is an image data set and the goal of this application is to annotate these images with keywords. Information on these data sets is given in Table \ref{multilabel_data set}.

\begin{table}[!ht]
\vspace{0.4cm}
  \centering
  \begin{tabular}{llllll}
    \toprule
    data set & $n$ & $n_{te}$ & $n_{features}$ & $n_{labels}$ & $\bar l$\\
    \midrule
     Bibtex & 4880 & 2515 & 1836 & 159 & 2.40\\
     Bookmarks & 60000 & 27856 & 2150 & 208 & 2.03\\
     Corel5k & 4500 & 499 & 37152 & 260 & 3.52\\
    \bottomrule
  \end{tabular}
  \vspace{0.2cm}
\caption{Multi-label data sets description. $\bar l$ denotes the averaged number of labels per point.}
\label{multilabel_data set}
\end{table}

\paragraph{Experimental setting. } For all multi-label experiments we used a Gaussian input and output kernels with widths $\sigma^2_{input}$ and $\sigma^2_{output} = \bar l$ , where $\bar l$ is the averaged number of labels per point. As candidate sets we used all the training output data. We measured the quality of predictions using example-based F1 score. We selected the hyper-parameters $\lambda$ and $p$ in logarithmic grids.

\paragraph{Comparison with SOTA methods. } We compare our method with several multi-label and structured prediction approaches including IOKR \citep{brouard2016input}, logistic regression (LR) trained independently for each label \citep{lin2014}, a two-layer neural network with cross entropy loss (NN) by \citep{belanger2016}, the multi-label approach PRLR (Posterior-Regularized Low-Rank) \citep{lin2014}, the energy-based model SPEN (Structured Prediction Energy Networks) \citep{belanger2016} as well as DVN (Deep Value Networks) \citep{gygli2017}. The results in Table \ref{tab:multilabel_comparison} show that surrogate methods (first two lines) can compete with state-of-the-art dedicated multilabel methods on the standard data sets Bibtex and Bookmarks. With Bookmarks ($n/n_{te}= 60000/27856$) we used a Nystr\"om approximation with 15000 anchors when computing $\hat h$ to reduce the training complexity, and we learned $\hat P$ only with a subset of 12000 training data. $\hat h$ decoding took about 56 minutes, and $\hat P \hat h$ decoding less than 4 minutes. With a drastically smaller amount of time, $\hat P \hat h$ (first line) achieves the same order of magnitude of F1 as $\hat h$ (line two) at a lower cost (see Table \ref{tab:experimental_computation_time}) and still has better performance than all other competitors.
\begin{table}[t]
  \centering
  \begin{tabular}{lll}
    \toprule
    Method     &  Bibtex & Bookmarks \\
    \midrule
    Reduced-rank IOKR  & 43.8  & 39.1\\
    IOKR & 44.0 & {\bf39.3}\\
    LR  & 37.2 & 30.7\\
    NN   & 38.9 & 33.8\\
    SPEN & 42.2 & 34.4\\
    PRLR & 44.2 & 34.9\\
    DVN  & {\bf 44.7} & 37.1\\
    \bottomrule
  \end{tabular}
    \vspace{0.2cm}
    \caption{Tag prediction from text data. $F_1$ score of reduced-rank IOKR compared to state-of-the-art methods. LR \citep{lin2014}, NN \citep{belanger2016}, SPEN \citep{belanger2016}, PRLR \citep{lin2014}, DVN \citep{gygli2017}. Results are taken from the corresponding articles.}
    \label{tab:multilabel_comparison}
\end{table}

\begin{table}[t]
  \centering
  \begin{tabular}{llll}
    \toprule
     & IOKR & Reduced-rank IOKR\\
    \midrule
    Bibtex & 2s/13s & 15s/4s\\
    Bookmarks & 465s/3371s & 617s/214s\\
    USPS & 0.1s/9s & 0.4s/1s\\
    \bottomrule
  \end{tabular}
    \vspace{0.2cm}
\caption{Fitting/Decoding computation time of IOKR compared to our method (in seconds)}
\label{tab:experimental_computation_time}
\end{table}

\paragraph{Small training data regime. } We evaluate the reduced-rank structured predictor in a setting where only a small number of training examples is known. For this setting, we consider only the $2000$ first couples $(x_i, y_i)$ of each multi-label data set as training set. Hyper-parameters have been selected using 5 repeated random sub-sampling validation (80\%/20\%) and the same $\lambda$ was used for IOKR. The results of this comparison are given in Table \ref{tab:multilabel_small}. We observe that the proposed reduced-rank structured predictor obtains higher F1 scores than the one using kernel ridge regression in this setup. This highlights the interest of our method in a setting where the data set is small in comparison to the difficulty of the task.

\begin{table}[t]
  \centering
    \begin{tabular}{llll}
    \toprule
    &  Bibtex & Bookmarks & Corel5k \\
    \midrule
 $n$ & $2000$ & $2000$ & $2000$\\
  $n_{te}$ & $2515$ & $2500$ & $499$\\
    \midrule
    IOKR &  35.9 & 22.9 & 13.7  \\
    Reduced-rank IOKR & {\bf 39.7} & {\bf25.9} & {\bf16.1} \\
    \bottomrule
  \end{tabular}
    \vspace{0.2cm}
    \caption{Test $F_1$ score of reduced-rank IOKR and IOKR on different multi-label problems in a small training data regime.}
    \label{tab:multilabel_small}
\end{table}

\paragraph{About the selected rank p. } We selected the rank  $p$ with integer logarithmic scales, ensuring that the selected dimensions were always smaller than the maximal one of the grids. From Table \ref{tab:multilabel_small} to Table \ref{tab:multilabel_comparison}, the selected dimension $p$ for Bibtex/Bookmarks are $80/30$, then $130/200$. In Table \ref{tab:multilabel_small} recall that we used a reduced number of training couples. Interpreting $p$ as a regularisation parameter, we see that when $n$ increases then the $p$ increases, i.e. the rank regularisation decreases.

\subsubsection{Metabolite Identification}

\paragraph{Problem and data set. } An important problem in metabolomics is to identify the small molecules, called metabolites, that are present in a biological sample. Mass spectrometry is a widespread method to extract distinctive features from a biological sample in the form of a tandem mass (MS/MS) spectrum. The goal of this problem is to predict the molecular structure of a metabolite given its tandem mass spectrum. The molecular structures of the metabolites are represented by fingerprints, that are binary vectors of length $d = 7593$. Each value of the fingerprint indicates the presence or absence of a certain molecular property. Labeled data are expensive to obtain, and despite the problem complexity only $n=6974$ labeled data are available. State-of-the-art results for this problem have been obtained with the IOKR method by \citet{brouard2016fast}. The median size of the candidate sets is 292, and the biggest candidate set is of size 36918. Hence, the metabolite identification data set is characterized by high-dimensional complex outputs, a small training set, and a very large number of candidates.

\paragraph{Experimental setting. } We adopt a similar numerical experimental protocol (5-CV Outer/4-CV Inner loops) than in \citet{brouard2016fast}, probability product input kernel for mass spectra, and Gaussian-Tanimoto output kernel on the molecular fingerprints (with parameter $\sigma^2=1$).  We selected the hyper-parameters $\lambda, p$ in logarithmic grids using nested cross-validation with 5 outer folds and 4 inner folds.

\paragraph{Improved prediction with reduced-rank estimation .} We compare the proposed reduced-rank structured predictor with SPEN, and with the state-of-the art method on this problem IOKR (which corresponds to our method with $p=+\infty$). The result are given in Table \ref{metabolite_table}. We observe that reduced-rank IOKR improved upon plain IOKR, in this context of supervised learning with complex outputs and a small training data set.

\begin{table}[ht]
  \centering
  \begin{tabular}{llll}
    \toprule
    Method     &  MSE  & Tanimoto-Gaussian loss & Top-k accuracies\\
    & & & $k=1$ $|$ $k=5$ $|$ $k=10$\\
    \midrule
    SPEN & $\:-\:$ & $0.537 \pm 0.008$ &  $25.9\% \,|\, 54.1\% \,|\, 64.3\%$  \\
    IOKR & $0.781 \pm 0.002$ & $0.463 \pm 0.009$ &  $29.6 \%\,|\, 61.1 \%\,|\, 71.0 \%$  \\
    Reduced-rank IOKR  & $\mathbf{0.766 \pm 0.003}$ & $\mathbf{0.459 \pm 0.010}$ &  $\mathbf{30.0 \%\,|\, 61.5 \%\,|\, 71.4 \%}$   \\
    \bottomrule
  \end{tabular}
    \vspace{0.2cm}
    \caption{Test mean losses and standard errors for the metabolite identification problem. SPEN MSE in $\bmH_z$ is not defined as predictions are directly done in $\bmZ$.}
    \label{metabolite_table}
      \vspace{0.4cm}
\end{table}

%% file: sections/proofs.tex
In this section we prove Theorem \ref{th} and Corollary \ref{cor}. The proofs are organized as follows:
\begin{itemize}
    \item Appendix \ref{subsec:nota_defi} introduces some necessary notations and definitions.
    \item Appendix \ref{subsec:krr_subspace} provides the proof for bounding $\E[\|\hat P(\hat h(x) - h^*(x))\|_{\bmY}^2]$ (Lemma \ref{lem:krr_sub}).
    \item Appendix \ref{subsec:supervised_subspace_learning} provides the proof for bounding $\E[\|\hat Ph^*(x) - h^*(x)\|_{\bmY}^2]$ (Lemma \ref{lem:set}).
    \item Appendix \ref{subsec:th} provides the proof for bounding $\E[\|\hat P\hat h(x) - h^*(x)\|_{\bmY}^2]$ (Theorem \ref{th}) using Lemmas \ref{lem:krr_sub} and \ref{lem:set}.
    \item Appendix \ref{subsec:corr} provides the proof for the Corollary \ref{cor} using Theorem \ref{th}.
    \item Appendix \ref{subsec:aux} gives some technical results used in the proofs.
    \item Appendix \ref{subsec:rk_indep_phi_psi} discusses the assumption that $\phi(x)$ and $\epsilon$ are independent.
\end{itemize}

\input{proofs/notations}

\subsection{KRR Error on a Subspace}\label{subsec:krr_subspace}

\input{proofs/krr_subspace}

\subsection{Supervised Subspace Learning}

\input{proofs/supervisedsubspacelearning}\label{subsec:supervised_subspace_learning}

\subsection{Theorem}\label{subsec:th}

\input{proofs/theorem}

\subsection{Corollary}\label{subsec:corr}

\input{proofs/corollary}

\subsection{Auxiliary Results}\label{subsec:aux}
In this section, we give four auxiliary results:
\begin{itemize}
    \item A bound on the KRR estimator which monitors the role of the total amount of noise $\Tr(E)$.
    \item A Bernstein inequality for bounded operator and the operator norm.
    \item A bound on $\|H_n - H\|_{\infty}$, used in the proof of Lemma \ref{lem:set}.
    \item Some properties of Löwner’s partial ordering
\end{itemize}

\input{proofs/fullrank_krr}

\input{proofs/conc_op}

\input{proofs/aux_proofs}

\input{proofs/lowner_properties}

\subsection{About the Independence of \texorpdfstring{$\phi(x)$}{} and \texorpdfstring{$\epsilon$}{}}\label{subsec:rk_indep_phi_psi}

\input{proofs/remark_independence_phi_eps}\label{subsec:app_rk_indep}

%% file: proofs/notations.tex
\subsection{Notations and Definitions}\label{subsec:nota_defi}

In the following we consider $\bmX$ to be a Polish space, and $\bmY$ a separable Hilbert space. We define here the ideal operators that we will use in the following
\begin{itemize}
    \item The feature map $\phi: \bmX \rightarrow \bmH_x$, $\forall x \in \bmX$, $\phi(x) = k(x, .) $, with $\|\phi(x)\|_{\bmH_x} \leq \kappa$ with $\kappa > 0$.
    \item The target $h^*(.) \in \bmH = \E_{y|.}[y]$, and $Q>0$ such that $\forall y \in \bmY, \|y\|_{\bmY} \leq Q$.
    \item $S : f \in \bmH_x \rightarrow \langle f, \phi(.) \rangle_{\bmH_x} \in L^2(\bmX, \rho_\bmX)$
    \item $Z : y \in \bmY \rightarrow \langle y, h^*(.) \rangle_{\bmY} \in L^2(\bmX, \rho_\mathcal{X})$
\end{itemize}
and their empirical counterparts
\begin{itemize}
    \item The KRR estimator $\hat h(.) \in \bmH $ trained with $n$ couples $(x_i, y_i)_{i=1}^n$
    \item $S_n : f \in \bmH_x \rightarrow \frac{1}{\sqrt{n}} (\langle f, \phi(x_i) \rangle_{\bmH_x})_{1 \leq i \leq n} \in \R^n$
    \item $Z_n : y \in \bmY \rightarrow \frac{1}{\sqrt{n}} (\langle y, y_i \rangle_{\bmY}))_{1 \leq i \leq n} \in \R^n$
\end{itemize}
From there, we can define the following covariance operators
\begin{itemize}
    \item $C = \E_x[\phi(x) \otimes \phi(x)] = S^*S$
    \item $V = \E_y[y \otimes y]$ 
    \item $M = \E_x[h^*(x) \otimes h^*(x)]$
    \item $Z^*S = \E_{x,y}[y \otimes \phi(x)]$
\end{itemize}
and their empirical counterparts
\begin{itemize}
    \item $C_n = \frac{1}{n}\sum\limits_{i=1}^n \phi(x_i) \otimes \phi(x_i)$
    \item $V_n = \frac{1}{n}\sum\limits_{i=1}^n y_i \otimes y_i$
    \item $M_n =  \frac{1}{n}\sum\limits_{i=1}^n \hat h(x_i) \otimes \hat h(x_i)$
    \item $Z_n^* S_n = \frac{1}{n}\sum\limits_{i=1}^n y_i \otimes \phi(x_i)$
\end{itemize}
From Lemmas 16 and 17 in \citet{Ciliberto2016} we recall that we have
\begin{itemize}
    \item $h^*(.)= H\phi(.)$ with $H = Z^*SC^\dagger \in \bmY \otimes \bmH_x$
    \item $\hat h(.) = H_n\phi(.)$ with $H_n = Z_n^*S_n (C_n + \lambda I)^{-1} \in \bmY \otimes \bmH_x $
    \item $M = HCH^*$
    \item $M_n = H_nC_nH_n^*$
\end{itemize}

%% file: proofs/krr_subspace.tex
In this subsection we prove a bound on the kernel ridge regression error on the subspace defined by $\hat P$:
\begin{align}\label{eq:etop}
    \E_x[\|\hat P\hat h(x) - \hat Ph^*(x)\|_{\bmY}^2]^{1/2} = \|\hat P(H_n - H)S^*\|_{\hs}.
\end{align}
Equation \eqref{eq:etop} is obtained by definition of the operators $H_n, H, S$ (see e.g. \citet{Ciliberto2016}).

In order to bound \eqref{eq:etop}, one can not directly apply standard learning bounds for kernel ridge estimator on the learning problem $(x, \hat Py)$ with $(x,y) \sim \rho$, as $\hat P$ depends on the training data. That is why we will decompose \eqref{eq:etop} as
\begin{align}\label{eq:decom}
    \|\hat P(H_n - H)S^*\|_{\hs} \leq \|\hat P (A+tI)^{1/2}\|_{\hs} \times \|(A+tI)^{-1/2}(H_n - H)S^*\|_{\infty}
\end{align}
with a well chosen operator $A: \bmY \rightarrow \bmY$.

As a first step, we give a bound on the KRR estimator excess-risk with respect to the operator norm.

\begin{lemma}[Bound $\|(H_n- H)S^*\|_{\infty}$]\label{lem:krr_op} Let $k: \bmX \times \bmX \rightarrow \R$ be a bounded kernel with $\forall x \in \bmX, k(x,x) \leq \kappa^2$. Let $\rho$ be a distribution on $\bmX \times \bmY$ such that its marginal w.r.t $y$ is supported on the ball $\|y\|_{\bmY} \leq Q$. Let $\hat h = H_n \phi(.)$ be the KRR estimator trained with $n$ independent couples drawn from $\rho$, and regularization parameter $\lambda_2> \frac{9\kappa^2}{n} \log(\frac{n}{\delta})$. Let $\delta \in [0,1]$. Then, under Assumption \ref{as:1}, $H_nS^* - HS^* = A_1 + A_2$, with
\begin{align}
    &A_1 := Z_n^*S_n(C_n + \lambda_2 I)^{-1}S^* - HC_n(C_n + \lambda_2 I)^{-1}S^*\\
    &A_2 := HC_n(C_n + \lambda_2 I)^{-1}S^* - HS^*
\end{align}
and with probability at least $1-2\delta$
\begin{align}
    &\|A_1\|_{\infty} \leq \sqrt{\frac{24\eta(Q^2 + \|E\|_{\infty}\lambda_2^{-1}\kappa^2)}{n}} + \frac{8\kappa Q\eta}{3\sqrt{\lambda_2} n};
    &\|A_2\|_{\infty} \leq \sqrt{2} \sqrt{\lambda_2} \|H\|_{\infty}
\end{align}
with $\eta = \log(\frac{4(\frac{2\Tr(C)}{\lambda_2} + \frac{\Tr(E)}{\|E\|_{\infty}})}{\delta})$, $E=\E[\epsilon \otimes \epsilon]$, $\epsilon = y - h^*(x)$, $R=\|H\|_{\hs}$.
\end{lemma}

\begin{proof}

\paragraph{Decomposition.} The decomposition $H_nS^* - HS^* = A_1 + A_2$ is obtained noticing that we have $H_n= Z_n^*S_n(C_n + \lambda_2 I)^{-1}$ (See section \ref{subsec:nota_defi}).

\vspace{1em}
\paragraph{1. Bound $\|A_1\|_{\infty}$.} We have
\begin{align}
    \|A_1\|_{\infty} &\leq \|(Z_n^*S_n- HC_n)(C + \lambda_2 I)^{-1/2}\|_{\infty} \times \|(C + \lambda_2 I)^{1/2}(C_n + \lambda_2 I)^{-1}S^*\|_{\infty}
\end{align}

\paragraph{1.1 Bound $\|(Z_n^*S_n- HC_n)(C + \lambda_2 I)^{-1/2}\|_{\infty}$.} We define 
\begin{align}
    \xi_i = \epsilon_i \otimes \phi(x_i)(C + \lambda_2 I)^{-1/2}
\end{align}
with $\epsilon_i = y_i - h^*(x_i)$. In this way, 
\begin{align}
    \|(Z_n^*S_n- HC_n)(C + \lambda_2 I)^{-1/2}\|_{\infty}= \|\frac{1}{n} \sum_{i=1}^n \xi_i\|_{\infty}
\end{align}
We aim at applying the Bernstein inequality given in Theorem 14 to the random linear
operator $\xi$. So, we define
\begin{align}
    T &= 2\kappa Q \lambda_2^{-1/2} \geq \|\xi_i\|_{\infty},\\
    \sigma^2 &= \max(\|\E[\xi\xi^*]\|_{\infty}, \|\E[\xi^*\xi]\|_{\infty}),\\
    d &= \Tr(\E[\xi^*\xi] + \E[\xi\xi^*]) / \sigma^2.
\end{align}
Note that $\|\epsilon\| \leq \|y\|_{\bmY} + \|h^*(x)\|_{\bmY} \leq 2Q$, and $\|\phi(x)\| \leq \kappa$. Then, we have
\begin{align}
    \|\E[\xi\xi^*]\|_{\infty} &= \|\E[\epsilon_i \otimes \epsilon_i \times \langle \phi(x_i),\, (C + \lambda_2 I)^{-1}\phi(x_i)\rangle_{\bmH_x}]\|_{\infty}\\
    &\leq \|\E[\epsilon \otimes \epsilon]\|_{\infty} \times \frac{\kappa^2}{\lambda_2}.
\end{align}
and
\begin{align}
    \|\E[\xi^*\xi]\|_{\infty} &= \|(C + \lambda_2 I)^{-1/2} C (C + \lambda_2 I)^{-1/2}\|_{\infty} \times \E[\|\epsilon \|_{\bmY}^2]\\
    &\leq 4Q^2.
\end{align}
Moreover, if $\lambda_2 < \|C\|_{\infty}$,
\begin{align}
    d \leq \frac{\Tr(\E[\xi^*\xi])}{\|\E[\xi^*\xi]\|_{\infty}} +  \frac{\Tr(\E[\xi\xi^*])}{\|\E[\xi\xi^*]\|_{\infty}} \leq \frac{2\Tr(C)}{\lambda_2} + \frac{\Tr(E)}{\|E\|_{\infty}}.
\end{align}
Thus, by applying the Bernstein inequality given in Theorem \ref{thm:tropp}, we have
\begin{align}
    \|(Z_n^*S_n- HC_n)(C + \lambda_2 I)^{-1/2}\|_{\infty} \leq \sqrt{\frac{2\eta(4Q^2 + \|E\|_{\infty}\kappa^2 \lambda_2^{-1})}{n}} + \frac{4\kappa Q\lambda_2^{-1/2}\eta}{3n}
\end{align}
where $\eta = \log(\frac{4(\frac{2\Tr(C)}{\lambda_2} + \frac{\Tr(E)}{\|E\|_{\infty}})}{\delta})$, $E=\E[\epsilon \otimes \epsilon]$.

\paragraph{1.2 Bound $\|(C + \lambda_2 I)^{1/2}(C_n + \lambda_2 I)^{-1}S^*\|_{\infty}$.} We apply Lemma B.6 in \citet{ciliberto2020general}, with $\lambda_2 \geq \frac{9\kappa^2}{n} \log(\frac{n}{\delta})$, and get with probability at least $1-\delta$, 
\begin{align}\label{eq:bound_c}
    \|(C + \lambda_2 I)^{1/2}(C_n + \lambda_2 I)^{-1}S^*\|_{\infty} \leq \|(C_n + \lambda_2 I)^{-1/2}(C+\lambda_2)^{1/2}\|_{\infty}^2 \leq 2.
\end{align}

\vspace{2em}
Finally, we have
\begin{align}
    \|A_1\|_{\infty} \leq \sqrt{\frac{24\eta(Q^2 + \|E\|_{\infty}\kappa^2 \lambda_2^{-1})}{n}} + \frac{8\kappa Q\lambda_2^{-1/2}\eta}{3n}.
\end{align}

\vspace{2em}
\paragraph{Bound $\|A_2\|_{\infty}$. } We have
\begin{align}
    \|A_2\|_{\infty} &= \|H(C_n(C_n + \lambda_2 I)^{-1}-I)S^*\|_{\infty}\\
    &= \|H(-\lambda_2 (C_n+\lambda_2 I)^{-1})S^*\|_{\infty}\\
    &\leq \lambda_2 \|H\|_{\infty} \|(C_n+\lambda_2 I)^{-1}S^*\|_{\infty}
\end{align}
and 
\begin{align}
    \|(C_n+\lambda_2 I)^{-1}S^*\|_{\infty} &\leq  \lambda_2^{-1/2} \|(C_n+\lambda_2 I)^{-1/2}S^*\|_{\infty}\\
    &= \lambda_2^{-1/2} \|(C_n+\lambda_2 I)^{-1/2}C^{1/2}\|_{\infty}\\
    &\leq \lambda_2^{-1/2} \|(C_n+\lambda_2 I)^{-1/2}(C+\lambda_2)^{1/2}\|_{\infty}\\
    &\leq \sqrt{2} \lambda_2^{-1/2}
\end{align}
because $\|(C_n + \lambda_2 I)^{-1/2}(C+\lambda_2)^{1/2}\|_{\infty}^2 \leq 2$ from Equation \eqref{eq:bound_c}.

\vspace{2em}
Finally, we have 
\begin{align}
    \|A_2\|_{\infty} &=  \sqrt{2} \sqrt{\lambda_2} \|H\|_{\infty}.
\end{align}

\paragraph{Conclusion. } The bound on $\|(H_n - H)S^*\|_{\infty}$ is obtained by summing the two bounds on $\|A1\|_{\infty}$ and $\|A2\|_{\infty}$.
\end{proof}

We are now ready to prove a bound on the excess-risk of the ridge estimator on the random subspace defined by $\hat P$, namely $\|\hat P(H_n-H)S^*\|_{\hs}$.
\begin{restatable}[KRR excess-risk on a subspace]{lemma}{lemkrrsub}\label{lem:krr_sub} Let $k: \bmX \times \bmX \rightarrow \R$ be a bounded kernel with $\forall x \in \bmX, k(x,x) \leq \kappa^2$. Let $\rho$ be a distribution on $\bmX \times \bmY$ such that its marginal w.r.t $y$ is supported on the ball $\|y\|_{\bmY} \leq Q$. Let $\hat h$ be the KRR estimator trained with $n$ independent couples drawn from $\rho$. Let $\delta \in [0,1]$. Define $S_p(E) = \sum_{i=1}^p \mu_i(E)$. Then, under the Assumptions \ref{as:1}, \ref{as:3}, \ref{as:4}, taking for $n$ big enough $\lambda_2 = \max(S_p(E)^{1/2} n^{-1/2}, n^{-1}, \frac{9\kappa^2}{n} \log(\frac{n}{\delta}))$, then with probability at least $1-2\delta$
\begin{align*}
    \E_x[\|\hat P \hat h(x) - \hat Ph^*(x)\|_{\bmY}^2]^{1/2} &\leq \Big(c_4 \sqrt{p}n^{-1/4} + c_5 S_p(E)^{1/4}\Big)n^{-1/4}\log(n/\delta) 
\end{align*}
with $c_4 = (7Q + 4\kappa Q + 2\|H\|_{\hs}(1 + 3\kappa))(1 +  c_6)$, $c_5=10\sqrt{(1+c_6)}\kappa \|E\|_{\infty}^{1/2} +  2\|H\|_{\hs}$, $c_6 = \log(8(\frac{\Tr(C)}{\|E\|_{\infty}^{1/2}} + \frac{\Tr(E)}{\|E\|_{\infty}}))$.
\end{restatable}

\begin{proof}

\paragraph{Decomposition. } We decompose $\|\hat P(H_n-H)S^*\|_{\hs}$ as follows
\begin{align}
    \|\hat P(H_n-H)S^*\|_{\hs} \leq \|\hat P A_1\|_{\hs} +  \|\hat P A_2\|_{\hs}
\end{align}
with $A_1, A_2$ defined above in Lemma \ref{lem:krr_op}.
Then, let be $t_1, t_2>0$,
\begin{align}
    &\|\hat P A_1\|_{\hs} \leq \|\hat P (E +t_1I)^{1/2}\|_{\hs} \times \|(E +t_1I)^{-1/2}A_1\|_{\infty}\\
    &= \Tr(\hat P(E +t_1I))^{1/2} \times \|(E +t_1I)^{-1/2}A_1\|_{\infty}\\
    &\leq \sqrt{S_p(E) + pt_1} \times \|(E +t_1I)^{-1/2}A_1\|_{\infty}.
\end{align}
and similarly
\begin{align}
    &\|\hat P A_2\|_{\hs} \leq \sqrt{S_p(HH^*) + pt_2} \times \|(HH^* +t_2I)^{-1/2}A_2\|_{\infty}.
\end{align}

\paragraph{Sketch of the following proof. } We are going to bound $\|(E +t_1I)^{-1/2}A_1\|_{\infty}$ and $\|(HH^* +t_2I)^{-1/2}A_2\|_{\infty}$, using the Lemma \ref{lem:krr_op} two times. This is done noticing that $\|(E +t_1I)^{-1/2}A_1\|_{\infty}$ is exactly the error "part $A_1$" of the KRR estimator trained with data $(x_i, (E +t_1I)^{-1/2}y)_{i=1}^n$, trying to solve the least-squares problem : $(E +t_1I)^{-1/2}y = (E +t_1I)^{-1/2}H \phi(x) +(E +t_1I)^{-1/2}\epsilon$. The same trick is used for $\|(HH^* +t_2I)^{-1/2}A_2\|_{\infty}$. In the two cases, we compute then the resulting modified constants in the bound because of these left linear operators multiplications.

\paragraph{1. Bound $\|(E +t_1I)^{-1/2}A_1\|_{\infty}$. } We apply Lemma \ref{lem:krr_op} on the KRR estimator trained with $(x_i, (E+t)^{-1/2}y_i)$.

\vspace{1em}
We have
\begin{align}
    &\|(E+t_1I)^{-1/2}E(E+t_1I)^{-1/2}\|_{\infty} \leq 1\\
    &\|(E+t_1I)^{-1/2}y\| \leq t_1^{-1/2} Q\\
    &\|(E+t_1I)^{-1/2}H\|_{\hs} \leq t_1^{-1/2} \|H\|_{\hs}.
\end{align}
Furthermore, if $\|E\|_{\infty} \geq t_1$, we have
\begin{align}
    \frac{\Tr(E(E+t_1)^{-1})}{\|E(E+t_1)^{-1}\|_{\infty}} &= \Tr(E(E+t_1)^{-1}) \frac{\|E\|_{\infty} + t}{\|E\|_{\infty}}\\
    &\leq 2\Tr(E(E+t_1)^{-1})\\
    &\leq 2\Tr(E)t_1^{-1}.
\end{align}
Thus we get with probability at least $1-2\delta$
\begin{align}
    \|(E+t_1)^{-1/2}A_1\|_{\infty} \leq \sqrt{\frac{24\eta(Q^2 t_1^{-1} + \lambda_2^{-1}\kappa^2\|E\|_{\infty})}{n}} + \frac{8\kappa Q t_1^{-1/2}\eta}{3\sqrt{\lambda_2} n}.
\end{align}
with $\eta = \log(\frac{8(\frac{\Tr(C)}{\lambda_2} + \frac{\Tr(E)}{t_1})}{\delta})$, $E=\E[\epsilon \otimes \epsilon]$, $\epsilon = y - h^*(x)$.

\paragraph{2. Bound $\|(HH^* +t_2I)^{-1/2}A_2\|_{\infty}$. }  We apply Lemma \ref{lem:krr_op} on the KRR estimator trained with $(x_i, (HH^*+t_2)^{-1/2}y_i)$.
\vspace{1em}
We have
\begin{align}
    \|(HH^* +t_2I)^{-1/2}H\|_{\infty} = \|(HH^* +t_2I)^{-1}HH^*\|_{\infty}^{1/2} \leq 1.
\end{align}
So, 
\begin{align}
    \|(HH^* +t_2I)^{-1/2}A_2\|_{\infty} \leq \sqrt{2} \sqrt{\lambda_2}
\end{align}

\paragraph{Conclusion. } We conclude by summing the bound. We have
\begin{align*}
    \|\hat P(H_n-H)S^*\|_{\hs} &\leq  \sqrt{S_p(E) + pt_1} \times \left(\sqrt{\frac{24\eta(Q^2 t_1^{-1} + \lambda_2^{-1}\kappa^2\|E\|_{\infty})}{n}} + \frac{8\kappa Q t_1^{-1/2}\eta}{3\sqrt{\lambda_2} n}\right)\\
    &+ \sqrt{S_p(HH^*) + pt_2} \times \left(\sqrt{\lambda_2} \sqrt{2}\right).
\end{align*}
Taking $t_1
=p^{-1}S_p(E) \leq \|E\|_{\infty}$, and $t_2=p^{-1}S_p(HH^*)$, we get
\begin{align*}
    \|\hat P(H_n-H)S^*\|_{\hs} &\leq \sqrt{\frac{48\eta(Q^2 p + 2S_p(E)\lambda_2^{-1}\kappa^2\|E\|_{\infty})}{n}} + \frac{4\kappa Q \sqrt{p} \eta}{\sqrt{\lambda_2} n}\\
    &+ 2\sqrt{S_p(HH^*)} \sqrt{\lambda_2}.
\end{align*}
Now, taking $\lambda_2 = \max(S_p(E)^{1/2} n^{-1/2}, n^{-1}, \frac{9\kappa^2}{n} \log(\frac{n}{\delta}))$, we get
\begin{align*}
    \|\hat P(H_n-H)S^*\|_{\hs} &\leq 7\sqrt{\frac{\eta Q^2 p}{n}} +7\sqrt{\frac{ 2\eta S_p(E)\lambda_2^{-1}\kappa^2\|E\|_{\infty}}{n}} + \frac{4\kappa Q \sqrt{p} \eta}{\sqrt{\lambda_2} n}\\
    &+ 2 \|H\|_{\hs} \sqrt{\lambda_2}\\
    &\leq 7\sqrt{\frac{\eta Q^2 p}{n}} + 7\sqrt{\frac{ 2\eta S_p(E)^{1/2}\kappa^2\|E\|_{\infty}}{n^{1/2}}} + \frac{4\kappa Q \sqrt{p} \eta}{n^{1/2}}\\
    &+ 2\|H\|_{\hs} \left(S_p(E)^{1/4} n^{-1/4} + n^{-1/2} + 3\kappa n^{-1/2} \log^{1/2}(\frac{n}{\delta})\right)\\
    &\leq \bigg[\Big(7\sqrt{\eta}Q + 4\kappa Q\eta + 2\|H\|_{\hs}(1 + 3\kappa \log(\frac{n}{\delta}))\Big)\sqrt{p}n^{-1/4} \\
    & + \Big(10\sqrt{\eta}\kappa \|E\|_{\infty}^{1/2} +  2\|H\|_{\hs}\Big) S_p(E)^{1/4}\bigg]n^{-1/4}\\
    &\leq \Big(c_4 \sqrt{p}n^{-1/4} + c_5 S_p(E)^{1/4}\Big)n^{-1/4}\log(n/\delta) 
\end{align*}
with $c_4 = (7Q + 4\kappa Q + 2\|H\|_{\hs}(1 + 3\kappa))(1 +  c_6)$, $c_5=10\sqrt{(1+c_6)}\kappa \|E\|_{\infty}^{1/2} +  2\|H\|_{\hs}$, $c_6 = \log(8(\frac{\Tr(C)}{\|E\|_{\infty}^{1/2}} + \frac{\Tr(E)}{\|E\|_{\infty}}))$, as $\eta \leq c_6 + \log(n/\delta) \leq (c_6 +1)\log(n/\delta)$ if $p\leq n$.

\end{proof}

%% file: proofs/supervisedsubspacelearning.tex
In this subsection we prove a bound on the supervised reconstruction error:
\begin{align}
    \E_x[\|\hat P h^*(x) - h^*(x)\|_{\bmY}^2]^{1/2} = \|(\hat P - I)M^{1/2}\|_{\hs}.
\end{align}
We use the proof scheme of \citet{Rudi2013OnTS} for subspace learning, retaking also the Lemma \ref{lemma:cov} restated just below. The novelty to deal with is that the random variable, whose reconstruction error is minimized here, is $h^*(x)$. The unknown $h^*(x_i)$ are estimated via our supervised subspace learning method \eqref{eq:obj} thanks to the couples $(x_i, y_i)_{i=1}^n$. This leads to additional derivations in our proofs.

We start by restating the Lemma 3.6 from \citet{Rudi2013OnTS} in a convenient form for our purposes.

\begin{lemma}[Convergence of covariance operators]\label{lemma:cov} Let $\bmX, \bmY$ be two Hilbert spaces, $H \in \bmY \otimes \bmX$, $A=\E_x[H x \otimes H x]$, $(x_i)_{i=1}^n $ i.i.d from a distribution $\rho$ on $\bmX$ supported on the unit ball, $A_n=\frac{1}{n} \sum\limits_{i=1}^n H x_i \otimes H x_i$, $B \in \bmY \otimes \bmY$ any positive semidefinite operator, $ \frac{9}{n}\log(\frac{n}{\delta}) \leq t \leq \|A\|_{\infty}$, then with probability at least $1-\delta$ it is
\[\sqrt{\frac{2}{3}} \leq \|(A + B + tI)^{\frac{1}{2}}(A_n + B + tI)^{-\frac{1}{2}}\|_{\infty} \leq \sqrt{2}\]
\end{lemma}
\begin{proof} By defining $B_n = (A + B + tI)^{-\frac{1}{2}}(A - A_n)(A + B + tI)^{-\frac{1}{2}}$, we have
\begin{align}
    \|(A + B + tI)^{\frac{1}{2}}(A_n + B + tI)^{-\frac{1}{2}}\|_{\infty} = \|(I-B_n)^{-1} \|_{\infty}^{1/2}
\end{align}
and $B$ is positive semidefinite so
\begin{align}
    \|B_n\|_{\infty} &= \|(A + B + tI)^{-\frac{1}{2}}(A - A_n)(A + B + tI)^{-\frac{1}{2}}\|_{\infty}\\
    &\leq \|(A + tI)^{-\frac{1}{2}}(A - A_n)(A + tI)^{-\frac{1}{2}}\|_{\infty}
\end{align}
Now, by applying Lemma 3.6 from \cite{Rudi2013OnTS}, we get with probability at least $1-\delta$, if $\frac{9}{n}\log(\frac{n}{\delta}) \leq t \leq \|A\|_{\infty}$
\begin{align}
    \|B_n\|_{\infty} \leq \frac{1}{2}.
\end{align}
We conclude by observing that
\begin{align}
    \frac{1}{\sqrt{1+\|B_n\|_{\infty}}} \leq \|(I-B_n)^{-1} \|_{\infty}^{1/2} \leq \frac{1}{\sqrt{1-\|B_n\|_{\infty}}}.
\end{align}
\end{proof}

The two following lemmas handle the estimation of $M = HCH^* = \E[h^*(x) \otimes h^*(x)]$ in our supervised subspace learning method. In particular, here is exploited the Assumption $\ref{as:3}$, whose the divergence rate of the plateau threshold $p_{max}$, from which the error remains constant (See \citet{Rudi2013OnTS}), depends on.

\begin{restatable}{lemma}{lemHnCtoH}\label{lem:HnC-H} Let be $\xi>0, \delta \in [0,1]$. Under Assumptions \ref{as:1}, \ref{as:3}, taking 
\begin{align}
    t \geq n^{-\frac{1}{\beta+1}} (\xi/2)^{-\frac{4}{\beta+1}} \left(4\kappa(Q + \kappa R) \left(1+ 2\kappa \|M\|_{\infty}^{  \frac{1}{4}(1-\beta)}\right)\log^2\frac{8}{\delta} + c_2^{1/2}\right)^{\frac{4}{\beta+1}}
\end{align}
and
\[\lambda_1 = t^{-\frac{1-\beta}{2}}n^{-1/2}\]
is enough to achieve with probability at least $1-\delta$
\begin{align}
    \|(H C H^* + tI)^{-\frac{1}{2}}(H_n-H)S^*\|_{\infty} &\leq \xi.
\end{align}
\end{restatable}
\begin{proof} We note for convenience $A = (H C H^* + tI)^{-\frac{1}{2}}$. We proceed as in the proof of Lemma 18 and Theorem 5 in \citep{Ciliberto2016} (showing a learning bound for the kernel ridge estimator). However, we monitor the action of $A$, and we use Assumption \ref{as:3}, in order to obtain the best bound w.r.t $t$ and $n$, decreasing fast when $n$ and $t$ increase. We have
\begin{align}
    \|A(H_n-H)S^*\|_{\infty} &= \|AZ_n^*S_n(C_n+\lambda_1)^{-1}S^* - AZ^*\|_{\infty}\\
    &\leq \text{(I)} + \text{(II)} + \text{(III)}
\end{align}
with
\begin{align*}
    \text{(I)} &= \|AZ_n^*S_n(C_n+\lambda_1)^{-1} - AZ^*S(C_n+\lambda_1)^{-1}S^*\|_{\infty}\\
    &\leq \sqrt{\frac{1}{t}} \times \|Z_n^*S_n(C_n+\lambda_1)^{-1} - Z^*S(C_n+\lambda_1)^{-1}S^*\|_{\hs}\\
    \text{(II)} &= \|AZ^*S(C_n+\lambda_1)^{-1}S^* - AZ^*S(C+\lambda_1)^{-1}S^*\|_{\infty}\\
    &\leq \sqrt{\frac{1}{t}} \times \|Z^*S(C_n+\lambda_1)^{-1}S^* - Z^*S(C+\lambda_1)^{-1}S^*\|_{\hs}\\
    \text{(III)} &= \|AZ^*S(C+\lambda_1)^{-1}S^* - AZ^*\|_{\infty}
\end{align*}

\paragraph{Bound \text{(III)}.} From Assumption \ref{as:1} we have $Z^* = HS^*$, and
\begin{align}
    \text{(III)} &= \|AZ^*(S(C+\lambda_1)^{-1}S^* - I)\|_{\infty}\\
    &= \|AHS^*(S(C+\lambda_1)^{-1}S^* - I)\|_{\infty}\\
    &= \|AH(S^* - \lambda_1 (C + \lambda_1)^{-1} - S^*)\|_{\infty}\\
    &= \lambda_1 \|AH (C + \lambda_1)^{-1}S^*\|_{\infty}\\
    &\leq \|AH\|_{\infty} \times \lambda_1\|(C + \lambda_1)^{-1}S^*\|_{\infty}\\
    &\leq \|AH\|_{\infty} \times \sqrt{\lambda_1}.
\end{align}
Using Assumption \ref{as:3} we have
\begin{align}
     \|AH\|_{\infty} &=  \|(M + tI)^{-\frac{1}{2}}H\|_{\infty}\\
     &\leq \|(H C H^* + tI)^{-\frac{1}{2}}c_2^{1/2}M^{(1-\beta)/2}\|_{\infty}\\
     &\leq c_2^{1/2} \times t^{-\frac{\beta}{2}}.
\end{align}

\paragraph{Bound \text{(I)} and \text{(II)}.}
We bound $\text{(I)}$ and $\text{(II)}$, as in \citet{Ciliberto2016} (Lemma 18). 

\paragraph{Conclusion.} This leads to the following bound with probability at least $1-\delta$:
\begin{align}
    \|A(H_n-H)S^*\|_{\infty} &\leq 4\kappa\frac{Q + \kappa R}{\sqrt{\lambda_1 n t}} \left(1+ \sqrt{\frac{4\kappa^2}{\lambda_1 \sqrt{n}}}\right)\log^2\frac{8}{\delta} + c_2^{1/2} \sqrt{\lambda_1} t^{-\frac{\beta}{2}}.
\end{align}
Now, choosing $\lambda_1 = \frac{t^{-\frac{1}{2}(1-\beta)}}{\sqrt{n}}$, if $t\leq \|M\|_{\infty}$, we obtain
\begin{align}
    \|A(H_n-H)S^*\|_{\infty} &\leq \left(4\kappa(Q + \kappa R) \left(1+ 2\kappa t^{  \frac{1}{4}(1-\beta)}\right)\log^2\frac{8}{\delta} + c_2^{1/2}\right) n^{-1/4} t^{-\frac{1}{4}(\beta+1)}\\
    &\leq \left(4\kappa(Q + \kappa R) \left(1+ 2\kappa \|M\|_{\infty}^{  \frac{1}{4}(1-\beta)}\right)\log^2\frac{8}{\delta} + c_2^{1/2}\right) n^{-1/4} t^{-\frac{1}{4}(\beta+1)}
\end{align}
Hence, taking $t \geq n^{-\frac{1}{\beta+1}} (\xi/2)^{-\frac{4}{\beta+1}} \left(4\kappa(Q + \kappa R) \left(1+ 2\kappa \|M\|_{\infty}^{  \frac{1}{4}(1-\beta)}\right)\log^2\frac{8}{\delta} + c_2^{1/2}\right)^{\frac{4}{\beta+1}} $ is enough to achieve
\begin{align}
    \|A(H_n-H)S^*\|_{\infty} &\leq \xi.
\end{align}
\end{proof}

We combine Lemmas \ref{lemma:cov} and \ref{lem:HnC-H} to finally prove a concentration bound for $H_nC_nH_n^*$ deviating from $HCH^*$.

\begin{restatable}[Convergence of the supervised covariance $M_n$]{lemma}{lemHnCnHn}\label{lemma:HnCnHn} Let be $\delta \in [0,1]$. Under Assumptions \ref{as:1}, \ref{as:3}, and defining
\[B_n = (H C H^* + tI)^{-\frac{1}{2}}(H_nC_nH_n^* - HCH^*)(H C H^* + tI)^{-\frac{1}{2}}\]
if $t \geq c_8 \log^8(\frac{8}{\delta}) n^{-\frac{1}{\beta + 1}}$, $n\geq n_0$ (constant independent of $\delta$), then with probability $1-2\delta$
\[\|B_n\|_{\infty} \leq \frac{1}{2} \]
with $c_8=(\xi/2)^{-\frac{4}{\beta+1}} \left(4\kappa(Q + \kappa R) \left(1+ 2\kappa \|M\|_{\infty}^{  \frac{1}{4}(1-\beta)}\right) + c_2^{1/2}\right)^{\frac{4}{\beta+1}}$ and $\xi=\frac{1}{14}$, $n_0 \in \mathbb{N}^*$ constant defined in the proof.
\end{restatable}
\begin{proof} We decompose in 7 terms the difference of products, then we will bound each associated term in $\|B_n\|_{\infty}$.
\begin{align*}
    H_nC_nH_n^* - HC_nH^* &= (H_n-H)CH^* \quad (i)\\
    &\quad + H C(H_n-H)^* \quad (ii)\\
    &\quad + (H_n-H)C(H_n-H)^* \quad (iii)\\
    &\quad + (H_n-H)(C_n-C)H^* \quad (iv)\\
    &\quad + H(C_n-C)(H_n-H)^* \quad (v)\\
    &\quad + (H_n-H)(C_n - C)(H_n-H)^* \quad (vi)\\  
    &\quad + H(C_n - C)H^* \quad (vii)  
\end{align*}
\paragraph{Bound $(i)$ and $(ii)$. }
\begin{align*}
    \|(H C H^* + tI)^{-\frac{1}{2}}H C(H_n-H)^*(H C H^* + tI)^{-\frac{1}{2}}\|_{\infty} &\leq \|(H C H^* + tI)^{-\frac{1}{2}}HS^*\|_{\infty} \\ 
    &\quad \times \|(H C H^* + tI)^{-\frac{1}{2}}(H_n-H)S^*\|_{\infty}
\end{align*}
But:
\begin{align*}
    \|(H C H^* + tI)^{-\frac{1}{2}}HS^*\|_{\infty} &= \|(H C H^* + tI)^{-\frac{1}{2}}HS^*SH^*(H C H^* + tI)^{-\frac{1}{2}}\|_{\infty}^{1/2}\\
    &= \|(H C H^* + tI)^{-\frac{1}{2}}HCH^*(H C H^* + tI)^{-\frac{1}{2}}\|_{\infty}^{1/2}\\
    &\leq 1
\end{align*}
And from Lemma \ref{lem:HnC-H}, defining $c_8= (\xi/2)^{-\frac{4}{\beta+1}} \left(4\kappa(Q + \kappa R) \left(1+ 2\kappa \|M\|_{\infty}^{  \frac{1}{4}(1-\beta)}\right) + c_2^{1/2}\right)^{\frac{4}{\beta+1}}$, $\xi=14$,if $t \geq c_8 \log^8(\frac{8}{\delta}) n^{-\frac{1}{\beta+1}} $ we get with probability at least $1-\delta$
\begin{align*}
    \|(H C H^* + tI)^{-\frac{1}{2}}(H_n-H)S^*\|_{\infty} &\leq \frac{1}{14}
\end{align*}
\paragraph{Bound $(iii)$. } As for (i) and (ii), from Lemma \ref{lem:HnC-H} we have
\begin{align*}
    \|(H C H^* + tI)^{-\frac{1}{2}}(H_n-H)C(H_n-H)^*(H C H^* + tI)^{-\frac{1}{2}}\|_{\infty} &\leq \|(H C H^* + tI)^{-\frac{1}{2}}(H_n-H)S^*\|_{\infty}^2\\
    &\leq \frac{1}{14^2} \leq \frac{1}{14}.
\end{align*}
\paragraph{Bound $(iv)$ and $(v)$. } We decompose
\begin{equation*}
\begin{split}
    &\|(H C H^* + tI)^{-\frac{1}{2}}(H_n-H)(C_n-C)H^*(H C H^* + tI)^{-\frac{1}{2}}\|_{\infty} 
    \leq \\
    &\|(H C H^* + tI)^{-\frac{1}{2}}(H_n-H)C_t^{1/2}\|_{\infty}
    \times \|C_t^{-1/2}(C_n-C)C_t^{-1/2}\|_{\infty} \times \|C_t^{1/2}H^*(H C H^* + tI)^{-\frac{1}{2}}\|_{\infty}.
\end{split}
\end{equation*}
We bound
\begin{equation*}
\begin{split}
    \|(H C H^* + tI)^{-\frac{1}{2}}(H_n-H)C_t^{1/2}\|_{\infty} &= \|(H C H^* + tI)^{-\frac{1}{2}}(H_n-H)C_t(H_n-H)^*(H C H^* + tI)^{-\frac{1}{2}}\|_{\infty}^{1/2}\\
    &\leq \|(H C H^* + tI)^{-\frac{1}{2}}(H_n-H)S^*\|_{\infty} + t^{1/2} \|(H C H^* + tI)^{-\frac{1}{2}}(H_n-H)\|_{\infty}\\
    &\leq \|(H C H^* + tI)^{-\frac{1}{2}}(H_n-H)S^*\|_{\infty} + \|H_n-H\|_{\infty}
\end{split}
\end{equation*}
and similarly,
\begin{equation*}
\begin{split}
    \|C_t^{1/2}H^*(H C H^* + tI)^{-\frac{1}{2}}\|_{\infty} &\leq \|(H C H^* + tI)^{-\frac{1}{2}}HS^*\|_{\infty} + t^{1/2} \|(H C H^* + tI)^{-\frac{1}{2}}cH\|_{\infty}\\
    &\leq \|(H C H^* + tI)^{-\frac{1}{2}}HS^*\|_{\infty} + \|H\|_{\infty}\\
    &\leq 1 + \|H\|_{\infty}
\end{split}
\end{equation*}
finally we obtain
\begin{equation*}
\begin{split}
    &\|(H C H^* + tI)^{-\frac{1}{2}}(H_n-H)(C_n-C)H^*(H C H^* + tI)^{-\frac{1}{2}}\|_{\infty} 
    \leq \\
    &\left(\|(H C H^* + tI)^{-\frac{1}{2}}(H_n-H)S^*\|_{\infty} + \|H_n-H\|_{\infty}\right)\\
    &\times \|C_t^{-1/2}(C_n-C)C_t^{-1/2}\|_{\infty} \\
    &\times  \left(1 + \|H\|_{\infty}\right).
\end{split}
\end{equation*}
From Lemma \ref{lem:HnC-H}, $\|(H C H^* + tI)^{-\frac{1}{2}}(H_n-H)(C_n-C)H^*(H C H^* + tI)^{-\frac{1}{2}}\|_{\infty} \leq 1/14$ if $t \geq c_8 \log^8(\frac{8}{\delta}) n^{-\frac{1}{\beta + 1}}$. From Lemma \ref{lemma:Hn}, $\|H_n-H\|_{\infty} \leq 2\log^8(\frac{8}{\delta})R$ if $n \geq n_1$ with $n_1$ a constant independent of $\delta$. So, defining $u = (1/14 + 2R) \times (1 + R)$. Now, using Lemma 3.6 from \citet{Rudi2013OnTS}, we can have $\|C_t^{-1/2}(C_n-C)C_t^{-1/2}\|_{\infty} \leq 1/14 \times u^{-1} \log^{-8}(\frac{8}{\delta})$ if $t \geq a_1 \frac{\log n/\delta}{n}$ with $a_1>0$ a constant independent of $\delta$. We conclude that
\begin{equation*}
\begin{split}
    \|(H C H^* + tI)^{-\frac{1}{2}}(H_n-H)(C_n-C)H^*(H C H^* + tI)^{-\frac{1}{2}}\|_{\infty}  \leq \frac{1}{14}.
\end{split}
\end{equation*}
 
\paragraph{Bound $(vi)$. } Similarly as for $(v)$, we have
\begin{equation*}
\begin{split}
    &\|(H C H^* + tI)^{-\frac{1}{2}}(H_n-H)(C_n - C)(H_n-H)^*(H C H^* + tI)^{-\frac{1}{2}}\|_{\infty}
    \leq \\
    &\left(\|(H C H^* + tI)^{-\frac{1}{2}}(H_n-H)S^*\|_{\infty} + \|H_n-H\|_{\infty}\right)^2\\
    &\times \|C_t^{-1/2}(C_n-C)C_t^{-1/2}\|_{\infty} \\
\end{split}
\end{equation*}
and, if $t \geq a_2 \frac{\log n/\delta}{n}$, with $a_2$ a constant independent of $\delta$, we also have
\begin{equation*}
\begin{split}
    \|(H C H^* + tI)^{-\frac{1}{2}}(H_n-H)(C_n - C)(H_n-H)^*(H C H^* + tI)^{-\frac{1}{2}}\|_{\infty}  \leq 1/14.
\end{split}
\end{equation*}
\paragraph{Bound $(vii)$. }
As previously, from Lemma 3.6 from \citet{Rudi2013OnTS}, there exists a constant $a_3>0$ such that with probability at least $1-\delta$ if $t \geq a_3 \frac{\log n/\delta}{n}$, with $a_3>0$, we have
\begin{align*}
    \|(H C H^* + tI)^{-\frac{1}{2}}H(C_n - C)H^*(H C H^* + tI)^{-\frac{1}{2}}\|_{\infty} &\leq 1/14.
\end{align*}
\paragraph{Conclusion. } But there exists $n_0$ independent of $\delta$ such that $\forall n \geq n_0 \geq n_1$, $ c_8 \log^8(\frac{8}{\delta}) n^{-\frac{1}{\beta + 1}} \geq \max(a_1, a_2, a_3) \frac{\log n/\delta}{n}$. So, we conclude that, if $t \geq c_8 \log^8(\frac{8}{\delta}) n^{-\frac{1}{\beta + 1}}$, and $n \geq n_0$,
\[\|B_n\|_{\infty} \leq \frac{1}{2}.\]
\end{proof}

We are now ready to prove the main result of this section. We prove a bound on the reconstruction error of $\hat P$ when reconstructing the $h^*(x)$, namely $\E_x[\|\hat Ph^*(x) - h^*(x)\|_{\bmY}^2]^{1/2}$.

\begin{restatable}[Supervised subspace learning]{lemma}{mainlem}\label{lem:set}Let $(x_i, y_i)_{i=1}^n$ be drawn independently from a probability measure $\rho$ and $(y_i)_{i=1}^m$ be drawn independently from the marginal $\rho$ w.r.t $y$ with support in the ball $\|y\|_{\bmY} \leq Q$. Let $\hat P$ be the estimated projection in the proposed method. Then, under Assumptions \ref{as:1}, \ref{as:2} and \ref{as:3}, there exist constants $c_8>0, n_0 \in \mathbb{N}^*$, such that, if $\mu_{p+1}(M) \geq c_8 \log^8(\frac{8}{\delta}) n^{-\frac{1}{\beta +1 }}$, $n\geq n_0$, $\lambda_1 = \mu_{p+1}(M)^{-\frac{1-\beta}{2}} n^{-\frac{1}{2}}$, then with probability at least $1-3\delta$
\[\E_x[\|\hat P h^*(x) - h^*(x)\|^2_{\bmY}]^{\frac{1}{2}} \leq \sqrt{3c_1} \mu_{p+1}(M)^{1/2(1-\alpha)}.\]
\end{restatable}
\begin{proof} We have (See Proposition C.4. in \citet{Rudi2013OnTS}):
\begin{align}
    \E_x[\|\hat P h^*(x) - h^*(x)\|]^{1/2} = \|(\hat P - I)M_c^{\frac{1}{2}}\|_{\hs}^2\label{hsd}
\end{align}

Then, as in the proofs of \citet{Rudi2013OnTS}, we split (\ref{hsd}) into three parts, and bound each term,
\[\|(\hat P - I)M^{\frac{1}{2}}\|_{\hs} \leq \underbrace{\|(M + tI)^{\frac{1}{2}}(M_n + tI)^{-\frac{1}{2}}\|_{\infty}}_{\mathcal{A}} \times \underbrace{(\mu_{p+1}(M_n) + t)^{\frac{1}{2}}}_{\mathcal{B}} \times \underbrace{\|(M + tI)^{-\frac{1}{2}}M^{\frac{1}{2}}\|_{\hs}}_{\mathcal{C}}\]

\paragraph{Bound $\mathcal{A}=\|(M + tI)^{\frac{1}{2}}(M_n + tI)^{-\frac{1}{2}}\|_{\infty}$. } We have:
\begin{align*}
    \|(M + tI)^{\frac{1}{2}} (M_n + tI)^{-\frac{1}{2}}\|_{\infty} &= \|(M + tI)^{\frac{1}{2}} (M_n + tI)^{-1} (M + tI)^{\frac{1}{2}}\|_{\infty}^{1/2}\\
    &= \|(I - B_n)^{-1}\|_{\infty}^{1/2}
\end{align*}
with $B_n = (M + tI)^{-1/2}(M - M_n)(M + tI)^{-1/2}$.
So, if $\|B_n\|_{\infty} < 1$, 
\begin{align*}
    \frac{1}{\sqrt{1 + \|B_n\|_{\infty}}} \leq \|(M + tI)^{\frac{1}{2}} (M_n + tI)^{-\frac{1}{2}}\|_{\infty} \leq \frac{1}{\sqrt{1- \|B_n\|_{\infty}}}
\end{align*}
Then applying Lemma \ref{lemma:HnCnHn}, if $t\geq c_8 \log^8(\frac{8}{\delta})  n^{-\frac{1}{\beta +1 }} $, with probability $1-3\delta $ it is
\[\sqrt{\frac{2}{3}} \leq \|(M + tI)^{\frac{1}{2}}(M_n + tI)^{-\frac{1}{2}}\|_{\infty} \leq \sqrt{2}\]

\paragraph{Bound $\mathcal{B}=(\mu_{p+1}(M_n) + t)^{\frac{1}{2}}$. } $\sqrt{\frac{2}{3}} \leq \|(M + tI)^{\frac{1}{2}} (M_n + tI)^{-\frac{1}{2}}\|_{\infty}$ is equivalent to $M_n + t \preceq \frac{3}{2} M_n + t$ (by Lemma B.2 point 4 in \citep{Rudi2013OnTS}). Then, $\forall k \in \mathbb{N}^*, \mu_k(M_n+t) \leq \frac{3}{2} \mu_k(M+t)$, so we have
\begin{align}
    \sqrt{\mu_{p+1}(M_n) + t} \leq \sqrt{\frac{3}{2}} \sqrt{\mu_{p+1}(M) + t}.
\end{align}

\paragraph{Bound $\mathcal{C}=\|(M + tI)^{-\frac{1}{2}}M^{\frac{1}{2}}\|_{\hs}$. } We have
\begin{align}
    \mathcal{C}^2 &= \Tr(M (M +t)^{-1})\\
    &= \Tr(M^{\alpha}M^{1-\alpha}(M +t)^{-1})\\
    &\leq \Tr(M^{\alpha}) \|M^{1-\alpha}(M +t)^{-1}\|_{\infty}\\
    &\leq c_1 \times t^{-\alpha} \quad \text{(from Assumption \ref{as:2} and Young's inequality for products)}.
\end{align}
Finally, we get the following upper bound.
\begin{equation}\label{eq:bound_t}
    \E_x[\|\hat P h^*(x) - h^*(x)\|]^{1/2} \leq \sqrt{3}\sqrt{\mu_{p+1}(M) + t} \times c_1^{1/2} \times t^{-\alpha/2}
\end{equation}
Taking $t=\mu_{p+1}(M)$, which is possible if $\mu_{p+1}(M) \geq c_8 \log^8(\frac{8}{\delta}) n^{-\frac{1}{\beta + 1}}$, we get
 \begin{align}
    \E_x[\|\hat P h^*(x) - h^*(x)\|]^{1/2} &\leq \sqrt{3c_1}\mu_{p+1}(M)^{1/2(1-\alpha)}.
\end{align}
We get the wanted upper bound.

\end{proof}

%% file: proofs/theorem.tex
In this subsection we give the main result of this paper which is a learning bound for the proposed method. That is we bound:
\begin{equation}
    \E_x[\|\hat P \hat h(x) - h^*(x)\|^2_{\bmY}].
\end{equation}
The proof consists in decomposing this excess-risk in two terms, as in equation \eqref{eq:comp}, then bounding each term applying the two lemmas previously proved.

\th*
\begin{proof} We decompose the excess-risk as follows
 \begin{equation}
        \E_x[\|\hat P \hat h(x) - h^*(x)\|_{\bmY}^2]^{1/2} \leq \underbrace{\E_x[\|\hat P \hat h(x) - \hat P h^*(x)\|_{\bmY}^2]^{1/2}}_{\text{regr. error on a subspace}} + \underbrace{\E_x[\|\hat P h^*(x) - h^*(x)\|_{\bmY}^2]^{1/2}}_{\text{reconstruction error}}.
\end{equation}
We apply the Lemmas \ref{lem:krr_sub} and \ref{lem:set}, and we get, if $\mu_{p+1}(M) \geq  c_8 \log^8(\frac{8}{\delta}) n^{-\frac{1}{\beta +1 }}$,  with probability at least $1-3\delta$:
\begin{align}
    \E_x[\|P\hat h(x) - h^*(x)\|^2_{\bmY}]^{1/2} &\leq \Big(c_4 \sqrt{p}n^{-1/4} + c_5 S_p(E)^{1/4}\Big)n^{-1/4}\log(n/\delta) + \sqrt{3c_1}\mu_{p+1}(M)^{1/2(1-\alpha)}.
\end{align}
with $c_4 = (7Q + 4\kappa Q + 2\|H\|_{\hs}(1 + 3\kappa))(1 +  c_6)$, $c_5=10\sqrt{(1+c_6)}\kappa \|E\|_{\infty}^{1/2} +  2\|H\|_{\hs}$, $c_6 = \log(8(\frac{\Tr(C)}{\|E\|_{\infty}^{1/2}} + \frac{\Tr(E)}{\|E\|_{\infty}}))$.

\end{proof}

%% file: proofs/corollary.tex
In this subsection we derive from the Theorem \ref{th} a corollary in the case where $M$ and $E$ have polynomial eigenvalue decay rates. This allows to explicit the optimal quantity of components $p$, and also obtaining a condition on the decay rates $s, e >1$ in order to obtain a statistical gain.
\cor*
\begin{proof} The proof consists in applying the Theorem \ref{th} in the specific case of polynomial eigenvalue decay rates. If $\mu_{p+1}(M) \geq c_8 \log^8(\frac{8}{\delta}) n^{-\frac{1}{\beta +1 }}$,  with probability at least $1-3\delta$:
\begin{align}
    \E_x[\|P\hat h(x) - h^*(x)\|^2_{\bmY}]^{1/2} &\leq \Big(c_4 \sqrt{p}n^{-1/4} + c_5 S_p(E)^{1/4}\Big)n^{-1/4}\log(n/\delta) + \sqrt{3c_1}\mu_{p+1}(M)^{1/2(1-\alpha)}.
\end{align}

\paragraph{Bound $S_p(E)$. }

The polynomial eigenvalue decay assumption, give us that $\frac{a}{p^{s}} \leq \mu_p(M) \leq \frac{A}{p^{s}}$. So, Assumption \ref{as:1} is verified with $\alpha = \frac{2}{s}$, and $c_1= \Tr(M^{\alpha}) \leq \sum_i i^{-2} \times A^{\alpha} \leq 2A^{\alpha}$. Hence,

\begin{align}
    \sqrt{3c_1}\mu_{p+1}(M)^{1/2(1-\alpha)} \leq \frac{\sqrt{6A^{\alpha}} A^{1/2(1-\alpha)}}{p^{\frac{1}{\alpha} - 1}} = \frac{\sqrt{6A}}{p^{\frac{s}{2} - 1}}.
\end{align}
Moreover,
\begin{align}
    S_p(E) = \sum_{i=1}^p \mu_i(E) \leq B \sum_{i=1}^p i^{-e} \leq B(1 + \int_{x=1}^p x^{-e} dx) \leq \frac{B}{1-e^{-1}} \times (1 - \frac{e^{-1}}{p^{e - 1}})
\end{align}
and using $(1-1/x) \leq \log(x) \leq x-1$, we get 
\begin{align}
    S_p(E) &\leq \frac{B}{1-e^{-1}} \times \left((e-1)\log(p) + \log(e)\right)\\
    &\leq \frac{B}{1-e^{-1}} \times \left((e-1)\log(p) + (e-1)\right)\\
    &= \frac{B}{1-e^{-1}} \times (e-1)(\log(p) + 1)\\
    &\leq \frac{B}{1-e^{-1}} \times 2(e-1)\log(p) \quad \text{(if } p > 3)\\
    &= 2Be\log(p).
\end{align}
Now, taking $p = c_9 (\log^8(\frac{8}{\delta}))^{-\frac{1}{s}} n^{\frac{1}{s(\beta + 1)}}$, defining $c_9=(\frac{c_8}{a})^{-\frac{1}{s}}$, ensures $\mu_p(M) \geq c_8 \log^8(\frac{8}{\delta}) n^{-\frac{1}{\beta +1}}$. Moreover, $B \leq \theta \times b \leq \theta \Tr(E) (\sum_{i=1}^{+\infty}i^{-e})^{-1} = \frac{\theta \Tr(E)}{\zeta(e)}$ by definition of the Riemann zeta function. So, using this defined $p$, we get,
\begin{align}
    S_p(E) &\leq 2Be \left( \frac{1}{s}\log\left(\frac{a}{c_8}\right) + \frac{\log(n)}{s(\beta+1)} \right)\\
    &\leq \frac{2\theta \Tr(E)e \log(n) }{\zeta(e) s} \left(\log\left(\frac{a}{c_8}\right) + 1 \right) \quad \text{(if } n > 3)
\end{align}

\paragraph{Bound $\sqrt{3c_1} \mu_{p+1}(M)^{1/2(1-\alpha)}$. } Now, taking $p = c_9 (\log^8(\frac{8}{\delta}))^{-\frac{1}{s}} n^{\frac{1}{s(\beta + 1)}}$, defining $c_9=(\frac{c_8}{a})^{-\frac{1}{s}}$, ensures $\mu_p(M) \geq c_8 \log^8(\frac{8}{\delta}) n^{-\frac{1}{\beta +1}}$. Using this defined $p$, we get
\begin{align}
    \sqrt{3c_1}\mu_{p+1}(M)^{1/2(1-\alpha)} &\leq \sqrt{6A}  (\frac{c_8}{a} \log^8(\frac{8}{\delta}))^{\frac{1}{2}(1-\frac{2}{s})} n^{-\frac{1}{2}\frac{1-\frac{2}{s}}{1+\beta}}\\
    &\leq \sqrt{6A} (\sqrt{\frac{c_8}{a}} +1) \log^8(\frac{8}{\delta}) n^{-\frac{1}{2}\frac{1-\frac{2}{s}}{1+\beta}}.
\end{align}

\paragraph{Bound $\sqrt{p}n^{-1/2}$. } Furthermore, one can check that $(\frac{1}{2} - \frac{1}{2s(\beta+1)}) > \frac{1}{2}\frac{1-2/s}{1+\beta}$, hence we have
\begin{align}
    \sqrt{p}n^{-1/2} &\leq \left(\frac{a}{c_8}\right)^{1/2s} n^{-(\frac{1}{2} - \frac{1}{2s(\beta+1)})} \leq \left(\frac{a}{c_8}+1\right) n^{-\frac{1}{2}\frac{1-2/s}{1+\beta}}.
\end{align}

\paragraph{Studying $c_4, c_5, c_8, n_0$ dependencies in $s,e$. } In this work we study the behavior of the bound when the shape of $E$ and $M$ vary, i.e. when $s$ and $e$ vary. Therefore, it's important to make some derivations to studying $c_4, c_5, c_8, n_0$'s dependencies in $s$ and $e$. First, $c_8, n_0$ are independent of $\delta, s,e$.

Then, observing that we have $\|E\|_{\infty}^{-1} = \mu_1(E)^{-1} \leq b^{-1} \leq \frac{\theta}{B} \leq \theta \frac{\zeta(e)}{\Tr(E)}$, leads to $c_6 \leq \log(8(\frac{\theta^{1/2}\Tr(C)}{\Tr(E)^{1/2}} + \theta)) + \log(\zeta(e))$. So, we have
\begin{align}
    c_4 &= (7Q + 4\kappa Q+ 2 R (1 + 3\kappa))(1 +  c_6)\\
    &\leq \left(\log(\zeta(e))+1\right) \left(1+\log(8(\frac{\theta^{1/2}\Tr(C)}{\Tr(E)^{1/2}} + \theta))\right)\left(7Q + 4 \kappa Q + 2R(1 + 3\kappa)\right)
\end{align}
and also
\begin{align}
    c_5 &= 10\sqrt{(1+c_6)}\kappa \|E\|_{\infty}^{1/2} +  2\|H\|_{\hs}\\
    &\leq \left( \log(\zeta(e) + 1)\right) \left(1 + \log(8(\frac{\theta^{1/2}\Tr(C)}{\Tr(E)^{1/2}} + \theta))\right) \left(10\kappa \|E\|_{\hs}^{1/2} +  2\|H\|_{\hs}\right).
\end{align}

\paragraph{Conclusion. } Thanks to the previous derivations we obtain the following bound
\begin{align*}
    \E_x[\|P\hat h(x) - h^*(x)\|^2_{\bmY}]^{1/2} &\leq c_{10}(s, e) \log^{5/4}(\frac{n}{\delta})n^{-1/4} + c_{11}(e) n^{-\frac{1}{2}\frac{1-2/s}{1+\beta}} \log^8(\frac{8}{\delta})
\end{align*}
with
$c_{10}(s,e) =  \tilde c_{10} (\log(\zeta(e))+1) \left(\frac{e}{\zeta(e) \times s}\right)^{1/4} $, $c_{11}(e) = \tilde c_{11} (\log(\zeta(e))+1) $. $\tilde c_{10}$ and $\tilde c_{11}$ are constants independent of $n, \delta, s, e$, defined below
\begin{align}
    &\tilde c_{10} = \left(1 + \log(8(\frac{\theta^{1/2}\Tr(C)}{\Tr(E)^{1/2}} + \theta))\right) \left(10\kappa \|E\|_{\hs}^{1/2} +  2\|H\|_{\hs}\right) \left(2\theta \Tr(E) (\log(\frac{a}{c_8}) + 1) \right)^{1/4}\\
    &\tilde c_{11}=\sqrt{6A} (\sqrt{\frac{c_8}{a}} +1) + \left(\frac{a}{c_8}+1\right) \left(1+\log(8(\frac{\theta^{1/2}\Tr(C)}{\Tr(E)^{1/2}} + \theta))\right)\left(7Q + 4 \kappa Q + 2R(1 + 3\kappa)\right).
\end{align}
The inequalities $\frac{1}{e-1} \leq \zeta(e) \leq \frac{e}{e-1}$ allow to conclude the proof.

\end{proof}

%% file: proofs/fullrank_krr.tex
\begin{restatable}[Full-rank KRR excess-risk ]{lemma}{lemkrr}\label{lem:krr} Let $k: \bmX \times \bmX \rightarrow \R$ be a bounded kernel with $\forall x \in \bmX, k(x,x) \leq \kappa^2$. Let $\rho$ be a distribution on $\bmX \times \bmY$ such that its marginal w.r.t $y$ is supported on the ball $\|y\|_{\bmY} \leq Q$. Let $\hat h$ be the KRR estimator trained with $n$ independent couples drawn from $\rho$. Let $\delta \in [0,1]$. Then, under the assumption \ref{as:1} and \ref{as:3}, taking
\begin{align}
\lambda_1 = \max\left(\frac{1}{n}, \frac{\|E^{1/2}\|_{\hs}}{\sqrt{n}}\right)
\end{align}
the following holds with probability at least $1-\delta$
\begin{align}
    \E_x[\|\hat h(x) - h^*(x)\|^2_{\bmY}]^{\frac{1}{2}} &\leq C(p) n^{-\frac{1}{4}} \log \frac{4}{\delta}
\end{align}
with
$C(p) = 10\Bigg[\mathcal{O}(n^{-\frac{1}{4}}) + (\kappa + R)\|E^{1/2}\|_{\hs}^{\frac{1}{2}}\Bigg]$, $R= \|H\|_{\hs}$.
\end{restatable}

\begin{proof} We follow the proofs of \citep{ciliberto2020general} in order to derive a learning bound of the KRR estimator. We carefully monitor the role of the total amount of noise $\Tr(E)$.

We make appear the conditional variance by modifying the Proposition B.7 in \citep{ciliberto2020general}, with the following change from equation (B.55) to (B.58):
\begin{align}
    \E_x [\|C_\lambda^{-1/2}\phi(x)\|^2 \sigma(x)^2] &\leq \frac{\kappa^2}{\lambda} \times \E_x [\sigma(x)^2]\\
    &= \frac{\kappa^2}{\lambda} \times \E[\|\epsilon\|_{\bmY}^2]\\
    &= \frac{\kappa^2}{\lambda} \times \|E^{1/2}\|_{\text{HS}}^2
\end{align}
by defining the noise $\epsilon = \psi(y) - h^*(x)$, and $E=\mathbb{E}[\epsilon \otimes \epsilon]$.

Then, doing the same proof than Theorem B.8 from \citep{ciliberto2020general}, we get the following bound
\begin{align*}
    \E_x[\|P\hat h(x) - P h^*(x)\|^2_{\bmY}]^{1/2} \leq \quad & \frac{8\kappa \log \frac{2}{\delta}}{\sqrt{\lambda}n} \times (Q + \kappa \|L_\lambda^{-1/2}Z \|_{\text{HS}})\\
    &+ \frac{1}{\sqrt{n}} \times \sqrt{64(\deff(\lambda)\times  \|E^{1/2}\|_{\text{HS}}^2 + \kappa^2\lambda \|L_\lambda^{-1}Z \|^2_{\text{HS}})\log \frac{4}{\delta}}\\
    &+ 10 \times \lambda \|L_{\lambda}^{-1}Z\|_{\text{HS}}
\end{align*}

Now, using the assumption 1, we have
\begin{align}
    \|L_\lambda^{-1}Z \|_{\text{HS}}  &= \|L_\lambda^{-1} S H^*\|_{\text{HS}}\\
    &\leq \|L_\lambda^{-1}S\|_{\text{HS}} \times \| H\|_{\text{HS}}\\
    &\leq \lambda^{-\frac{1}{2}} \times R
\end{align}
and similarly $\|L_\lambda^{-1}Z\|_{\text{HS}} \leq R$. Moreover,
\begin{align}
    \deff(\lambda):= \Tr((C + \lambda I)^{-1}C) \leq \lambda^{-1}\kappa^2.
\end{align}
So, we get
\begin{align*}
    \E_x[\|\hat h(x) - h^*(x)\|^2_{\bmY}]^{1/2} \leq \quad & \frac{\kappa (Q + \kappa R)}{\sqrt{\lambda}n} \times 10\log \frac{4}{\delta}\\
    &+ \frac{1}{\sqrt{n}} \times \sqrt{(\lambda^{-1}\kappa^2  \|E^{1/2}\|_{\text{HS}}^2 + \kappa^2R)} \times 10\log \frac{4}{\delta}\\
    &+ \lambda^{\frac{1}{2}} \times R \times 10\log \frac{4}{\delta}
\end{align*}
Now, we define $\lambda$ in order to minimize this bound, with
\[\lambda = \max\left(\frac{1}{n}, \frac{\|E^{1/2}\|_{\text{HS}}}{\sqrt{n}}\right)\]
so we obtain
\begin{align*}
    \E_x[\|\hat h(x) - h^*(x)\|^2_{\bmY}]^{1/2} \leq \quad & n^{-\frac{1}{4}} \times 10\log \frac{4}{\delta} \times \Bigg[(\kappa (Q + \kappa R))n^{-\frac{1}{4}} \\
    &+ \sqrt{\kappa^2 \|E^{1/2}\|_{\text{HS}} + \kappa^2R^2 n^{-\frac{1}{2}}}\\
    &+ \left(n^{-\frac{1}{4}} +  \|E^{1/2}\|_{\text{HS}}^{\frac{1}{2}} \right)\times R\Bigg].\\
\end{align*}
We conclude
\begin{align}
    \E_x[\|\hat h(x) - h^*(x)\|^2_{\bmY}]^{1/2} &\leq C(p) n^{-\frac{1}{4}} \log \frac{4}{\delta}
\end{align}
with
\begin{align*}
    C(p) &= 10\Bigg[(\kappa (Q + \kappa R))n^{-\frac{1}{4}}
    + \kappa \sqrt{ \|E^{1/2}\|_{\text{HS}} + R^2 n^{-\frac{1}{2}}}
    + \left(n^{-\frac{1}{4}} +  \|E^{1/2}\|_{\text{HS}}^{\frac{1}{2}} \right)R\Bigg]\\
    &= 10\Bigg[\mathcal{O}(n^{-\frac{1}{4}}) + (\kappa + R)\|E^{1/2}\|_{\text{HS}}^{\frac{1}{2}}\Bigg].
\end{align*}

\end{proof}

%% file: proofs/conc_op.tex
\begin{theorem}[Concentration inequality on the operator norm, \citet{tropp2012user}(Theorem 7.3.2)]\label{thm:tropp} Let $\xi_i$ be independent copies of the random variable $\xi$ with values in the space of bounded operators over a Hilbert space $\bmH$ such that $\E[\xi]=0$. Let there be $R>0$ such that $\|\xi\|_{\infty} \leq T$. Define $\sigma^2 = \max(\|\E[\xi\xi^*]\|_{\infty}, \|\E[\xi^*\xi]\|_{\infty})$,  and $d = \Tr(\E[\xi^*\xi] + \E[\xi\xi^*]) / \sigma^2$. Then, if $\delta \in [0,1]$, with probability at least $1-\delta$
\begin{align}
    \left\|\frac{1}{n}\sum_{i=1}^n \xi_i\right\|_{\infty} \leq \sqrt{\frac{2\eta\sigma^2}{n}} + \frac{2T\eta}{3n}
\end{align}
where $\eta = \log(\frac{4d}{\delta})$.
\end{theorem}

\begin{proof} This theorem is a restatement of Theorem 7.3.2 of \citep{tropp2012user} generalized to the separable Hilbert space case by means of the technique in Section 3.2 of \citep{minsker2011some}.
\end{proof}

%% file: proofs/aux_proofs.tex
\begin{lemma}[Bound $\|H_n - H\|_{\infty}$]\label{lemma:Hn} With probability at least $1-2\delta $ it is
\[\|H_n - H\|_{\infty} \leq \frac{4\log\frac{2}{\delta}}{\lambda_1 \sqrt{n}}(Q\kappa + \kappa^2\|h^*_{\psi}\|_{\bmH}) + \|h^*_{\psi}\|_{\bmH}\]
\end{lemma}
\begin{proof}
In order to bound $\|H_n - H\|_{\infty}$ we do the following decomposition in three terms, and bound each term:
\begin{align*}
    \|H_n - H\|_{\infty} &=  \|Z_n^*S_n(C_n + \lambda_1 I)^{-1} - Z^*SC^{\dagger}\|_{\infty}\\
    &\leq \underbrace{\|(Z_n^*S_n- Z^*S)(C_n + \lambda_1 I)^{-1}\|_{\infty}}_{(A)} + \underbrace{\| Z^*S((C_n + \lambda_1 I)^{-1} - (C + \lambda_1 I)^{-1})\|_{\infty}}_{(B)}\\
    &\quad + \underbrace{\|Z^*S ((C + \lambda_1 I)^{-1} - C^{\dagger})\|_{\infty}}_{(C)}
\end{align*}
\paragraph{Bound (A). }
We have:
\begin{align*}
    (A) &= \|(Z_n^*S_n- Z^*S)(C_n + \lambda_1 I)^{-1}\|_{\infty}
    \leq \frac{1}{\lambda_1}\|Z_n^*S_n- Z^*S\|_{\hs}
\end{align*}
From \cite{Ciliberto2016} (proof of lemma 18.), with probability $1-\delta$: $(A) \leq \frac{4Q\kappa\log\frac{2}{\delta}}{\lambda_1 \sqrt{n}}$.

\paragraph{Bound (B). }
We have:
\begin{align*}
    (B) &= \| Z^*S((C + \lambda_1 I)^{-1} - (C_n + \lambda_1 I)^{-1})\|_{\infty}\\
    &= \| Z^*S((C + \lambda_1 I)^{-1}(C_n - C)(C_n + \lambda_1 I)^{-1})\|_{\infty}\\
    &\leq \| Z^*S(C + \lambda_1 I)^{-1}\|_{\infty} \|(C_n - C)\|_{\infty} \|(C_n + \lambda_1 I)^{-1}\|_{\infty}\\
    &\leq \frac{1}{\lambda_1} \|h^*_{\psi}\|_{\bmH} \|(C_n - C)\|_{\infty}
\end{align*}
where we used the fact that for two invertible operators $A, B$: $A^{-1} - B^{-1} = A^{-1}(B-A)B^{-1}$, and noting that $\| Z^*S(C + \lambda_1 I)^{-1}\|_{\infty} \leq \| Z^*S(C + \lambda_1 I)^{-1}\|_{\hs} \leq \|H\|_{\hs} = \|h^*_{\psi}\|_{\bmH}$. From \cite{Ciliberto2016}, with probability $1-\delta$: $(B) \leq \frac{4\|h^*_{\psi}\|_{\bmH}\kappa^2\log\frac{2}{\delta}}{\lambda_1 \sqrt{n}}$.

\paragraph{Bound (C). }
We have:
\begin{align*}
    (C) &= \|Z^*S ((C + \lambda_1 I)^{-1} - C^{\dagger})\|_{\infty}\\
    &= \|H S^*S((C + \lambda_1 I)^{-1} - C^{\dagger})\|_{\infty}\\
    &= \|H (C(C + \lambda_1 I)^{-1} - I)\|_{\infty}\\
    &= \lambda_1 \|H (C + \lambda_1 I)^{-1}\|_{\infty}\\
    &\leq \|h^*_{\psi}\|_{\bmH}
\end{align*}
We conclude by union bound, with probability at least $1-2\delta$:
\[\|H_n - H\|_{\infty} \leq \frac{4Q\kappa\log\frac{2}{\delta}}{\lambda_1 \sqrt{n}} + \frac{4\|h^*_{\psi}\|_{\bmH}\kappa^2\log\frac{2}{\delta}}{\lambda_1 \sqrt{n}} + \|h^*_{\psi}\|_{\bmH}\]
Notice that if we choose $\lambda_1 = (c_8 \log^8(\frac{8}{\delta}))^{-\frac{1-\beta}{2}}n^{-\frac{\beta}{\beta + 1}}$ as chosen in Lemma \ref{lem:HnC-H}, we obtain
\begin{align}
    \|H_n - H\|_{\infty} \leq (4Q\kappa+ R\kappa^2)\log\frac{2}{\delta} \times a \times  n^{\frac{\beta}{1+\beta}-\frac{1}{2}}+ R
\end{align}
with $a=c_8 \log^8(\frac{8}{\delta})^{\frac{1}{2}}$, such that $\|H_n - H\|_{\infty} \leq 2R \log^9(\frac{8}{\delta})^{\frac{1}{2}}$ when $n \geq N$ with $N>0$ a constant independent of $\delta$. 

\end{proof}

%% file: proofs/lowner_properties.tex
\begin{lemma}[Properties of Löwners's partial ordering $\preceq$]\label{lemma:lowner} Let $A, B$ be positive semidefinite linear operators on $\bmY$ such that $A \preceq B$, and $M$ a bounded linear operator on $\bmY$, then
\begin{enumerate}
    \item If $A, B$ are random variables then $\E[A] \preceq \E[B]$.
    \item $MAM^* \preceq MBM^*$.
\end{enumerate}
\end{lemma}

\begin{proof}

\paragraph{1)} For any $u \in \bmY$, we have $\langle u, \E[A] u\rangle_{\bmY} =  \E[ \langle u, A u\rangle_{\bmY}] \leq \E[ \langle u, B u\rangle_{\bmY}] = \langle u, \E[B] u\rangle_{\bmY}$.

\paragraph{2)} From Lemma B.2 in \citet{Rudi2013OnTS}.

\end{proof}

%% file: proofs/remark_independence_phi_eps.tex
In this section, we discuss the assumption that the random variables $\phi(x)$ and $\epsilon$ are independent.

In this work, this assumption allows to obtain shorter and lighter derivations, and an easier reading of the proofs. Nevertheless, such assumption is not exploited by the proposed method, and similar results can be proven without this assumption. More precisely, one can prove bounds with the same dependencies in the parameters of the learning setting, leading to the same conclusions. We discuss how below.

\paragraph{How to obtain similar bounds without this assumption? } 

The independence of $\phi(x)$ and $\epsilon$ allow simpler derivations when bounding expectations involving products of these two random variables using $\mathbb{E}[f(\phi(x)) g(\epsilon)] = \mathbb{E}[f(\phi(x)) \times \mathbb{E}[g(\epsilon)]$. This is used multiple times from Equations (38) to (48) to prove the Lemma \ref{lem:krr_op}, and only there.

We carried out derivations below in order to bound the same quantities but we do not make use of the assumption. Then, we will check that the dependencies in the parameters of the learning setting are similar.

\paragraph{Sketch of the proof (Bound $\|(Z_n^*S_n- HC_n)(C + \lambda_2 I)^{-1/2}\|_{\infty} $ without the independence assumption). } We define 
\begin{align}
    \xi_i = \epsilon_i \otimes \phi(x_i)(C + \lambda_2 I)^{-1/2}
\end{align}
with $\epsilon_i = y_i - h^*(x_i)$. In this way,
\begin{align}
    \|(Z_n^*S_n- HC_n)(C + \lambda_2 I)^{-1/2}\|_{\infty}= \|\frac{1}{n} \sum_{i=1}^n \xi_i - \E[\xi]\|_{\infty}.
\end{align}
We aim at applying the Bernstein inequality given in Theorem \ref{thm:tropp} to the random linear operator $u:= \xi - \E[\xi]$. So, we define 
\begin{align}
    T &:= 4 \kappa Q \lambda_2^{-1/2} \geq \|u\|_{\infty},\\
    \sigma^2 &:= \max(\|\E[u u^*]\|_{\infty}, \|\E[u^*u]\|_{\infty}),\\
    d &:= \Tr(\E[u^*u] + \E[uu^*]) / \sigma^2.
\end{align}
Note that $\|\epsilon\| \leq \|y\|_{\bmY} + \|h^*(x)\|_{\bmY} \leq 2Q$, and $\|\phi(x)\| \leq \kappa$. Then, we have
\begin{align}
    \E[u u^*] &= \E[(\epsilon \otimes \phi(x) - \E[\epsilon \otimes \phi(x)])(C + \lambda_2 I)^{-1}(\epsilon \otimes \phi(x) - \E[\epsilon \otimes \phi(x)])^*]\\
    &\preceq \lambda_2^{-1} \E[(\epsilon \otimes \phi(x) - \E[\epsilon \otimes \phi(x)])(\epsilon \otimes \phi(x) - \E[\epsilon \otimes \phi(x)])^*]\\
    &= \lambda_2^{-1} \E[(\epsilon \otimes \phi(x) (\epsilon \otimes \phi(x))^*] - \E[\epsilon \otimes \phi(x)] \E[\epsilon \otimes \phi(x)]^*\\
    &\preceq \lambda_2^{-1} \E[\epsilon \otimes \epsilon \|\phi(x)\|^2]\\
    &\preceq \lambda_2^{-1} \kappa^2 \E[\epsilon \otimes \epsilon] = \lambda_2^{-1} \kappa^2 E
\end{align}
where $\preceq$ denotes the Löwner’s partial ordering of positive semidefinite operators. We used properties of Löwner’s partial ordering (cf. Lemma \ref{lemma:lowner}). So, we have
\begin{align}
    \|\E[u u^*]\|_{\infty} &\leq \lambda_2^{-1} \kappa^2 \|E\|_{\infty}.
\end{align}
Then, similarly, we have
\begin{align}
    \E[u^* u] &= (C + \lambda_2 I)^{-1/2}  \E[(\epsilon \otimes \phi(x) - \E[\epsilon \otimes \phi(x))^*(\epsilon \otimes \phi(x) - \E[\epsilon \otimes \phi(x))](C + \lambda_2 I)^{-1/2}\\
    &= (C + \lambda_2 I)^{-1/2} \left(\E[ \phi(x) \otimes \phi(x) \|\epsilon\|^2] - \E[\phi(x) \otimes \epsilon] \E[\phi(x) \otimes \epsilon]^*\right)(C + \lambda_2 I)^{-1/2}\\
    &\preceq (C + \lambda_2 I)^{-1/2} 4Q^2 C (C + \lambda_2 I)^{-1/2}\\
    &\preceq 4Q^2 I_{\bmY}.
\end{align}
 So, we have
\begin{align}
    \|\E[u^* u]\|_{\infty} &\leq 4Q^2.
\end{align}
Now, from previous derivations, if $\lambda_2 < \|C\|_{\infty}$, we also have 
\begin{align}
    \Tr(\E[u u^*]) &\leq \lambda_2^{-1} \Tr(E) \kappa^2,\\
    \Tr(\E[u^* u]) &\leq 4Q^2 \lambda_2^{-1} \Tr(C),\\
    \|\E[u u^*]\|_{\infty} &\geq \frac{\|\text{Var}(\epsilon \otimes \phi(x))\|_{\infty}}{2\|C\|_{\infty}}.
\end{align}
by defining $\text{Var}(\epsilon \otimes \phi(x)) = \E[(\epsilon \otimes \phi(x) - \E[\epsilon \otimes \phi(x)])(\epsilon \otimes \phi(x) - \E[\epsilon \otimes \phi(x)])^*]$.
So, we have
\begin{align}
  d &\leq \frac{\Tr(\E[u^*u]) + \Tr(\E[uu^*])}{\|\E[uu^*]\|_{\infty}}\\
  &\leq \lambda_2^{-1} \frac{2(\Tr(E)\kappa ^2 + 4 Q^2\Tr(C))\|C\|_{\infty}}{\|\text{Var}(\epsilon \otimes \phi(x))\|_{\infty}}.
\end{align}

\paragraph{Conclusion. } Then, one can bound $\|(Z_n^*S_n- HC_n)(C + \lambda_2 I)^{-1/2}\|_{\infty}$ as in the proof of Lemma \ref{lem:krr_op} by applying the Bernstein inequality given in Theorem \ref{thm:tropp}.

The dependencies in the learning setting's parameters of the resulting bound will depend on the dependencies in the learning setting's parameters of the obtained bounds on $\|\E[u^* u]\|_{\infty}$, $\|\E[u u^*]\|_{\infty}$, and $d$.

Notice that the bounds on $\|\E[u^* u]\|_{\infty}$, $\|\E[u u^*]\|_{\infty}$ have the same dependencies in the learning setting's parameters than the ones obtained in Lemma \ref{lem:krr_op} on $\|\E[\xi^* \xi]\|_{\infty}$, $\|\E[\xi \xi^*]\|_{\infty}$.

The bound on $d$ obtained above without the independence assumption has poorer dependencies in the learning setting's parameters than the one obtained in Lemma \ref{lem:krr_op}. More precisely, $d$ has poorer dependencies in $t_1$ and $\lambda_2$. Nevertheless, it remains polynomial dependencies in $t_1^{-1}$ and $\lambda_2^{-1}$, such that the resulting $\eta = \log(\frac{4d}{\delta})$, in the proof of Lemma \ref{lem:krr_sub}, has similar dependencies in the learning setting's parameters than the one obtained in Lemma \ref{lem:krr_sub}.

We conclude that, without the independence assumption of $\phi(x)$ and $\epsilon$, one can prove bounds similar to Theorem \ref{th}, namely with the same dependencies in the parameters of the learning setting.

%% file: sections/exp_details.tex
In this section, we give an additional synthetic experiment (Section \ref{subsec:diff}) that aims at discussing the difference between the output source condition (Assumption \ref{as:3}) and the standard source condition \citep{ciliberto2020general}. We also give additional details on the experiments for the sake of reproducibility (Sections \ref{subsec:usps}, \ref{subsec:mult}).

\subsection{Difference Between Standard Source Condition and Assumption \ref{as:3}.}\label{subsec:diff}

From Assumption \ref{as:1} we have $M=HCH^*$. Hence, Assumption \ref{as:3} measures the alignment between $HCH^*$ and $HH^*$. Notice that it's a different assumption than requiring the alignment of $C$ and $H^*H$ (source condition). Indeed, in general strong Assumption \ref{as:3} doesn't imply strong source condition. For instance, when $H$ is finite rank (e.g. $H = y_0 \otimes h_0$ with $y_0 \in \bmY, h_0 \in \bmH_x$), Assumption \ref{as:3} is verified with $\beta=0$ (best case), while the source condition can be arbitrarily bad (e.g. if $\langle h_0 |C^{-(1-v)} h_0\rangle_{\bmH_x} = +\infty$ with $v>0$, then the source condition can't be verified for $r \leq v$). Source condition is verified with $r=1-2u$ by operators of the form $H=H_0C^u$ with $H_0 \in \bmY \otimes \bmH_x$, $\|H_0\|_{\hs} < +\infty$, $u \in [0, \frac{1}{2}]$. Similarly, Assumption \ref{as:3} is verified with $\beta = \frac{1}{2u + 1}$ by operators of the form $H = (H_0CH_0^*)^{u}H_0$ with $\|H_0\|_{\infty} < +\infty$, $u \in [0, +\infty[$. \\
We illustrate this empirically. For $d=200$, $\bmX = \bmH_x = \bmY = \R^d$, we choose $\mu_p(C) = \frac{1}{p^{2}}$ and draw randomly the eigenvector associated to each eigenvalue. We draw $H_0 \in \R^{d \times d}$ with independently drawn coefficients from the standard normal distribution. Notice that $\beta$ and $r$ can be measured as the increasing rates, when $t, \lambda \rightarrow 0$, in $t^{-\beta}$ and $\lambda^{-r}$ of the quantities $\|(M + t)^{-\frac{1}{2}}H\|_{\infty}^2$ and $\|H(C+\lambda)^{-\frac{1}{2}}\|_{\infty}^2$. Hence, we compute and plot on Figure \ref{fig:source_cond} $\|H(C+\lambda)^{-\frac{1}{2}}\|_{\infty}^2$ w.r.t $\lambda$ (left), and $\|(M + t)^{-\frac{1}{2}}H\|_{\infty}^2$ w.r.t $t$ (right), with $H=(H_0CH_0^*)^{\gamma}H_0$ for various $\gamma \in [0, 1.5]$. We also plot in Figure \ref{fig:source_cond} (right) the slopes $\beta = \frac{1}{2\gamma + 1}$. Firstly, we see that Assumption \ref{as:3} indeed improved when $\gamma$ increases, while the source condition is low and does not change. Then, as explained $H = (H_0CH_0^*)^{\gamma}H_0$ verifies Assumption \ref{as:3} with at least $\beta = \frac{1}{2\gamma + 1}$, but depending on $H_0$ it might be verified for $\beta \ll \frac{1}{2\gamma + 1}$. Nonetheless, notice that with our generated $H_0$, $\beta = \frac{1}{2\gamma + 1}$ are sharp for $H =  (H_0CH_0^*)^{\gamma}H_0$. 
 \begin{figure}[ht!]
\begin{minipage}[b]{0.5\linewidth}\centering
\includegraphics[width=70mm]{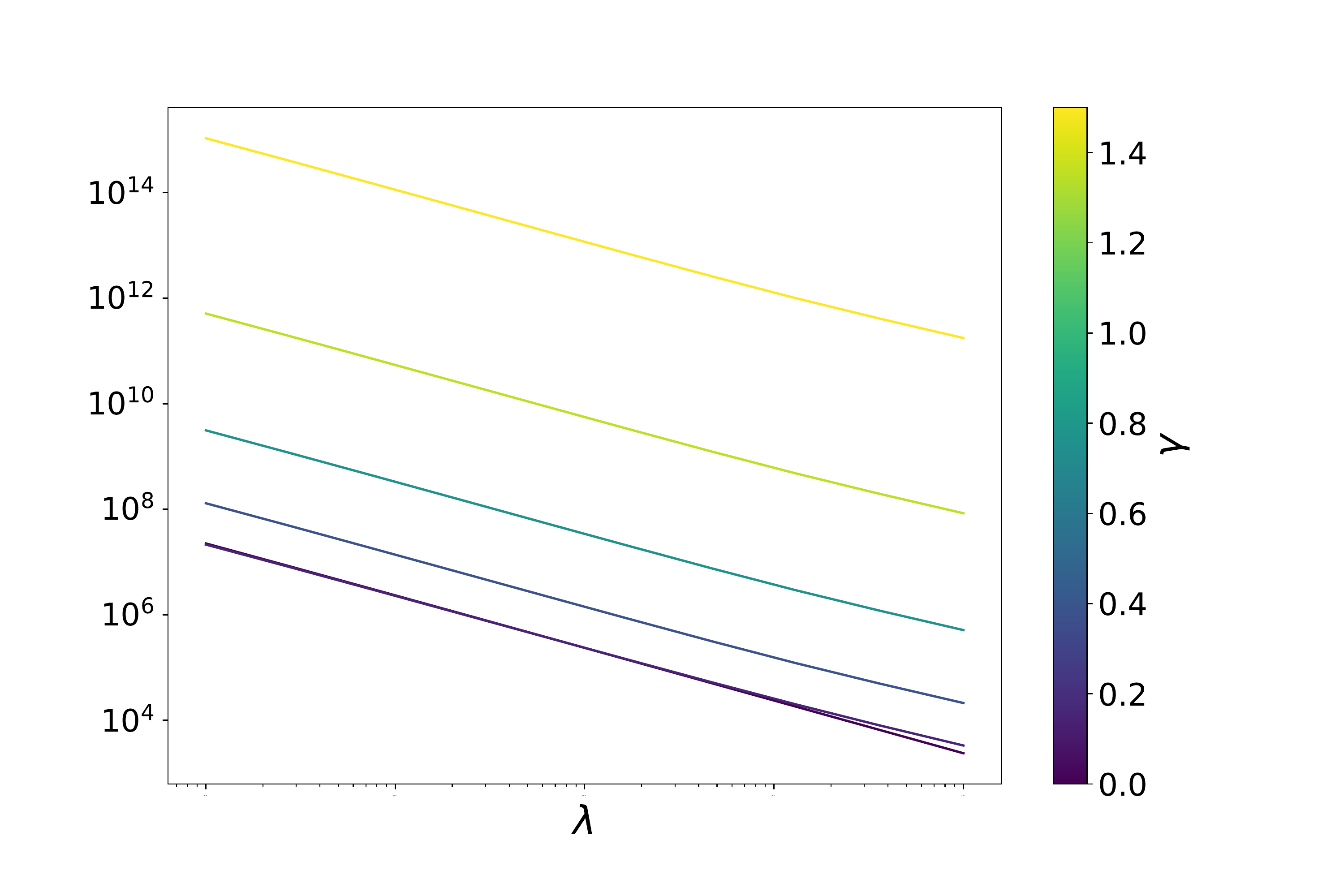}
\end{minipage}
\begin{minipage}[b]{0.5\linewidth}\centering
\includegraphics[width=70mm]{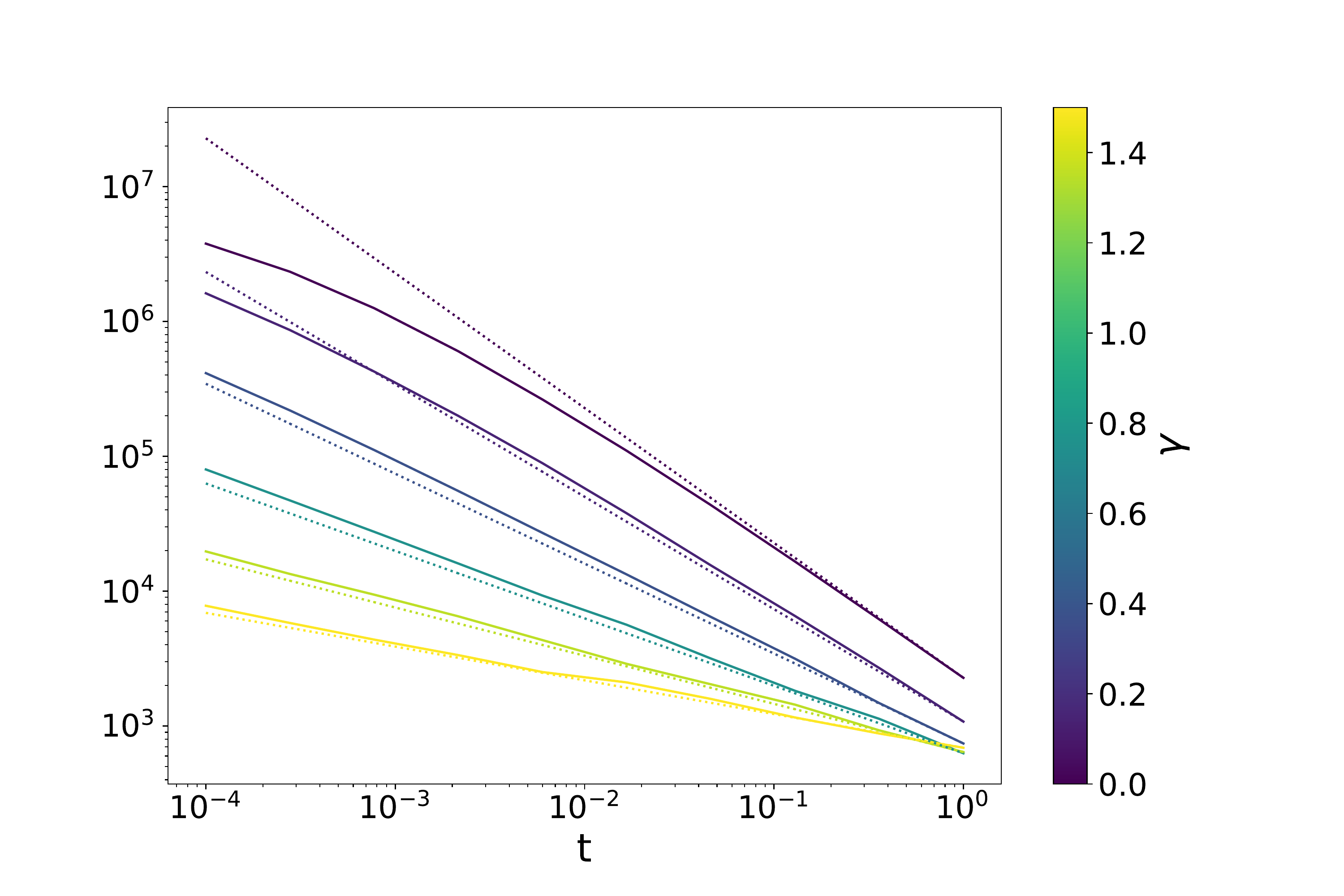}
\end{minipage}
\caption{Source condition $\|H(C + \lambda)^{-\frac{1}{2}}\|_{\infty}^2$ w.r.t $\lambda$ (left) and output source condition $\|(M + t)^{-\frac{1}{2}}H\|_{\infty}^2$ w.r.t $t$ (right) in log-log scale for $H=(H_0CH_0^*)^{\gamma}H_0$ and various $\gamma \in \{0, 0.1, 0.25, 0.5, 0.9, 1.5\}$.}
\label{fig:source_cond}
\end{figure}

\subsection{Image Reconstruction}\label{subsec:usps}

\textbf{Link to downloadable data set } \url{https://web.stanford.edu/~hastie/StatLearnSparsity_files/DATA/zipcode.html}

\paragraph{SPEN USPS experiments' details.} We used an implementation of SPEN in python with PyTorch by Philippe Beardsell and Chih-Chao Hsu (cf. https://github.com/philqc/deep-value-networks-pytorch). Small changes have been made. 
SPEN was trained using standard architecture from \cite{belanger2016}, that is a simple 2-hidden layers neural network for the feature network with equal layer size $n_h=110$, and a single-hidden layer neural network for the structure learning network with size $n_s=50$. The size of the two hidden layers $n_h \in [10,30,50,70,90,110,130]$ was selected during the pre-training of the feature network using 5 repeated random sub-sampling validation (80\%/20\%) selecting the best mean validation MSE (cf. Figure \ref{fig:convergence} for convergence of this phase). $n_s \in [5, 10, 20, 50, 70]$ was selected during the training phase of the SPEN network (training of the structure learning network plus the last layer of the feature network) doing approximate loss-augmented inference (cf. Figure \ref{fig:convergence}  for inferences' convergences), and minimizing the SSVM loss, using 5 repeated random sub-sampling validation (80\%/20\%) selecting the best mean validation MSE (cf. Figure \ref{fig:convergence} for convergence of this phase).

\begin{figure}[!ht]
    \centering
    \begin{minipage}[l]{0.33\textwidth}
        \centering
        \includegraphics[height=0.17\textheight]{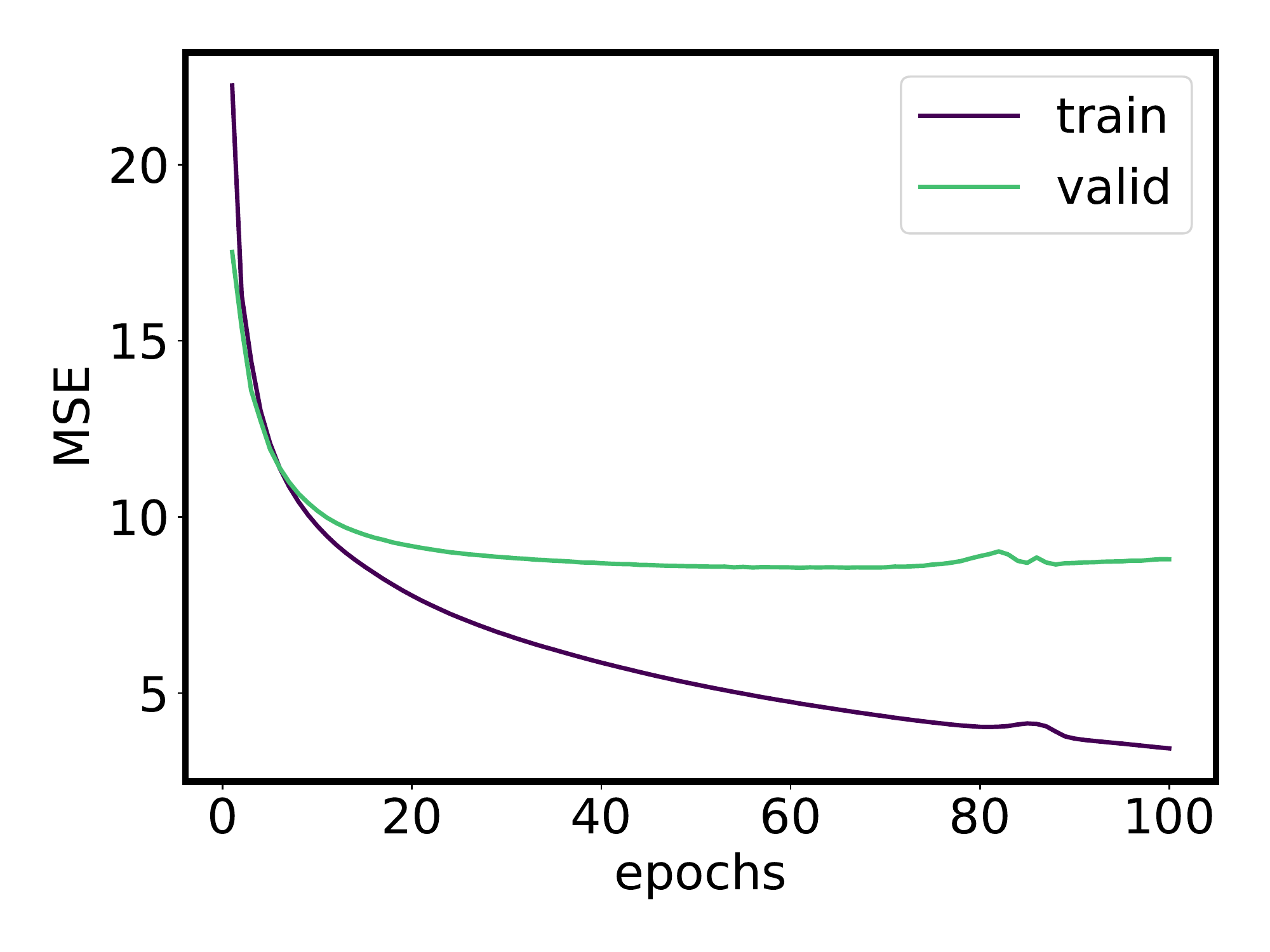}
    \end{minipage}%
    \begin{minipage}[l]{0.33\textwidth}
        \centering
        \includegraphics[height=0.17\textheight]{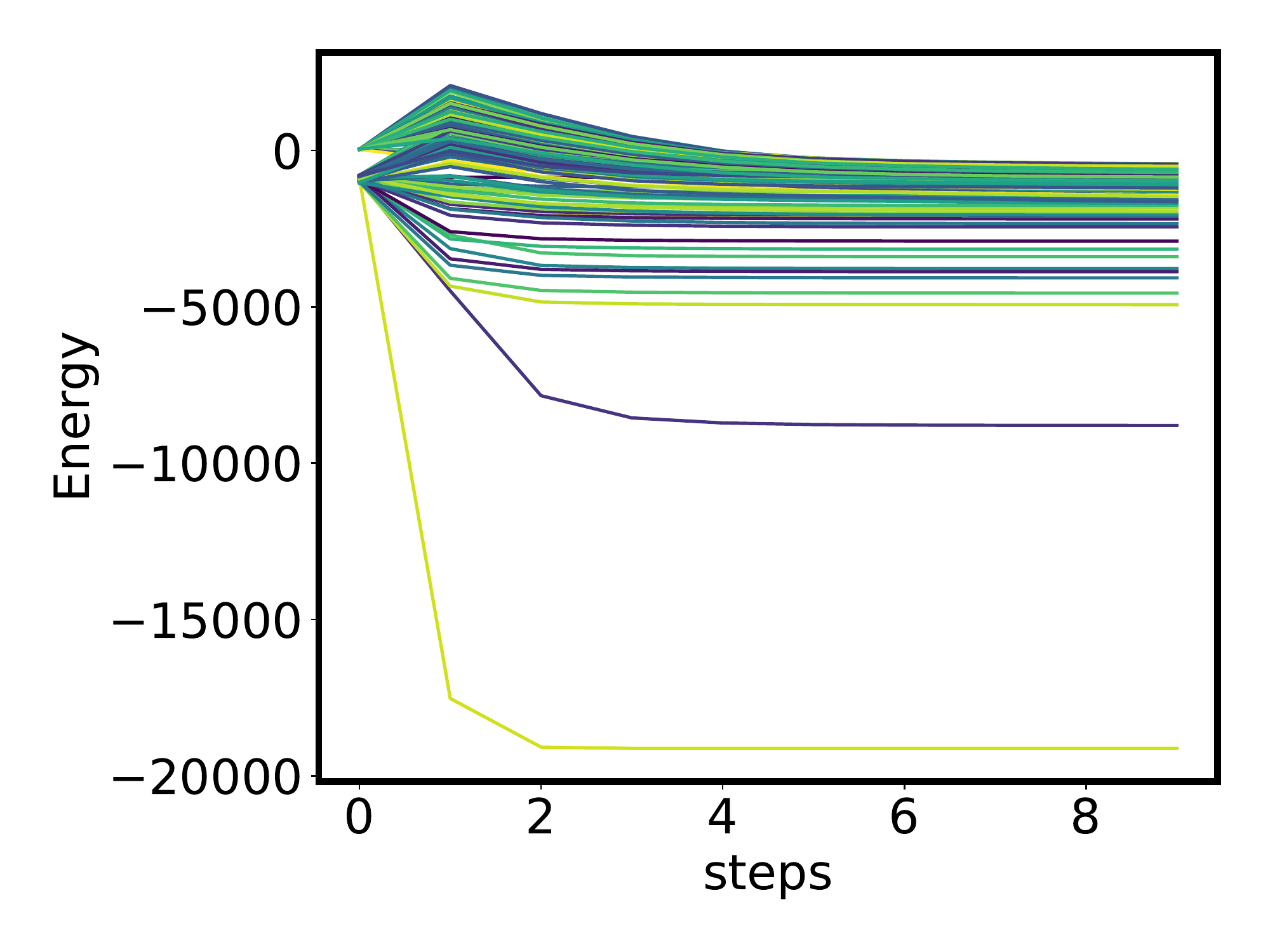}
    \end{minipage}%
    \begin{minipage}[l]{0.33\textwidth}
        \centering
        \includegraphics[height=0.17\textheight]{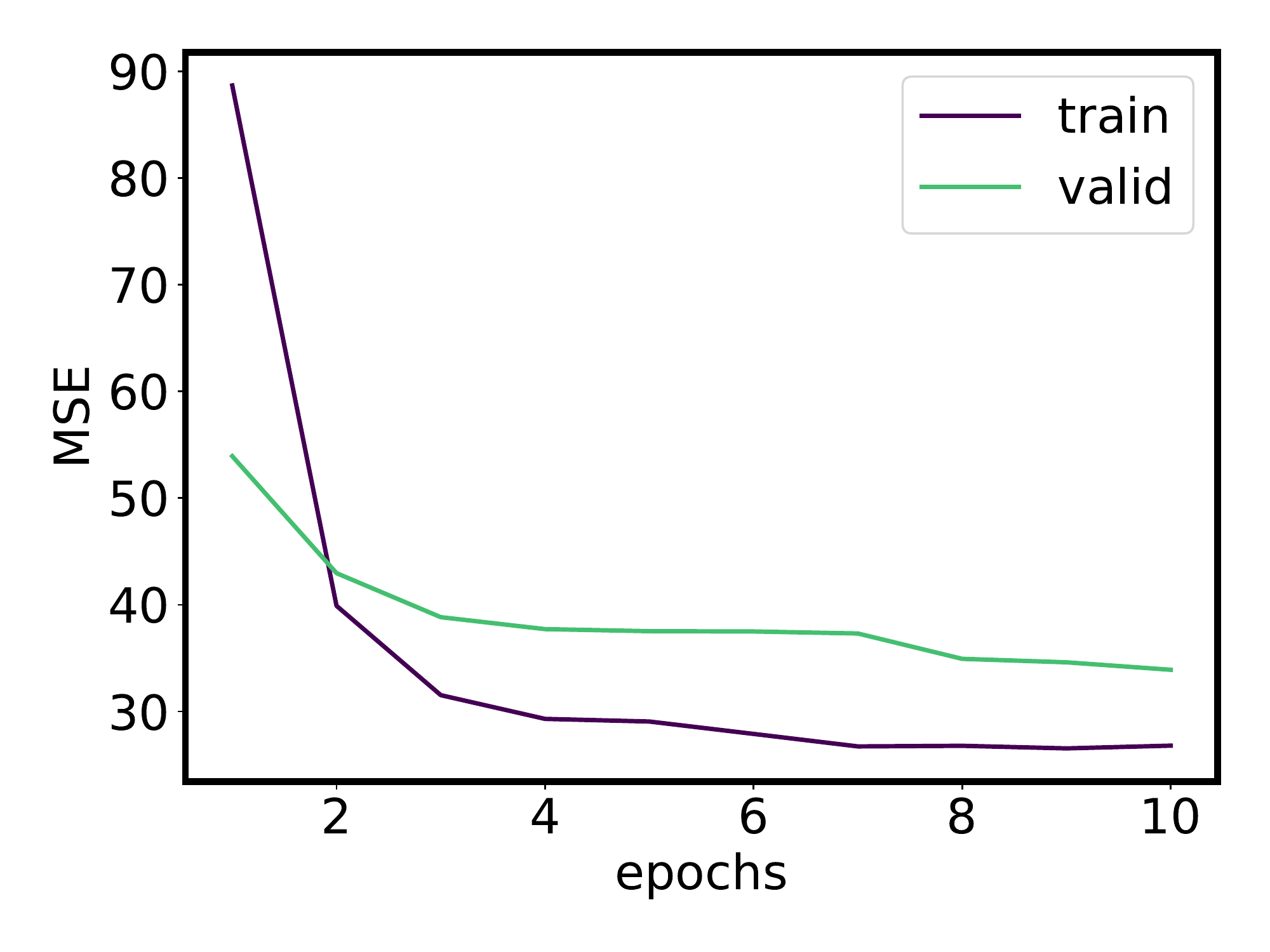}
    \end{minipage}
    \caption{Left: Convergence of train/validation MSE when pre-training the feature network. / Center: approximate loss-augmented inferences' convergences. / Right: Convergence of train/validation SSVM loss when training the SPEN network.}
    \label{fig:convergence}
\end{figure}

\subsection{Multi-label Classification}\label{subsec:mult}

\textbf{Link to downloadable data set } \url{http://mulan.sourceforge.net/datasets-mlc.html}